\documentclass[11pt]{article}

\usepackage[T1]{fontenc}
\usepackage{amsmath,amssymb,amsfonts,amsthm}

\usepackage[mathcal]{euscript}

\usepackage{natbib}

\usepackage{bbm}  % for \mathbbm{1}
\usepackage{bm}   % for \bm{\rho}, etc.

\usepackage{graphicx}
\usepackage{booktabs}
\usepackage{multirow}
\usepackage{float}       % for [H]
\usepackage{caption}
\usepackage{subcaption}

\usepackage{enumitem}
\usepackage{xcolor}

\usepackage[ruled,vlined,linesnumbered]{algorithm2e}

\usepackage{tikz}
\usetikzlibrary{positioning,shapes.geometric,shapes.misc,arrows}

\newtheorem{theorem}{Theorem}[section]
\newtheorem{lemma}[theorem]{Lemma}
\newtheorem{proposition}[theorem]{Proposition}
\newtheorem{corollary}[theorem]{Corollary}
\newtheorem{remark}[theorem]{Remark}
\newtheorem{definition}[theorem]{Definition}

\DeclareMathOperator*{\argmin}{arg\,min}
\newcommand{\E}{\mathbf{E}}

\tikzstyle{terminator} = [rectangle, draw, text centered, rounded corners, minimum height=2em]
\tikzstyle{process}    = [rectangle, draw, text centered, minimum height=2em]
\tikzstyle{decision}   = [diamond, draw, text centered, minimum height=2em]
\tikzstyle{data}       = [trapezium, draw, text centered, trapezium left angle=60, trapezium right angle=120, minimum height=2em]
\tikzstyle{connector}  = [draw, -latex']

\title{Stochastic Path Planning in Correlated Obstacle Fields}
\author{
	Li Zhou$^{*}$ ~\&~ Elvan Ceyhan\thanks{Department of Mathematics and Statistics, Auburn University, Auburn, AL 36849, USA \\
		\texttt{ emails: lzz0062@auburn.edu (corresponding author), ceyhan@auburn.edu}}
}
\date{}

\newcommand{\funding}{\footnotetext{This work was supported by ONR Grant N00014-22-1-2572 and NSF Award \#2319157. }}

\begin{document}
\maketitle
\funding

\begin{abstract}
\noindent 
We introduce the Stochastic Correlated Obstacle Scene (SCOS) problem,
a navigation setting with spatially correlated obstacles of uncertain blockage status,
realistically constrained sensors that provide noisy readings and costly disambiguation.
Modeling the spatial correlation with Gaussian Random Field (GRF),
we develop Bayesian belief updates that refine blockage probabilities,
and use the posteriors to reduce search space for efficiency.
To find the optimal traversal policy,
we propose a novel two-stage learning framework.
An offline phase learns a robust base policy via optimistic policy iteration,
augmented with information bonus to encourage exploration in informative regions,
followed by an online rollout policy with periodic base updates via
a Bayesian mechanism for information adaptation.
This framework supports both Monte Carlo point estimation and
distributional reinforcement learning (RL) to learn full cost distributions,
leading to stronger uncertainty quantification.
We establish theoretical benefits of correlation-aware updating and
convergence property under posterior sampling.
Comprehensive empirical evaluations across varying obstacle densities,
sensor capabilities demonstrate consistent performance gains over baselines. 
This framework addresses navigation challenges in environments with
adversarial interruptions or clustered natural hazards.
\end{abstract}


\section{Introduction}\label{sec: introduction}
Navigation and path planning in uncertain, partially observable environments presents a critical challenge in operations research and robotics,
particularly when complicated by additional factors including adversarial interruptions and environmental obstacles.
Applications span network interdiction problems with incomplete information (\cite{azizi2024shortest}, \cite{smith2020survey}), 
autonomous driving \citep{wang2024survey}, disaster response and fire evacuation \citep{shiri2019online},
and adversarial route planning in military operations (\cite{berger2012new}, \cite{hickling2023robust}).
In such contexts, obstacles may have two statuses: false obstacles correspond to false alarms,
while true obstacles represent actual threats requiring costly detours or clearance.
Classic models, building upon foundational frameworks such as the Canadian Traveler's Problem (CTP) \citep{bar1991canadian}
and the Stochastic Obstacle Scene (SOS) problem \citep{papadimitriou1991shortest}
typically simplify the problem by making strong assumptions rarely met in practice:
(i) obstacle blockages are independent \citep{alkaya2021heuristics}, ignoring potential spatial patterns arising from adversarial behavior or terrain features \citep{alkaya2021heuristics, pitilakis2016systemic},
and (ii) sensing is often idealized, assuming noiseless reveals with no explicit sensing cost or
global sensing with high accuracy \citep{lamarre2025risk},
whereas real-world sensors have limited detection range and varying reliability levels.

To address these limitations, we introduce the Stochastic Correlated Obstacle Scene (SCOS) model,
which extends the SOS framework by explicitly incorporating correlation among obstacles and 
realistic sensors with limited range and varying accuracy level.
In this setting,
a navigating agent moves through an environment containing spatially correlated disk-shaped obstacles (e.g., danger zones)
whose blockage nature is revealed only through noisy sensor readings within a limited sensing range and costly disambiguation upon contact, 
introducing a strategic tradeoff between information collection costs and traversal risks. 
These enhancements make the SCOS model particularly suitable for real-world challenges, 
including navigation through natural disaster zones or defense operations in adversarial context.

The SCOS presents two major decision-making challenges: 
(i) the agent must optimally balance \emph{exploration}, 
gathering information from highly uncertain paths to potentially reduce future traversal cost, 
and \emph{exploitation}, taking a low-risk path with lower immediate expected cost based on current knowledge.
(ii) The problem requires effective handling of correlation information. 
Since sensor readings and disambiguation outcomes not only reveal local blockage information 
but also provide insights about correlated, unobserved regions. 
These challenges are difficult to address using existing approaches \citep{eyerich2010high,lim2017shortest,blumenthal2023rollout, lamarre2025risk}, 
which exhibit various limitations including limited robustness under high uncertainty, 
lack of theoretical performance guarantees
or frameworks for continuous belief refinement under spatial correlations.

To tackle these challenges,
we develop a statistical inference and computation pipeline that combines Bayesian inference with reinforcement learning (RL) and scalable search:
%\begin{enumerate}[leftmargin=1.25em]
(i) We derive GRF-based belief updates that refine posterior blockage probabilities from noisy, local signals and show how the resulting posteriors support principled \emph{search-space reduction} (pruning) for planning.
(ii)
We propose a two-stage strategy: an \emph{offline} optimistic policy-iteration (OPI) stage with an \emph{information-gain bonus} to direct exploration toward informative regions, followed by an \emph{online} rollout policy with periodic Bayesian updates to adapt as data accrue.
(iii)
Beyond Monte Carlo point estimates, we incorporate distributional RL to learn full cost distributions, improving uncertainty quantification and robustness under high noise/risk.
(iv)
We establish guarantees that correlation-aware beliefs \emph{weakly dominate} their independence-coarsened counterparts in expected total cost,  
information gain is \emph{submodular} under the linear–Gaussian sensing model, justifying greedy sensing policies, 
and the OPI/rollout scheme with \emph{posterior sampling} converges under standard discounted or stochastic-shortest-path (SSP) conditions. 
Extensive simulations across sensing ranges, noise levels, and correlation regimes demonstrate consistent improvements over strong baselines.
%\end{enumerate}

%To tackle these challenges,
%we propose a two-stage policy learning strategy with the following key contributions:
%(i) \emph{Exploration guided by information gain}:
%we augment the learning process with information bonus built upon mutual information,
%encouraging deeper exploration of uncertain and informative regions compared to traditional methods,
%with convergence property established.
%(ii) \emph{Flexible value function estimate approaches}: 
%in additional to classic Monte Carlo methods providing point estimates,
%our framework supports distribution RL to learn full cost distributions,
%providing stronger uncertainty quantification and more robust policies under high noise and risk. 
%(iii) \emph{Offline-online integration}:
%we develop a two-stage strategy enabling online adaptation, which outperforms standard rollout methods relying on static base policies. 
%It consists of
%an offline phase that learns a high quality base policy via hybrid RL, and
%an online rollout phase with periodic base updates via a Bayesian mechanism that improves point estimates and distribution approximation.
%(iv) \emph{Belief updates with space reduction}: 
%blockage probabilities are refined using a Bayesian updating framework based on Gaussian Random Field model,
%with a search space reduction step guided by the posteriors to enhance computational efficiency.

The article proceeds as follows: 
Section \ref{sec: related work} reviews related work and Section \ref{sec: c-sos define} formally defines the SCOS problem 
and introduce its formulation as a sequential decision model. 
Section \ref{sec: updating} describes the Bayesian updating process for obstacle information based on Gaussian Random Field (GRF). 
Section \ref{sec: components_policy} introduces three crucial building blocks of our proposed two-stage policy learning framework,
followed by detailed descriptions of the offline learning strategy and online rollout policy with base updates in Sections \ref{sec:offline} and \ref{sec:online}.
Section \ref{sec: simulations} presents empirical evaluations of our approach.

\section{Related Work}\label{sec: related work}
Stochastic path planning problems have been extensively studied,
with the Canadian Traveler Problem (CTP) \citep{bar1991canadian} representing a key graph-theoretical foundation
and providing insights into navigating graphs with incrementally revealed edge statuses.
The Stochastic Obstacle Scene (SOS) problem \citep{papadimitriou1991shortest} extends these insights to
continuous obstacles settings, emphasizing dynamic learning and information gathering cost.
However, simplified assumptions like obstacle independence and unrestricted sensor capability 
have limited the practical applicability of SOS and existing variants (\cite{aksakalli2011}, \cite{aksakalliari2014}, \cite{alkaya2017optimal}, \cite{alkaya2021heuristics}).

\emph{Network Interdiction Problem} (\cite{smith2020survey}, \cite{azizi2024shortest}) shares conceptual similarities with SOS and CTP,
but usually focuses on edges interdiction from an adversary's (i.e., interdictor) perspective with assumed perfect information.
By contrast, our problem provides a complementary aspect.
We prioritize path planning from the traveler's perspective,
managing uncertainty, dynamic exploration costs and spatially correlated obstacles causing realistic area blockage effects. 

Spatial correlation in navigation has been explored under the CTP framework.
The Gaussian Traveler Problem (GTP) \citep{dey2014gauss} models correlations between edge travel time or costs using a Gaussian process (GP),
but assumes all edges are traversable, limiting its applicability to the setting with blockages where costly disambiguation is required.
Similarly, 
the risk-averse CVaR-CTP \citep{lamarre2025risk} allows correlated edge costs and updates from noiseless observations.
The Bayesian Canadian Traveler Problem (BCTP) \citep{lim2017shortest,hou2022dynamic} incorporates correlated edge blockages,
but relies on the strong assumption of a restrictive prior hypothesis space.
In contrast, 
we use Gaussian Random Field (GRF) model to represent spatial correlations, 
as a generalization of Gaussian processes for structured spatial domains. 
This model allows us to model continuous spatial correlations among obstacles without restricting the hypothesis space 
and supports a complete posterior updating under noisy sensing.

To manage exploration-exploitation tradeoffs in uncertain environments, 
deterministic threshold-based policies were proposed (\cite{koenig2002d}, \cite{bnaya2009canadian}, \cite{lim2017shortest}), 
which lack a clear method to determine an appropriate threshold value and overlook the probabilistic information.
Policies based on rollout (\cite{eyerich2010high}, \cite{hou2022dynamic}, \cite{blumenthal2023rollout}) 
and UCT (Upper Confidence Bounds Applied to Trees) related policies (\cite{kocsis2006bandit}, \cite{tolpin2012mcts}) 
improve decision-making by using probabilistic sampling and trajectory simulation. 
However, they are sensitive to the quality of simulation policies and the accuracy level of probabilistic information, 
with no guarantee of convergence to optimality. 
Penalty-based policies modify traversal costs by adding penalties to discourage high-risk paths (\cite{alkaya2015penalty}, \cite{sahin2015comparison}, \cite{alkaya2021heuristics}), 
offering computationally efficient but requiring manual hyperparameter tuning with performance heavily dependent on problem settings.

Reinforcement learning (RL) approaches \citep{sutton2018reinforcement} 
have been successfully applied to balance immediate and long-term gains in navigation problems \citep{yu2013q, polydoros2017survey, wang2020mobile}. 
RL methods can broadly be categorized into model-based and model-free approaches.
Model-based approaches construct a model of environment's dynamics for efficient exploration and fast convergence, 
but suffer from high computational requirements and sensitivity to model accuracy.
Model-free approaches bypass the need for explicit modeling, 
but their reliance on extensive exploration often leads to slower convergence. 
Hence, hybrid approaches combining the strength of both have emerged,
shown to achieve promising performance (\cite{silver2008sample}, \cite{feinberg2018model}, \cite{pinosky2023hybrid}).
 
We build upon these foundations, unifying their strength while addressing limitations, through a framework that 
integrates correlation modeling, efficient exploration, and real-time decision-making, as detailed in the following sections.

\section{The Stochastic Correlated Obstacle Scene Problem (SCOS)}\label{sec: c-sos define}
The Stochastic Correlated Obstacle Scene (SCOS) extends the original SOS framework by modeling correlated obstacles and realistic sensor limitations.
The SCOS models an agent traversing from a starting point to a goal point within a traversal region containing disk-shaped obstacles
whose blockage statuses are initially unknown.

\begin{figure}[H]
	\centering
	\begin{tabular}{cc}
		\begin{subfigure}[b]{0.45\textwidth}
			\centering
			\includegraphics[width=\textwidth]{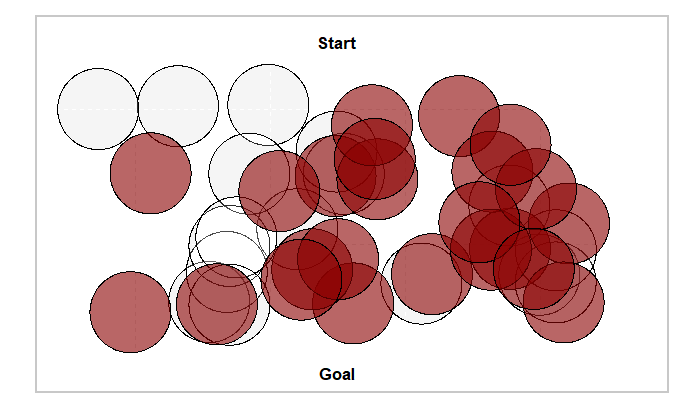}
		\end{subfigure} &
		\begin{subfigure}[b]{0.45\textwidth}
			\centering
			\includegraphics[width=\textwidth]{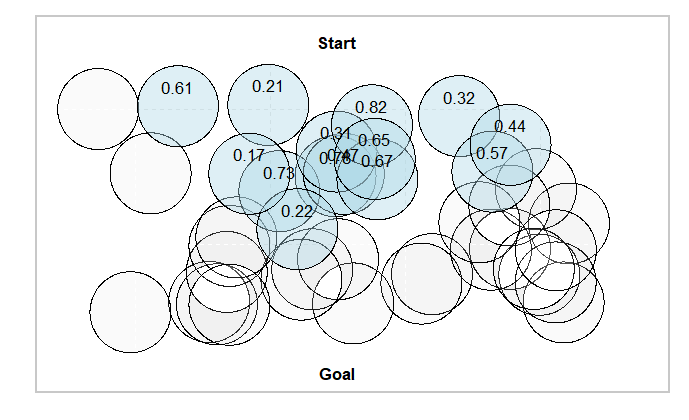}
		\end{subfigure} 
	\end{tabular}
	\caption{SCOS environment with obstacle statuses (left, red and gray disks representing true and false obstacles, respectively) 
		versus with probability estimates in sensing range (right).}
	\label{fig:Example_SCOS}
\end{figure}

Formally, let the two-dimensional traversal region be $\Omega\subset\mathbb{R}^2$, 
with a set of disk-shaped obstacles located at $X=X^F \cup X^O$, where $X^{F}$ and $X^{O}$ represent locations of false and true obstacles, respectively.
False obstacles can be traversed through, while true obstacles block traversal.
Each obstacle $x\in X$ has a fixed $\text{radius}(x)>0$.
An undirected graph $\mathcal{G}$ is imposed over $\Omega$, 
consisting of a set of vertices $\mathcal{V}(\mathcal{G})$ representing discrete navigation locations,
and $\mathcal{E}(\mathcal{G})$ represents a set of feasible undirected edges along which the agent navigates between vertices.
For each edge $e\in \mathcal{E}(\mathcal{G})$, undirected graph implies that $e=(v_i,v_j)$ and $e=(v_j, v_i)$ define the same edge, where $v_i, v_j\in \mathcal{V}(\mathcal{G})$.

The agent is equipped with a noisy sensor providing local blockage probability estimates, $\tilde{\rho}_x\in[0,1]$, of each obstacle $x$
lying within certain range $R$.
These probability estimates can be updated when obstacles are re-encountered, enabling dynamic belief revision.
Figure \ref{fig:Example_SCOS} illustrates this scenario, 
showing the contrast between actual obstacle states and the agent's probabilistic beliefs within its sensing range.
Nearby obstacles exhibit spatial correlations,
motivated by strategic placements or natural clustering.
These dependencies are captured via Gaussian Random Field (GRF), 
where blockage probabilities are inferred from a spatial correlated latent process.
Additionally,
the agent can disambiguate each obstacle's status with certainty once arriving at a vertex with an edge adjacent to the obstacle,  
incurring a disambiguation cost $c(x)$.

The agent can only traverse edges that do not intersect true or ambiguous obstacles
that have not been disambiguated.
When encountering an ambiguous obstacle, the agent performs costly disambiguation,
then either traverses through the obstacle if it is false, 
or takes a detour around the blockage if it is true.
Consequently, the total traversal cost from starting to goal point is a random variable
that depends on the blockage status of encountered obstacles.

Given a starting point designated as a specific vertex, denoted as $s$, and a goal point denoted as $g$, 
the objective is formulated as:
$$\min_{p\in\mathcal{P}_{sg}}\E\left[L_p+C_p\right],$$
where $\mathcal{P}_{sg}$ denotes the set of all paths from $s$ to $g$ in $\mathcal{G}$, $L_p$ represents the random variable denoting the traversal length of a path $p$, and $C_p$ is the random variable indicating the disambiguation costs associated with traversing path $p$.

\subsection{Formulating SCOS as a Sequential Decision Problem} \label{sec: sequential decision formulation}
Because disambiguations and sensor estimates continuously provide new information about the traversal region, 
SCOS is naturally posed as a sequential decision problem.
At each decision step, the agent collects new information and updates the belief about the traversal region,
then chooses the subsequent traversal path and the disambiguation location. 
Previous works typically formulate such problems as deterministic Partially Observable Markov Decision Process (DET-POMDP) 
\citep{bonet2012deterministic, dey2014gauss, aksakalli2016based}.
However, solving DET-POMDPs is PSPACE-complete \citep{bonet2012deterministic}, 
and the model transitions between successive steps are often complicated. 
To better facilitate the decision process, 
we instead model the problem using a universal framework with belief state proposed for sequential decision problems \citep{powell2019unified}.

Adapted to our problem setting,
this framework includes five components at each discrete time step $t\in\{0,1,...,T\}$, where a decision is required:
\begin{itemize}
	\item \emph{State Variables}, $S_t=\{\mathcal{V}_t, B_t\}$, contains information needed for decision making. 
	$\mathcal{V}_t$ represents the agent's physical location (i.e., the vertex in the discretized graph). 
	$B_t$ is the belief state capturing the agent's knowledge about obstacle uncertainty. 
	$B_t=\{X_t^{O}, X_t^{F}, X_t^{U}\}$, where $X_t^{O}$, $X_t^{F}$, and $X_t^{U}$ 
	denote the sets of true obstacles, false obstacles and ambiguous obstacles, respectively. 
	For ambiguous obstacles, their blockage probabilities are included, denoted as $\rho(X_t^{U})=\{\rho_x: x\in X_t^{U}\}$. 
	We use $s_t$, $v_t$ and $b_t$ for their specific realizations at $t$. 
	
	\item \emph{Decision}, $d_t$, represents the agent's decision of the next obstacle-free path segment to follow.
	Rather than selecting an action at every graph edge, we let the agent choose a \emph{macro–path}: 
	an obstacle-free path segment that terminates either 
	(i) at the first vertex adjacent to an ambiguous obstacle, or 
	(ii) at the goal $g$. 
	This approach reduces the decision complexity from potentially hundreds of edge level choices to only a few macro-decisions.
	Then, the decision $d_t$ can be viewed as selecting the stopping vertex, 
	since the path segment is uniquely determined by the current location and the chosen stopping point.
	We can also denote it as $d(s_t)$ to emphasize its dependency on state variables.
	
	\item \emph{Exogenous Information}, $W_t$, includes new information available to the agent. 
	Specifically, $W_{t+1}$ contains the actual status of obstacles observed after performing $d_{t}$, 
	as well as the new sensor estimates for obstacles within the sensing range. 
	It can be empty if the agent reaches $g$ directly without any sensing or disambiguation.
	
	\item \emph{Transition Function}, $S_{t+1}=\mathcal{T}(S_t, d_t, W_{t+1})$,
	describes the state information evolution based on the current state variables $S_t$, decision $d_t$ and new observation $W_{t+1}$.
	The new physical location $V_{t+1}$ becomes the vertex chosen by $d_t$, 
	while the belief state $B_{t+1} = \{X^{O}_{t+1}, X^{F}_{t+1}, X^{U}_{t+1}\}$ is updated based on the new obstacle knowledge provided in $W_{t+1}$.
	The transition uncertainty arises from the stochastic nature of obstacle blockage. 
	For $x\in X^{U}_{t+1}$, blockage probabilities are updated using a Bayesian mechanism incorporating spatial correlations between obstacles. 
	This transition process is depicted in Figure \ref{diagram}.

	\item \emph{Objective Function} seeks an optimal policy $\pi^*$ that minimizes the expected total traversal cost. 
	The immediate cost $\mathcal{C}(S_{t}, d_{t})$ includes
	the Euclidean length $\ell_t$ of traversing to the stopping vertex, 
	and the disambiguation cost $c_t$ if required. 
	Therefore, the objective function is formulated as:
	$$\pi^*=\argmin_{\pi}\E\left(\sum_{t=0}^{T}\mathcal{C}(S_{t}, d^\pi(S_{t}))\right),$$
	where $T$ denotes the (random) arrival time, $\mathcal{C}(S_t, d^\pi(S_t))=\ell_t+c_t$,
	and $d^\pi(S_{t})$ is the decision taken at $S_t$ under policy $\pi$.
\end{itemize}

\begin{figure}[H]
	\centering
	\begin{tikzpicture}[scale=0.7,transform shape]
		\node [terminator, fill=white!20] at (-10, 0) (start) {\textbf{$S_t$}};
		\node [data, fill=white!20] at (-6, 0) (data1) {\textbf{$d_t$}};
		\node [process, fill=white!20] at (-2, 0) (process1) {\textbf{$W_{t+1}$}};
		\node [terminator, fill=white!20] at (2, 0) (start2) {\textbf{$S_{t+1}$}};

		\node[draw=none] at (-8.5, -0.5) (yes) {decision};
		\node[draw=none] at (-4, -0.5) (no) {observation};
		\node[draw=none] at (-0, -0.5) (no) {next state};
		
		\path [connector] (start) -- (data1);
		\path [connector] (data1) -- (process1);
		\path [connector] (process1) -- (start2);
	\end{tikzpicture}
	\caption{A schematic illustration of transitions in the decision-making process}
	\label{diagram}
\end{figure}
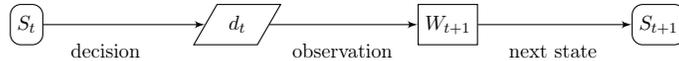

Based on the problem definition and sequential decision formulation,
we work on a finite, undirected graph with bounded, nonnegative edge costs,
and allow sensing, disambiguation at additional bounded costs.
The goal state $g$ is absorbing since the traversal stops and no further cost incurs once $g$ is reached.
Assuming there is at least one feasible path, possibly unnecessarily long, from $s$ to $g$, 
there exists a proper policy (i.e., one that reaches g almost surely from every state after finitely many steps).
Consequently, our objective is undiscounted with a finite horizon.
Throughout,
we restrict the analysis to proper policies.

Table \ref{tab:notations} in Appendix summarizes the key notations used for problem formulation and belief representation.

\section{Updating Framework} \label{sec: updating}
Due to the inter-dependency (i.e., correlation) between obstacles' blockage probabilities, 
sensor readings and disambiguation outcomes at one location provide useful information about the status of all ambiguous obstacles.
This motivates an online update mechanism that iteratively refines beliefs, 
allowing decisions to be made based on increasingly accurate information.
We establish the theoretical benefit to justify the proposed correlation-aware updating approach in Section \ref{sec: benefits_corr}.

Previous work on Gaussian classification for binary object status updates model correlations through a latent variable, 
followed by a squash step to transform values into probability scale \citep{kapoor2010gaussian}. 
However, the resulting likelihood function complicates the derivation of posterior and predictive distributions.
Approximation techniques \citep{nickisch2008approximations} can mitigate this issue, 
but introduce additional computational cost and concerns on approximation accuracy, especially when used in sequential updates. 

To address these challenges,
we map each obstacle's blockage probability $\rho_x\in(0,1)$ to log-odds values $\log\left(\frac{\rho_x}{1-\rho_x}\right)$. 
This transformation is a common practice in occupancy mapping problems (\cite{o2012gaussian}, \cite{li2018gaussian}).
We then develop a Bayesian updating framework using the Gaussian Random field (GRF)
to model the spatial structure of these log-odds values.
This strategy enables exact posterior derivation for sequential updates.

\subsection{Prior Information and Assumptions}
Following the problem definition in Section \ref{sec: c-sos define}, 
the traversal region contains $n$ obstacles located at positions $X=X^O\cup X^F=\{x_1, x_2, \ldots, x_n\}$. 
Each obstacle has an unknown binary status denoted as $Z=\{z_1, z_2, \ldots, z_n\}$, where $z_i=1$ indicates a true (blocked) obstacle
with underlying blockage probability $\rho^*_i$ such that $P(Z_i=1)=\rho^*_i$.

We denote the transformed log-odds vector as $Y=(y_1, y_2, \ldots, y_n)$,
where each $y_i=\log\left(\frac{\rho^*_i}{1-\rho^*_i}\right)$, for all $i=1, 2, ..., n$.
and assume a Gaussian prior:
$$Y|X\sim\mathcal{N}(\mu, K),$$
where $\mu$ is the prior mean vector and $K$ is the covariance matrix. 
We set $\mu = 0$ which centers the prior log-odds for each obstacle at zero, 
corresponding to a prior median (and mean, in the absence of uncertainty) of $\rho_i = 0.5$ for all $i$.
This setup indicates that each obstacle is equally likely to be true or false a priori. 
The matrix $K$ contains elements $\{K_{ij}=k(x_i,x_j)\}$ determined by a kernel function $k$ capturing the inter-dependency between obstacles at two locations $x_i$ and $x_j$. 
We use a squared exponential kernel in spatial distance to model correlation such that nearby obstacles tend to have similar log-odds values,
thereby capturing the intuition that spatially close obstacles tend to have similar blockage status:
$$k(x_i, x_j)=\sigma_f^2\exp\left\{-\frac{||x_i-x_j||^2}{2l^2}\right\},$$
where $\sigma_f$ and $l$ are kernel parameters controlling the strength and decay rate of spatial correlations.

Given a sequence of sensor readings and disambiguation outcomes for each obstacle,
denoted as $\tilde{\rho_i}^1, \ldots, \tilde{\rho_i}^t$, 
we aim to refine the probability estimates $P(Z_i=1|\tilde{\rho_i}^1, ..., \tilde{\rho_i}^t), i=1, 2, ..., n$. 
This framework allows us to update belief over obstacles that have not been directly sensed due to limited sensing range,
by using the spatial correlation captured in the Gaussian prior. 

\subsection{Observations}
In most works addressing traversal problems in regions with potential blockages, 
the sensor is assumed to be noisy and conditionally dependent on the true nature of the blockage.
The generated probabilistic readings about the blockage status are typically modeled using a $Beta(\alpha, \beta)$ distribution (\cite{fishkind2005}, \cite{ye2011sensor}, \cite{aksakalliari2014}) which has support matching the probability range. 

Let $z_i \in \{0,1\}$ denote the true (but unobserved) status of obstacle $i$. Given $z_i$, the sensor reading $\tilde{\rho}_i$ is drawn from:
\[
\tilde{\rho}_i \mid z_i \sim 
\begin{cases}
	\text{Beta}(\alpha_O, \beta_O), & \text{if } z_i = 1 \\
	\text{Beta}(\alpha_F, \beta_F), & \text{if } z_i = 0
\end{cases}
\]
The parameters are chosen to satisfy $\E(\tilde{\rho_i}|z_i=1)=\frac{\alpha_O}{\alpha_O+\beta_O}>\E(\tilde{\rho_i}|z_i=0)=\frac{\alpha_F}{\alpha_F+\beta_F}$,
reflecting that the sensor is more likely to return higher probability marks for true obstacles and lower values for false ones,
with uncertainty captured by the distribution shape parameters.

Given a new sensor reading $\tilde{\rho}_i^t$, we update its log-odds via Bayes' rule.
Assuming conditional independence of sensor readings given the obstacle statuses, the odds ratio becomes:
\[
\frac{P(z_i=1\mid\tilde{\rho}_i^{1}, ..., \tilde{\rho}_i^{t})}{P(z_i=0\mid\tilde{\rho}_i^{1}, ..., \tilde{\rho}_i^{t})}
=\frac{f(\tilde{\rho}_i^{t} \mid z_i=1)P(z_i=1\mid\tilde{\rho}_i^{1}, ..., \tilde{\rho}_i^{t-1})}{f(\tilde{\rho}_i^{t} \mid z_i=0)P(z_i=0\mid\tilde{\rho}_i^{1}, ..., \tilde{\rho}_i^{t-1})},
\]
we then obtain log-odds value by taking log-transformation:
\[
\log\left(\frac{P(z_i=1\mid\tilde{\rho}_i^{1}, ..., \tilde{\rho}_i^{t})}{P(z_i=0\mid\tilde{\rho}_i^{1}, ..., \tilde{\rho}_i^{t})}\right)
=\log\left(\frac{f(\tilde{\rho}_i^{t} \mid z_i=1)}{f(\tilde{\rho}_i^{t} \mid z_i=0)}\right)+\log\left(\frac{P(z_i=1\mid\tilde{\rho}_i^{1}, ..., \tilde{\rho}_i^{t-1})}{P(z_i=0\mid\tilde{\rho}_i^{1}, ..., \tilde{\rho}_i^{t-1})}\right).
\]
The resulting $\tilde{y}_i^{t}=\log\left(\frac{P(z_i=1\mid\tilde{\rho}_i^{1}, ..., \tilde{\rho}_i^{t})}{P(z_i=0\mid\tilde{\rho}_i^{1}, ..., \tilde{\rho}_i^{t})}\right)$  is treated as a noisy observation of the latent true log-odds $y_i$, approximated by:
\[
\tilde{y}_i^t \approx y_i + \epsilon_i, \quad \epsilon_i \sim \mathcal{N}(0, \sigma_i^2),
\]
where $\sigma_i^2$ reflects the uncertainty due to the sensor noise and the informativeness of the reading $\tilde{\rho}_i^t$.
Then $\tilde{y}_i$ serves as the input observation for the Bayesian updating process using correlation information.

To accommodate the fact that only a subset of obstacles are observed at each time step (due to limited sensing range), 
we assign an observation noise variance of $\sigma_i^2 = \infty$ for unobserved obstacles. 
This effectively removes their influence from the likelihood function, 
ensuring that their posterior estimates are driven solely by spatial correlation with observed obstacles.

\subsection{Posterior Distribution}
The prior distribution on log-odds value $Y=\left(y_1, y_2, ... ,y_n\right)^T$ is assumed as:
$$f(\mathbf{y})\propto\exp\left\{-\frac{1}{2}\mathbf{y}^TK^{-1}\mathbf{y}\right\}.$$
Let $\tilde{\mathbf{y}} = (\tilde{y}_1, \dots, \tilde{y}_n)^T$ be the observed log-odds vector with 
independent Gaussian noise $\epsilon_i \sim \mathcal{N}(0, \sigma_i^2)$. 
Then the likelihood is:
$$L(\tilde{\mathbf{y}}|\mathbf{y})\propto\prod_{i=1}^{n}\exp\left\{-\frac{1}{2}\frac{\left(\tilde{y_i}-y_i\right)^2}{\sigma_i^2}\right\}
=\exp\left\{-\frac{1}{2}\left(\tilde{\mathbf{y}}-\mathbf{y}\right)^T\Sigma^{-1}\left(\tilde{\mathbf{y}}-\mathbf{y}\right)\right\}, $$
where $\Sigma=\text{diag}(\sigma_1^2, ..., \sigma_n^2)$.  
The posterior of $Y$ given observations $\tilde{Y}$ becomes a standard Bayesian posterior of multivariate normal with diagonal Gaussian noise: 
\[
f(\textbf{y}|\tilde{\textbf{y}})\propto f(\mathbf{y})L(\tilde{\mathbf{y}}|\mathbf{y}),
\]
with the posterior mean and covariance:
$$\mu_{pos}=\E\left(Y|\tilde{Y}\right)=\left[\Sigma^{-1}+K^{-1}\right]^{-1}\Sigma^{-1}\tilde{\mathbf{y}},$$
$$K_{pos}=\left[K^{-1}+\Sigma^{-1}\right]^{-1}.$$
The posterior mean $\mu_{pos}$ can be transformed elementwise into posterior probabilities via the logistic function:
\[
\rho_i = \frac{\exp(\mu_{pos,i})}{1 + \exp(\mu_{pos,i})}, \quad i = 1, \dots, n.
\]

\subsection{Benefits of Posterior Updating based on the Correlation Structure} \label{sec: benefits_corr}
Having established the correlation model and updating framework, 
we analyze the theoretical benefits of 
(i) incorporating spatial correlation into belief updating process and 
(ii) exploration through collecting new observations.

\subsubsection{Result on Correlation-Aware Dominance}
\label{sec:belief-dominance}
We consider either a finite-horizon problem with $t\in\{0,\dots,T\}, \; T<\infty$  on an underlying partially observed environment
or an SSP (stochastic shortest path) setting with absorbing goal $g$.
Single-stage costs are bounded: $\sup_{s,d} \mathcal{C}(s,d)<\infty$.
At time $t$, the state is $S_t=\{V_t,B_t\}$ with location $V_t$ and belief $B_t$.
Let $b^{\mathrm{cor}}_t\in\mathcal B_{\mathrm{cor}}$ denote the correlation-aware belief (full joint posterior under the GRF model)
and $U(b^{\mathrm{cor}}_t,d_t,Y_{t+1})$ the Bayesian update after action $d_t$ and observation $Y_{t+1}$.
%Assume there exists a measurable ``coarsening'' map $\Gamma:\mathcal B_{\mathrm{cor}}\to\mathcal B_{\mathrm{ind}}$ such that, pathwise,
%\[
%b^{\mathrm{ind}}_t=\Gamma\!\big(b^{\mathrm{cor}}_t\big),\qquad
%b^{\mathrm{ind}}_{t+1}=\Gamma\!\Big(U\big(b^{\mathrm{cor}}_t,d_t,Y_{t+1}\big)\Big).
%\]
%(For example, $\Gamma$ may replace the joint posterior by the product of its marginals.)
Policies are sequences $(\pi_t)_{t}$ with measurable selectors $\pi_t:\,(V_t,b)\mapsto d_t\in\mathcal D(S_t)$ adapted to the canonical filtration.

We now provide the standing assumptions more explicitly:
\begin{enumerate}[label=(A\arabic*), leftmargin=2em]
	\item \label{ass:finite} (\emph{Finite horizon and bounded costs})
	$T<\infty$ and $\sup_{s,d}\mathcal C(s,d)<\infty$.
	\item \label{ass:beliefs} (\emph{Correlation-aware beliefs})
	$b^{\mathrm{cor}}_{t+1}=U(b^{\mathrm{cor}}_t,d_t,Y_{t+1})$ for all $t$.
	\item \label{ass:gamma} (\emph{Per-stage belief coarsening})
	There exists a measurable ``coarsening'' map $\Gamma:\mathcal B_{\mathrm{cor}}\to\mathcal B_{\mathrm{ind}}$ such that, pathwise,
	\[
	b^{\mathrm{ind}}_t=\Gamma\!\big(b^{\mathrm{cor}}_t\big),\qquad
	b^{\mathrm{ind}}_{t+1}=\Gamma\!\Big(U\big(b^{\mathrm{cor}}_t,d_t,Y_{t+1}\big)\Big).
	\]
	(For example, $\Gamma$ may replace the joint posterior by the product of its marginals.)
\end{enumerate}

%\paragraph{Standing assumptions and notation.}

\begin{theorem}[Correlation-Aware Dominance]\label{thm:coarsening-dominance}
	Under \ref{ass:finite}--\ref{ass:gamma}, let
	\[
	J^{\star}_{\mathrm{cor}}
	:=\inf_{\pi}\,
	\mathbb E\!\Big[\textstyle\sum_{t} \mathcal{C}\big(S_t,\pi_t(V_t,b^{\mathrm{cor}}_t)\big)\Big],
	\qquad
	J^{\star}_{\mathrm{ind}}
	:=\inf_{\tilde\pi}\,
	\mathbb E\!\Big[\textstyle\sum_{t} \mathcal{C}\big(S_t,\tilde\pi_t(V_t,b^{\mathrm{ind}}_t)\big)\Big],
	\]
	be the optimal expected total costs when decisions are based on correlation-aware vs.\ coarsened (independent) beliefs, respectively.
	Then $J^{\star}_{\mathrm{cor}}\le J^{\star}_{\mathrm{ind}}$.
\end{theorem}

\begin{proof}
	Fix any admissible $\tilde\pi=(\tilde\pi_t)$ on $(V_t,b^{\mathrm{ind}}_t)$ and define the \emph{replicated} policy on correlated beliefs by composition,
	\[
	\pi_t(V_t,b)\ :=\ \tilde\pi_t\!\big(V_t,\Gamma(b)\big).
	\]
	By the stagewise coarsening property (i.e. Assumption \ref{ass:gamma}), for every $t$ we have
	$\pi_t(V_t,b^{\mathrm{cor}}_t)=\tilde\pi_t(V_t,b^{\mathrm{ind}}_t)$, so the two policies induce identical action sequences and thus the same state/cost trajectories on every sample path. Therefore
	\[
	\mathbb E\!\Big[\textstyle\sum_{t} \mathcal{C}\big(S_t,\pi_t(V_t,b^{\mathrm{cor}}_t)\big)\Big]
	=\mathbb E\!\Big[\textstyle\sum_{t} \mathcal{C}\big(S_t,\tilde\pi_t(V_t,b^{\mathrm{ind}}_t)\big)\Big].
	\]
	Taking $\inf_{\tilde\pi}$ on the right yields \(J^{\star}_{\mathrm{cor}}\le J^{\star}_{\mathrm{ind}}\).
	That is, taking the infimum over $\tilde\pi$ on the right shows that the correlation-aware decision maker 
	can mimic any independent-belief policy by first coarsening and then acting; 
	therefore $J^{\star}_{\mathrm{cor}}\le J^{\star}_{\mathrm{ind}}$.
\end{proof}

\noindent
\emph{Equivalently}, the theorem implies an inequality (or dominance) in expectations as well.
Because the finite-horizon problem with finite state and action sets admits optimal policies, let
\[
\pi^{\star}_{\mathrm{cor}}\in\arg\min_{\pi}\,
\mathbb E\!\Big[\textstyle\sum_{t=0}^{T}\mathcal C\big(S_t,\pi_t(V_t,b^{\mathrm{cor}}_t)\big)\Big],
\quad
\pi^{\star}_{\mathrm{ind}}\in\arg\min_{\tilde\pi}\,
\mathbb E\!\Big[\textstyle\sum_{t=0}^{T}\mathcal C\big(S_t,\tilde\pi_t(V_t,b^{\mathrm{ind}}_t)\big)\Big].
\]
Then
\[
\mathbb E\!\Big[\textstyle\sum_{t=0}^{T}\mathcal C\big(S_t,\pi^{\star}_{\mathrm{cor},t}(V_t,b^{\mathrm{cor}}_t)\big)\Big]
\ \le\
\mathbb E\!\Big[\textstyle\sum_{t=0}^{T}\mathcal C\big(S_t,\pi^{\star}_{\mathrm{ind},t}(V_t,b^{\mathrm{ind}}_t)\big)\Big].
\]

\begin{remark}\label{rem:blackwell-bellman}
	(i) \textbf{On Blackwell Dominance.}
	Suppose that for each stage $t$ there exists a Markov kernel $K_t$ such that, conditional on the latent state and the past history, the \emph{independent} signal $Z^{\mathrm{ind}}_t$ is obtained by \emph{garbling} the \emph{correlated} signal $Z^{\mathrm{cor}}_t$:
	\[
	\mathcal L\!\big(Z^{\mathrm{ind}}_t \mid \text{state, history}\big)
	\;=\;
	\int K_t(\cdot \mid z)\,\mathcal L\!\big(Z^{\mathrm{cor}}_t \in dz \mid \text{state, history}\big),
	\]
	with $K_t$ independent of the latent state given $Z^{\mathrm{cor}}_t$.
	Then the correlated signal is stagewise more informative in Blackwell sense, and a standard dynamic extension implies the same value inequality as Theorem~\ref{thm:coarsening-dominance}.
	In practice, verifying \ref{ass:gamma} via a concrete belief coarsening map $\Gamma$ (e.g., product-of-marginals) is \emph{sufficient} for Theorem~\ref{thm:coarsening-dominance} and avoids constructing signal-level garblings; note that such a $\Gamma$ need not correspond to a Blackwell garbling of raw signals.
	
	(ii) \textbf{Bellman Perspective.}
	Let $V^{\mathrm{cor}}_t:\mathcal B_{\mathrm{cor}}\to\mathbb R$ and $V^{\mathrm{ind}}_t:\mathcal B_{\mathrm{ind}}\to\mathbb R$ denote the optimal cost-to-go at time $t$ in the correlation-aware and independent-belief problems, respectively.
	Define $\widetilde V_t:\mathcal B_{\mathrm{cor}}\to\mathbb R$ by $\widetilde V_t(b):=V^{\mathrm{ind}}_t\big(\Gamma(b)\big)$.
	Under \ref{ass:gamma}, a backward-induction argument shows
	\[
	V^{\mathrm{cor}}_t(b)\ \le\ \widetilde V_t(b)\ =\ V^{\mathrm{ind}}_t\!\big(\Gamma(b)\big)
	\qquad\text{for all }b\in\mathcal B_{\mathrm{cor}},\ t=T,\dots,0,
	\]
	i.e., the dynamic programming operator is monotone under information coarsening after composing with $\Gamma$.
	Evaluating at $t=0$ yields Theorem~\ref{thm:coarsening-dominance}.\\
	\emph{Sketch of Proof.} At $t=T$, the inequality is trivial.
	Assume it holds at $t+1$ and compare the Bellman minima for $V^{\mathrm{cor}}_t$ and $\widetilde V_t$;
	use that $b^{\mathrm{ind}}_{t+1}=\Gamma\!\big(U(b^{\mathrm{cor}}_t,d,Y_{t+1})\big)$ (Assumption~\ref{ass:gamma}) to align the conditional expectations, and minimize over the same feasible action set $\mathcal D(S_t)$.
\end{remark}

This result implies that correlation-aware updating provides theoretical guarantees for improved performance.
We also examine how additional information monotonically improves decision quality through the following result, 
further supporting the goal of uncertainty reduction in our posterior beliefs by sequentially gathering new observations through sensing and disambiguation.

\subsubsection{Result on Monotonicity with Additional Observations}
\label{sec:monotonicity}
Let $X^U$ denote the index set of ambiguous obstacles that could be observed (e.g., via sensing).
That is, the set $X^U=\{x_1,\ldots,x_m\}$ indexes all \emph{ambiguous} obstacles whose statuses are uncertain and, in principle, observable (e.g., via sensing) at time $0$.
For any $X\subseteq X^U$, let $\widetilde Y_X:=\{\widetilde Y_x:\ x\in X\}$ be the the corresponding pre-decision observation vector collected from $X$ at time $0$.
Write $\mathcal F(X):=\sigma(\widetilde Y_X)$ for the $\sigma$-field generated by those observations (together with the common prior).
Admissible policies for $X$ are those that are measurable w.r.t.\ the canonical filtration generated by $\mathcal F(X)$ and the state/action history.
Define the optimal expected total cost
\[
\E[\mathcal C(X)]\;:=\;\inf_{\pi\in\Pi_X}\;
\E\!\Big[\textstyle\sum_{t=0}^T \mathcal C\big(S_t,\ d_t\big)\Big],
\quad\text{where }d_t=\pi_t(\text{history up to }t;\ \mathcal F(X)).
\]
Thus $X^U_1\subseteq X^U_2\subseteq X^U$ represents increasing observational richness:
\[
\mathcal F(X^U_1)\ \subseteq\ \mathcal F(X^U_2),
\]
so a policy admissible under $\mathcal F(X^U_1)$ is also admissible under $\mathcal F(X^U_2)$.
Intuitively, the decision maker with access to $X^U_2$ can condition on all information available under $X^U_1$ plus the additional signals for $X^U_2\setminus X^U_1$ (and can always ignore extra coordinates if desired).
\begin{corollary}[Monotonicity with additional observations]\label{cor:monotone}
	Let $X^U_1\subseteq X^U_2\subseteq X^U$.
	Then the optimal expected cost is monotone in the observation set:
	\[
	\E\big[\mathcal C(X^U_2)\big]\ \le\ \E\big[\mathcal C(X^U_1)\big].
	\]
	Equivalently, for the cost-reduction function $f(X):=\E[\mathcal C(\emptyset)]-\E[\mathcal C(X)]$,
	we have $f(X^U_2)\ge f(X^U_1)$ (monotone nondecreasing).
\end{corollary}

\begin{proof}
	Since $X^U_1\subseteq X^U_2$, we have $\mathcal F(X^U_1)\subseteq \mathcal F(X^U_2)$.
	By the definition of admissible policies, any $\mathcal F(X^U_1)$–adapted policy is also $\mathcal F(X^U_2)$–adapted; hence
	$\Pi_{X^U_1}\subseteq \Pi_{X^U_2}$.
	Therefore
	\[
	\E[\mathcal C(X^U_2)]
	=\inf_{\pi\in\Pi_{X^U_2}}\E\!\Big[\textstyle\sum_{t=0}^T \mathcal C(S_t,d_t)\Big]
	\ \le\
	\inf_{\pi\in\Pi_{X^U_1}}\E\!\Big[\textstyle\sum_{t=0}^T \mathcal C(S_t,d_t)\Big]
	=\E[\mathcal C(X^U_1)],
	\]
	as claimed.
\end{proof}

\begin{proof}[Alternative Constructive Proof:]
	Because $X^U_1\subseteq X^U_2$, we have $\mathcal F(X^U_1)\subseteq \mathcal F(X^U_2)$ and there is a measurable projection
	$\mathrm{proj}_{1}: \Omega\to$ realizations of $\mathcal F(X^U_1)$ that discards the extra observations in $X^U_2\setminus X^U_1$.
	Take any admissible policy $\pi^{(1)}\in\Pi_{X^U_1}$.
	%Construct a policy $\pi^{(2)}\in\Pi_{X^U_2}$ by \emph{ignoring} the extra information:
	%at each time $t$, let
	%\[
	%\pi^{(2)}_t(\text{history};\ \mathcal F(X^U_2))\ :=\ \pi^{(1)}_t\big(\text{history};\ \mathrm{proj}_{1}(\mathcal F(X^U_2))\big).
	%\]
	%By construction, $\pi^{(2)}$ takes exactly the same actions as $\pi^{(1)}$ on every sample path, so the induced
	%state/observation/cost processes are pathwise identical and
	%\[
	%\E\!\Big[\textstyle\sum_{t=0}^T \mathcal C\big(S_t,\pi^{(2)}_t\big)\Big]
	%=\E\!\Big[\textstyle\sum_{t=0}^T \mathcal C\big(S_t,\pi^{(1)}_t\big)\Big].
	%\]
	%Since $\pi^{(1)}$ was arbitrary, taking the infimum over $\Pi_{X^U_2}$ yields
	%$\E[\mathcal C(X^U_2)]\le \E[\mathcal C(X^U_1)]$, proving the claim.
	
	Define the restriction map $r:\widetilde Y_{X^U_2}\to \widetilde Y_{X^U_1}$ that drops coordinates in $X^U_2\setminus X^U_1$.
	Given $\pi^{(1)}\in\Pi_{X^U_1}$, define $\pi^{(2)}\in\Pi_{X^U_2}$ by
	$\pi^{(2)}_t(\text{history},\widetilde Y_{X^U_2}) := \pi^{(1)}_t(\text{history}, r(\widetilde Y_{X^U_2}))$.
	Then the two policies take identical actions pathwise, yielding the same cost; taking infima gives the result.
	
\end{proof}

%\begin{remark}[When a Blackwell view applies]\label{rem:blackwell-monotone}
%If one treats the raw signals as $Z_2=(\widetilde Y_{X^U_1},\widetilde Y_{X^U_2\setminus X^U_1})$ and $Z_1=\widetilde Y_{X^U_1}$,
%then $Z_1$ is obtained from $Z_2$ by the Markov kernel that projects away the extra coordinates.
%Thus $Z_2$ Blackwell-dominates $Z_1$ in the one-shot sense, and the above monotonicity follows.
%The proof given here avoids signal-level conditions and works directly with $\sigma$-fields and sequential policies.
%\end{remark}
%
%\begin{remark}[Sequential variants]\label{rem:sequential-monotone}
%If observations are acquired across time and the richer filtration $\{\mathcal F_t(X^U_2)\}$ dominates
%the poorer one $\{\mathcal F_t(X^U_1)\}$ at every $t$, the same replication argument yields
%$\E[\mathcal C(X^U_2)]\le \E[\mathcal C(X^U_1)]$ for the finite-horizon problem.
%\end{remark}

\begin{remark}[When a Blackwell view applies]\label{rem:blackwell-monotone}
	Let the raw signals be $Z_2=(\widetilde Y_{X^U_1},\,\widetilde Y_{X^U_2\setminus X^U_1})$ and $Z_1=\widetilde Y_{X^U_1}$.
	Define the (deterministic) Markov kernel $K$ by
	$K(A\mid z_2):=\mathbf 1\{\mathrm{proj}(z_2)\in A\}$,
	where $\mathrm{proj}$ drops the coordinates in $X^U_2\setminus X^U_1$.
	Then, \emph{conditional on the latent state and past history},
	\[
	\mathcal L\!\big(Z_1 \mid \text{state, history}\big)
	=\int K(\cdot\mid z_2)\,\mathcal L\!\big(Z_2\in dz_2 \mid \text{state, history}\big),
	\]
	so $Z_2$ Blackwell–dominates $Z_1$ in the one-shot sense.
	Consequently, any Bayes decision problem based on these signals has value (risk) under $Z_2$ no worse than under $Z_1$,
	and the monotonicity conclusion follows.
	Our proof in the main text circumvents signal-level conditions by working directly with $\sigma$-fields and sequential policy classes.
\end{remark}

\begin{remark}[Sequential variants]\label{rem:sequential-monotone}
	If observations are acquired across time and the richer filtration $\{\mathcal F_t(X^U_2)\}_{t=0}^T$ satisfies
	$\mathcal F_t(X^U_1)\subseteq \mathcal F_t(X^U_2)$ for every $t$,
	then the admissible policy classes obey $\Pi_{X^U_1}\subseteq \Pi_{X^U_2}$.
	Hence the same replication / policy-class inclusion argument yields
	$\E[\mathcal C(X^U_2)]\le \E[\mathcal C(X^U_1)]$ for the finite-horizon problem
	(and analogously for SSP with the sum taken to the hitting time).
\end{remark}

\section{Two-Stage Policy Learning Framework} \label{sec: components_policy}
Based on the sequential model proposed in Section \ref{sec: sequential decision formulation}, 
our goal is to find an optimal policy $\pi^*$ that guides the agent in making decisions based on the current state information. 
Each decision step determines not only the immediate path before encountering a new obstacle, but also impacts all future decision steps. 
Expanding the objective function, 
the optimal policy $\pi^*$ selects the decision $d_t$ from the decision set $\mathcal{D}_t(S_t)$ at each step according to the following equation:
\begin{equation} \label{equation:optimal}
	\resizebox{\textwidth}{!}{$
		d_t^* = \argmin_{d_t \in \mathcal{D}_t(S_t)} \left\{ \mathcal{C}(S_t, d_t) 
		+ \E_{W_{t+1} | S_t, d_t} \left[ \min_{\{d_{t'}, t' = t+1,...,T\}} \E\left( \sum_{t' = t+1}^{T} 
		\mathcal{C}(S_{t'}, d_{t'}) \Big| S_{t+1} \right) \right] \right\},
		$}
\end{equation}
This is a \emph{full lookahead policy} \citep{powell2022designing} that optimizes over the entire time horizon until reaching the goal $g$. 
However, we need to enumerate all possible future state-decision sequences, $\{S_{t'},d_{t'}, W_{t'+1}\}$,
which is infeasible in practice.  

Rollout policies offer a popular and computationally efficient approximation for real-time decision making \citep{bertsekas2021rollout}. 
However, they are highly sensitive to underlying approximation policy quality and environmental information accuracy, 
making their performance less robust in highly uncertain environments \citep{eyerich2010high,hou2022dynamic,blumenthal2023rollout}.

To address this challenge, 
we propose a two-stage policy learning framework (Figure \ref{fig:offline/online_loop}) for efficient decision making:
(i) the offline phase learns a good estimation of the minimum expected value,
$ \min\E\left(\sum_{t'=t+1}^{T}\mathcal{C}(S_{t'},d_{t'})|S_{t+1}\right)$,
which induces a base policy $\pi^{\text{base}}$, and
(ii) the online phase generates rollout policy $\pi^{\text{roll}}$, 
which uses the following rule for real-time decision making,
$d_t^{\text{roll}}=
\argmin_{d_t \in \mathcal{D}_t(S_t)}\left(\mathcal{C}(S_t, d_t)+
\E_{W_{t+1} | S_t, d_t}\E\left(\sum_{t'=t+1}^{T}\mathcal{C}(S_{t'},d_{t'}^{\text{base}})|S_{t+1}\right)\right)$.
Throughout the traversal process,
we periodically update the base policy using observations from rollout phase.
\begin{figure}[ht]
	\centering
	\begin{tikzpicture}[
		scale=0.8,
		transform shape,
		node distance=0.8cm and 2.0cm,
		box/.style={
			draw,
			rounded corners=3pt,
			minimum width=2.8cm,
			minimum height=0.8cm,
			align=center,
			fill=blue!8,
			font=\small
		},
		decision/.style={
			diamond,
			draw,
			minimum width=2cm,
			minimum height=1.2cm,
			inner sep=2pt,
			align=center,
			fill=orange!10,
			font=\small
		},
		process/.style={
			draw,
			rounded corners=3pt,
			minimum width=2.8cm,
			minimum height=0.8cm,
			align=center,
			fill=green!8,
			font=\small
		},
		arr/.style={->, thick, >=stealth},
		label/.style={font=\footnotesize, fill=white, inner sep=1pt}
		]
		
		% Stage 1: Offline Learning
		\node[box] (offline) {Offline Base Policy Learning \\ (Information-guided OPI)};
		
		% Stage 2: Online Execution
		\node[box, below=2.5cm of offline] (base) {Base Policy\\$\pi^{\mathrm{base}}, V^{\mathrm{base}}$};
		\node[process, right=of base] (rollout) {Online Rollout\\$\pi^{\mathrm{roll}}$};
		\node[process, below=of rollout] (exec) {Execute $d_t^{\mathrm{roll}}$\\Observe cost \& $S_{t+1}$};
		\node[decision, below=of exec] (test) {Reached\\$g$?};
		\node[process, left=of test] (update) {Belief Update \&\\Simulated Rollouts};
		\node[box, right=of test] (end) {Traversal\\Complete};
		
		% Stage labels
		\node[above=0.2cm of offline, font=\normalsize\bfseries] {Stage 1: Offline Learning};
		\node[above=0.2cm of base, font=\normalsize\bfseries] {Stage 2: Online Execution};
		
		% Arrows
		\draw[arr] (offline) -- (base);
		\draw[arr] (base) -- (rollout);
		\draw[arr] (rollout) -- (exec);
		\draw[arr] (exec) -- (test);
		
		% Decision arrows with labels
		\draw[arr] (test.west) -- node[label, above] {No} (update.east);
		\draw[arr] (test.east) -- node[label, above] {Yes} (end.west);
		
		% Clean feedback loop with label
		\draw[arr] (update.north) -- node[label, right] {Base Update} (base.south);
		
		% Dashed line to separate stages
		\draw[dashed, gray, thick] 
		([xshift=-1.5cm, yshift=-1.2cm]offline.south) -- 
		([xshift=1.5cm, yshift=-1.2cm]offline.south);
		
	\end{tikzpicture}
	\caption{Two-stage policy learning for SCOS}
	\label{fig:offline/online_loop}
\end{figure}
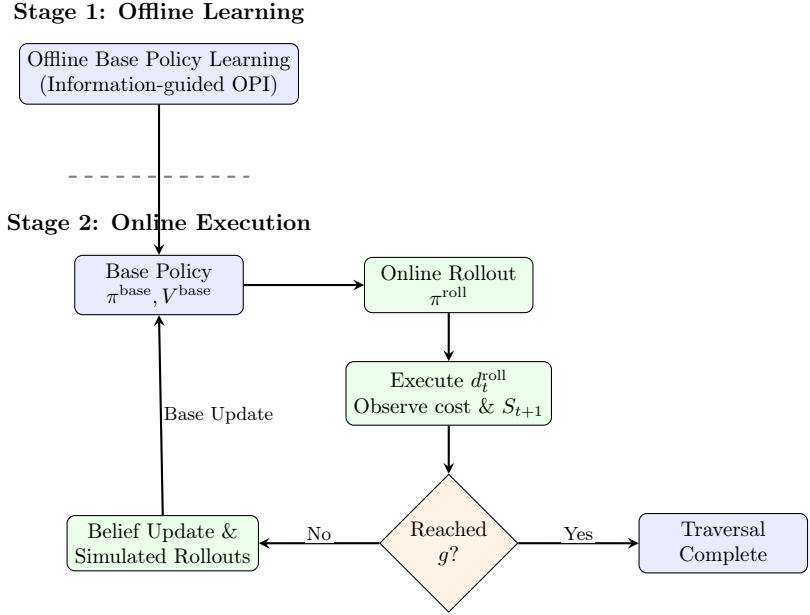

In this section, we first introduce three key components that enable our framework for effective and efficient policy learning.
we then detail the offline and online policy learning in Sections \ref{sec:offline} and \ref{sec:online}.

\subsection{Identifying the Decision Set with Search Space Reduction}\label{sec: identify Dt}
At each decision step, 
we first need to identify the decision set $\mathcal{D}_t(S_t)$,
which contains the stopping vertices of all feasible macro-path segments.
These stopping vertices are either vertices adjacent to ambiguous obstacles or the goal $g$.
Constructing $\mathcal{D}_t$ at each decision step using an exhaustive search by examining all obstacles and adjacent vertices
quickly becomes computationally prohibitive as the number of obstacles increases. 
To overcome this challenge, we introduce a heuristic approach that: 
(i) identifies a set of decisions associated with prioritized obstacles using an optimistic greedy search, 
and (ii) reduces the search space by discarding paths unlikely to be optimal.

\subsubsection{Decision Structure}
At decision step $t$,
the decision set including all possible decisions (i.e., stopping vertices can be reached via obstacle-free path from $v_t$) is
\[
\mathcal{D}_t(S_t)
=\left\{\, v\in\mathcal{V}(\mathcal{G}): v=g\ \text{or}\ v\text{ is adjacent to }x\in X_t^{U}, \text{ with an obstacle-free path from } v_t \right\}.
\]
Implementing $d_t\in\mathcal{D}_t(S_t)$ moves the agent from $v_t$ to the chosen stopping vertex,
incurring the segment length. 
If the stopping vertex is adjacent to ambiguous obstacles,
the agent may disambiguate at additional costs, and collect readings of obstacles within the sensing range.

\subsubsection{Optimistic Candidate Generation and Soundness}
To construct $\mathcal{D}_t$, two key questions should be addressed:
which ambiguous obstacles should be prioritized for disambiguation and sensing,
and at which adjacent vertices (referred to as the disambiguation vertex) should this be implemented?
We use an optimistic greedy approach that iteratively identifies candidates.

The process of determining the decision set is outlined in Algorithm \ref{alg:decision_set_determine}, 
starting by assuming all ambiguous obstacles are traversable, $\tilde{X}^F=X^F\cup X^U$ (i.e., the optimism assumption). 
For $i^{th}$ iteration, we find the minimum cost path under current assumptions.
If the path reaches $g$ directly, we add $g$ into $\mathcal{D}_t$ and terminate.
Otherwise, we identify the first ambiguous obstacle $x_i$ intersected and add its disambiguation vertex $v(x_i)$ to $\mathcal{D}_t$,
where $v(x_i)$ is the vertex adjacent to the edge first intersecting the obstacle boundary.
Then, we update the assumptions by treating $x_i$ as blocked
(i.e., $\tilde{X}^{F}=\tilde{X}^{F}\backslash\{x_i\}$, $\tilde{X}^{O}=\tilde{X}^{O}\cup\{x_i\}$). 
This process ensures $g$ is always included and prioritizes obstacles with small traversal cost.

%\subsubsection{Macro-Path and Decision Set Construction}
%Each decision $d_t\in\mathcal{D}_t(S_t)$ moves the agent from $v_t$ to a location $v_{t+1}$ via a macro-path segment.
%Following the sensing and disambiguation (if needed), the state transitions from $S_t$ to $S_{t+1}$. 
%The key question is:
%which ambiguous obstacles should be prioritized for disambiguation and sensing,
%and at which adjacent vertices (referred to as the disambiguation vertex) should this be implemented?
%We use an optimistic greedy approach that iteratively identifies candidates to be included in $\mathcal{D}_t(S_t)$.

%The process of determining the decision set is outlined in Algorithm \ref{alg:decision_set_determine}, 
%starting by assuming all ambiguous obstacles are traversable, $\tilde{X}^F=X^F\cup X^U$ (i.e., the optimism assumption). 
%For $i^{th}$ iteration, we find the minimum cost path under current assumptions.
%If the path reaches $g$ directly, we add $g$ into $\mathcal{D}_t$ and terminate.
%Otherwise, we identify the first ambiguous obstacle $x_i$ intersected and add its disambiguation vertex $v(x_i)$ to $\mathcal{D}_t$,
%where $v(x_i)$ is the vertex adjacent to the edge first intersecting the obstacle boundary.
%Then, we update the assumptions by treating $x_i$ as blocked
%(i.e., $\tilde{X}^{F}=\tilde{X}^{F}\backslash\{x_i\}$, $\tilde{X}^{O}=\tilde{X}^{O}\cup\{x_i\}$). 
%This process ensures $g$ is always included and prioritizes obstacles with small traversal cost.

\begin{algorithm}[ht]
	\SetAlgoLined
	\KwIn{Current state $S_t = (\mathcal{V}_t, B_t)$, goal vertex $g$}
	\KwOut{Decision set $\mathcal{D}_t$}
	\textbf{Initialize:} $\mathcal{D}_t \leftarrow \emptyset$, $\tilde{X}^F \leftarrow X^F \cup X^U$, $\tilde{X}^O \leftarrow X^O$\;
	
	\Repeat{$g \in \mathcal{D}_t$}{
		\tcc{\textbf{Path Selection}}
		Find minimum-cost path $p^*$ from $v_t$ to $g$ under $\{\tilde{X}^F, \tilde{X}^O\}$\;
		\eIf{\text{no ambiguous obstacle intersects} $p^*$}{
			$\mathcal{D}_t \leftarrow \mathcal{D}_t \cup \{g\}$\;
			\textbf{break}\;
		}{
			Identify the first ambiguous obstacle intersected $x_i\in X^{U}$\;
			Let $v(x_i)$ be the disambiguation vertex just before $x_i$\;
			%Compute heuristic $\mathcal{H}(x_i) \leftarrow \ell_p + c_p$\;
			
			\tcc{\textbf{Decision Set Update}}
			$\mathcal{D}_t \leftarrow \mathcal{D}_t \cup \{v(x_i)\}$\;
			$\tilde{X}^F \leftarrow \tilde{X}^F \setminus \{x_i\}$, $\tilde{X}^O \leftarrow \tilde{X}^O \cup \{x_i\}$\;
		}
	}
	\caption{Identification of Decision Set}
	\label{alg:decision_set_determine}
\end{algorithm}

%While including more candidates provides more information to enhance the decision-making process, 
%it may lead to over-exploration and waste of computational resources \citep{macdonald2020reactive}. 
%The following proposition (Proposition \ref{prop: decision set}) ensures that our approach (Algorithm \ref{alg:decision_set_determine}) effectively avoids over-exploration 
%and stops when no further improvement can be achieved over a pure exploitation decision (i.e., the best path avoiding all ambiguous and true obstacles based on current knowledge).
%Define a candidate in the decision set $\mathcal{D}_t$ to be a valid candidate if it can potentially achieve improvement over pure exploitation decision, we show that:

This iterative optimistic-greedy procedure adds only candidates with potential improvement over a pure exploitation decision, 
therefore avoiding over-exploration and wasted computation \citep{macdonald2020reactive}. 
The formal soundness guarantee is established in the  Proposition~\ref{prop:decision-set-soundness} (Section \ref{sec:decision_sound}).

\subsubsection{Search Space Reduction}
To further improve efficiency, 
we reduce the search space by discarding obstacles and their associated paths whose best performance cannot exceed the current best solution.
Combining the decision set construction with search space reduction strategy,
we minimize computation cost while preserving only the most promising traversal options.
This enables the optimal policy to scale effectively to larger and more complex environments.

To establish the condition for discarding each obstacle, we compare two values: 
(i) the upper bound of the optimal path and 
(ii) the lower bound of the best path reaching the goal via the obstacle. 
If the lower bound exceeds the upper bound, 
then no path via this obstacle can beat the current best path, 
indicating the obstacle and its associated paths can be excluded from consideration.

We first retain all obstacles identified by Algorithm \ref{alg:decision_set_determine},
and the final shortest obstacle-free path identified can provide an upper bound on the optimal path cost. 
The reduction algorithm (Algorithm \ref{alg:search_space_reduce}) then starts with evaluating the remaining obstacles. 
For each remaining obstacle, we calculate a lower bound on the path cost passing through it. 
Since each obstacle contains multiple interior vertices, 
finding the best path through this obstacle theoretically leads to the exhaustive search. 
Instead, we introduce a pseudo vertex to facilitate the calculation. 
Let $\mathcal{V}_{\text{in}} = \{ v \in \mathcal{V} : v \in \text{interior}(x) \}$ for obstacle $x$, 
where the interior is approximated using a radius $\text{radius}(x)$ (e.g., derived from geometric or sensor range assumptions). 
A pseudo-vertex $\tilde{v}$ is then introduced, connected to each $v \in \mathcal{V}_{\text{in}}$ via zero-cost edges. 
The lower bound becomes the concatenation of the minimum cost obstacle-free path from current vertex to this pseudo vertex and the minimum cost path under optimism assumption from this pseudo vertex to $g$. 
Then obstacles whose lower bound exceeds the upper bound and the associated paths are discarded.

\begin{algorithm}[H]
	\SetAlgoLined
	\KwIn{State $S_t = (\mathcal{V}_t, B_t)$, goal vertex $g$, decision set $\mathcal{D}_t$, associated obstacles $X(\mathcal{D}_t)$, upper bound $\mathcal{C}_U$}
	\KwOut{Discarded obstacle set $X_{\text{discard}}$}
	\textbf{Initialize:} $X_{\text{discard}} \leftarrow \emptyset$, $\tilde{X}^F \leftarrow X^F \cup X^U$, $\tilde{X}^O \leftarrow X^O$\;
	
	\ForEach{$x \in X \setminus (X^O \cup X(\mathcal{D}_t))$}{
		Identify interior vertices: $\mathcal{V}_{\text{in}} \leftarrow \{v \in \mathcal{V} : \|v - x\| \leq \text{radius}(x)\}$\;
		Introduce pseudo-vertex $\tilde{v}$ and connect to all $v \in \mathcal{V}_{\text{in}}$ via zero-cost edges\;
		
		\tcc{Compute lower bound cost via obstacle $x$}
		$\mathcal{C}_1 \leftarrow \min_{p \in \mathcal{P}_{v_t, \tilde{v}}: p \cap X^U = \emptyset} \ell_p$\;
		$\mathcal{C}_2 \leftarrow \min_{p \in \mathcal{P}_{\tilde{v}, g}} (\ell_p + \mathcal{C}_p)$\;
		
		\If{$\mathcal{C}_1 + \mathcal{C}_2 \geq \mathcal{C}_U$}{
			$X_{\text{discard}} \leftarrow X_{\text{discard}} \cup \{x\}$\;
		}
	}
	\caption{Search Space Reduction}
	\label{alg:search_space_reduce}
\end{algorithm}
\noindent
\textbf{How Algorithm \ref{alg:search_space_reduce} Prunes the Search Space.}
Algorithm \ref{alg:search_space_reduce} further reduces the planning complexity 
by eliminating obstacles that provably cannot contribute to lower-cost traversal paths. 
These obstacles are excluded from future disambiguation, simulation, 
and policy considerations, thereby reducing the graph size, candidate action set, 
and overall computational burden. 
This pruning step is especially valuable in the SCOS framework, 
where the number of ambiguous obstacles can be large, each disambiguation is costly, 
and evaluating traversal paths under correlated uncertainty is computationally intensive. 
While Algorithm \ref{alg:decision_set_determine} identifies the most promising obstacles for immediate disambiguation, 
Algorithm~\ref{alg:search_space_reduce} complements this by removing the remaining obstacles 
that offer no potential benefit, resulting in a streamlined and more tractable decision process.

\subsubsection{Soundness of the Decision Set}\label{sec:decision_sound}
We work on a finite, undirected graph $G=(V,E)$ with strictly positive edge lengths.
Let $v_t\in V$ be the current vertex and $g\in V$ the goal.
Let $X^O$ denote the set of (known) true obstacles and $X^U$ the set of ambiguous obstacles.
For each ambiguous obstacle $x\in X^U$, fix a deterministic tie-breaking rule that selects a \emph{disambiguation vertex} $v(x)$ whenever multiple choices are possible (e.g., the vertex encountered along the chosen optimistic shortest path).

\begin{definition}[Pure exploitation baseline]\label{def:exploitation}
	Let $\mathcal{C}_{\mathrm{exploit}}(v_t)$ be the length of a shortest path from $v_t$ to $g$ that avoids all true and ambiguous obstacles (i.e., a path feasible with respect to $X^O\cup X^U$); set $\mathcal{C}_{\mathrm{exploit}}(v_t)=+\infty$ if no such path exists.
\end{definition}

\begin{definition}[Per-obstacle optimistic lower bound]\label{def:lowerbound}
	For $x\in X^U$, define
	\[
	\underline{\mathcal{C}}(x)\ :=\ \ell_{\mathrm{free}}(v_t,x)\ +\ c(x)\ +\ \mathcal C_{\mathrm{optimism}}(x,g),
	\]
	where $\ell_{\mathrm{free}}(v_t,x)$ is the length of a shortest obstacle-free path segment from $v_t$ to the disambiguation vertex $v(x)$ (i.e., a segment that intersects neither true nor ambiguous obstacles), $c(x)$ is the disambiguation cost at $x$, and $\mathcal C_{\mathrm{optimism}}(x,g)$ is the shortest-path cost from $v(x)$ to $g$ under the \emph{optimism assumption} that all ambiguous obstacles are traversable (true obstacles remain forbidden).
\end{definition}

\begin{definition}[Optimistic evaluation and minimizer]\label{def:opteval}
	For any $v\in V$, define the \emph{optimistic evaluation} of starting at $v$ by
	\[
	\mathrm{OptEval}(v)\ :=\ \min\!\Big\{\, \mathcal C_{\mathrm{optimism}}(v,g)\,,\ \min_{x\in X^U}\ \underline{\mathcal{C}}(x)\,\Big\}.
	\]
	Equivalently: either reach $g$ without touching any ambiguous obstacle (first term), or first disambiguate some $x$ and then proceed optimistically (second term). An \emph{optimistic-evaluated minimizer} is any path that attains $\mathrm{OptEval}(v_t)$ and, if it touches an ambiguous obstacle, let $x$ denote the first such obstacle on that path.
\end{definition}

\begin{lemma}[Optimistic evaluation decomposition]\label{lem:decomposition}
	Let $P^\star$ be an optimistic-evaluated minimizer from $v_t$.
	If $P^\star$ reaches $g$ without meeting any ambiguous obstacle, then $\mathrm{OptEval}(v_t)=\mathcal C_{\mathrm{optimism}}(v_t,g)$.
	Otherwise, if $x$ is the first ambiguous obstacle met along $P^\star$, then
	\[
	\mathrm{OptEval}(v_t)\ =\ \ell_{\mathrm{free}}(v_t,x)\ +\ c(x)\ +\ \mathcal C_{\mathrm{optimism}}(x,g)\ =\ \underline{\mathcal{C}}(x).
	\]
\end{lemma}
\begin{proof}
	By definition of $\mathrm{OptEval}$, either the no-ambiguity option is optimal, or some $x$ is optimal to disambiguate first.
	In the latter case, the prefix to $v(x)$ is obstacle-free by definition of “first ambiguous,” so its length is $\ell_{\mathrm{free}}(v_t,x)$; at $v(x)$ we incur $c(x)$; and the suffix from $v(x)$ to $g$ is a shortest path under optimism (optimal substructure of shortest paths with positive edge lengths). Concatenation yields $\underline{\mathcal{C}}(x)$.
\end{proof}

\begin{lemma}[Optimistic evaluation vs.\ exploitation]\label{lem:order}
	We have
	\[
	\mathrm{OptEval}(v_t)\ \le\ \mathcal{C}_{\mathrm{exploit}}(v_t).
	\]
\end{lemma}
\begin{proof}
	Consider an exploitation-shortest path (Definition~\ref{def:exploitation}): it avoids all ambiguous obstacles, hence it is also feasible under optimism, and its optimistic evaluation equals its length $\mathcal{C}_{\mathrm{exploit}}(v_t)$ (no $c(x)$ is incurred). Taking the minimum over all options in $\mathrm{OptEval}(v_t)$ cannot exceed this value.
\end{proof}

\begin{proposition}[Soundness of the optimistic candidate set]\label{prop:decision-set-soundness}
	Let $\mathcal{D}_t$ be the decision set built by Algorithm~\ref{alg:decision_set_determine}.
	Then every obstacle $x$ whose disambiguation vertex is added to $\mathcal{D}_t$ satisfies
	\[
	\underline{\mathcal{C}}(x)\ \le\ \mathcal{C}_{\mathrm{exploit}}(v_t),
	\]
	and, if $g$ is added, it is trivially valid. In particular,
	\[
	\mathcal D_t\ \subseteq\ \{\,g\,\}\ \cup\ \bigl\{\,v(x)\ :\ x\in X^U,\ \underline{\mathcal{C}}(x)\le \mathcal{C}_{\mathrm{exploit}}(v_t)\,\bigr\}.
	\]
\end{proposition}
\begin{proof}
	By Lemma~\ref{lem:decomposition}, in any iteration that adds a disambiguation vertex for obstacle $x$, the evaluation value equals $\underline{\mathcal{C}}(x)$. By Lemma~\ref{lem:order}, $\mathrm{OptEval}(v_t)\le \mathcal{C}_{\mathrm{exploit}}(v_t)$, hence $\underline{\mathcal{C}}(x)\le C_{\mathrm{exploit}}(v_t)$ for that $x$. If the minimizer reaches $g$ without touching an ambiguous obstacle, we add $g$ and terminate. Repeating the argument per iteration yields the stated inclusion.
\end{proof}

The test $\underline{\mathcal{C}}(x)\le \mathcal{C}_{\mathrm{exploit}}(v_t)$ formalizes “potential improvement over pure exploitation.”
Proposition~\ref{prop:decision-set-soundness} shows \emph{soundness} (no invalid candidates are added).
The procedure is generally \emph{not complete}: there may exist obstacles with $\underline{\mathcal{C}}(x)\le \mathcal{C}_{\mathrm{exploit}}(v_t)$ that are not the first ambiguous obstacle on any optimistic-evaluated minimizer and hence are not added.

Strictly positive edge lengths preclude zero-cost cycles and ensure shortest-path problems are well-defined (ties broken deterministically).
If $\mathcal{C}_{\mathrm{exploit}}(v_t)=+\infty$ (no exploitation path exists), the inequality $\underline{\mathcal{C}}(x)\le \mathcal{C}_{\mathrm{exploit}}(v_t)$ holds vacuously for any $x$ with finite $\underline{\mathcal{C}}(x)$, and the proof remains valid.

\subsection{Information-Guided Exploration Strategy}
As the agent traverses the environment, collecting new information through sensor readings and disambiguation is crucial for making decisions.
However, not all information is equally valuable,
and we need to quantify the information value to guide more informed exploration.
We build on the mutual information \citep{chaloner1995bayesian} proposed for Bayesian experimental design problem. 
In the context of observations with Gaussian noise, the information gain can be calculated using \citep{srinivas2010gaussian}:
$$\textbf{I}(\tilde{Y}^t)=\frac{1}{2}\log\left|I+\sigma^{-2}\Sigma_v\right|,$$
where $\tilde{Y}^t$ is the random vector corresponding to potential log-odds observations can be collected,
$K_v$ is the associated covariance matrix and $\sigma^2$ is the noise variance. 
This quantity measures how much uncertainty is reduced by incorporating new observations into the decision-making process. 

In most cases, gaining information is more critical during early stages of the traversal, 
and as the agent approaches the goal $g$, 
the influence of new information decreases,
which is also justified by the submodularity of information gain (see Proposition~\ref{prop:MI-submod}).
To reflect this decreasing influence, 
we adapt the method from \cite{contal2014gaussian} to compute the information gain $G(d)$ at time step $t$ of implementing any decision $d$ as:
$$G(d)=\sqrt{\gamma}\left(\sqrt{\textbf{I}(\tilde{Y}^t)+\sum_{i=1}^{t-1}\textbf{I}(\tilde{Y}^i)}-\sqrt{\sum_{i=1}^{t-1}\textbf{I}(\tilde{Y}^i)}\right),$$
where $\gamma$ is the scaling factor controlling the weight of information gain in decision-making. 
We set $\gamma$ proportionally to the empirical standard deviation of the estimated value function across the decision set $\mathcal{D}_t$,
ensuring information gain and value estimates are on similar scale.  

\subsection{Posterior Sampling of Environment Models}
In our problem setting, the uncertainty related to the transition function (i.e., obstacle blockage statuses) requires 
additional step to address its associated variability throughout the navigation.
Unlike traditional threshold-based methods that categorize obstacles using fixed cut-off probabilities \citep{koenig2002d, bnaya2009canadian} or
prior sampling approaches that assume either static probabilistic information or a small subset of environment status as a prior,
we use posterior sampling that evolves with accumulated observations.
This approach provides environment models for both offline base policy learning and online decision making stages,
offering several advantages over existing methods: 
(i) it automatically balances exploration and exploitation as uncertainty changes over time, 
(ii) it preserves spatial correlation information throughout the sampling process, and 
(iii) it eliminates the need for manual threshold tuning.

Specifically, using the posterior distribution derived as in Section \ref{sec: updating}, 
we sample log-odds, which are modeled as latent variables with correlation structure, 
from $f(\mathbf{y}|\tilde{\rho}^1, ..., \tilde{\rho}^t)$ at time $t$.
These log-odds are then transformed into probabilities using $\rho_i=\frac{\exp(y_i)}{1 + \exp(y_i)}$. 
The obstacle status for each location is determined using binary sampling according to probabilities $\rho_i$.
And $\tilde{X}=\tilde{X}^O\cup \tilde{X}^F$ is updated accordingly and helps to form the predicted traversal region $\tilde{\mathcal{G}}$ for interaction simulation.
The process is outlined in Algorithm \ref{alg:sampling_method}. 

\begin{algorithm}[ht]
	\SetAlgoLined
	\KwIn{$f(\mathbf{y}|\tilde{\rho}^1, ..., \tilde{\rho}^t)$}
	\KwOut{Predicted traversal region $\tilde{\mathcal{G}}$}
	Sample $(y_1, y_2, ..., y_n)\sim f(\mathbf{y}|\tilde{\rho}^1, ..., \tilde{\rho}^t)$\;
	Transform $\rho_i=\frac{\exp(y_i)}{1+\exp(y_i)}$, $i=1, 2, ..., n$\;
	Sample $z_i \sim \text{Bernoulli}(\rho_i)$ for each i=1,2,...,n, where $z_i=1$ indicates a true obstacle\;
	Update $\tilde{X}^O=\{x_i: z_i=1\}$, $\tilde{X}^F=\{x_i: z_i=0\}$\;
	Update traversal region $\tilde{\mathcal{G}}$ using $\tilde{X}=\tilde{X}^O\cup \tilde{X}^F$
	\caption{Posterior Sampling for Environment Model Generation}
    \label{alg:sampling_method}
\end{algorithm}

Together, these three components enable our two-stage policy learning framework to develop robust base policies efficiently,
and to maintain an online rollout policy that continuously adapts to updated beliefs.

\section{Offline Base Policy Learning Guided by Information}\label{sec:offline}
The offline stage aims to learn a high-quality base policy $\pi^{\text{base}}$ that
provides estimates of the expected future traversal cost following state $S_{t+1}$, 
denoted as the optimal state value function $V^*(S_{t+1})$:
$$V^*(S_{t+1})=\min_{\{d_{t'}, t' = t+1,...,T\}}\E\left(\sum_{t'=t+1}^{T}\mathcal{C}(S_{t'},d_{t'})|S_{t+1}\right)
\approx\E\left(\sum_{t'=t+1}^{T}\mathcal{C}(S_{t'},d^{\text{base}}(S_{t'}))|S_{t+1}\right)$$ 
We propose an optimistic policy iteration (OPI) framework augmented with information bonus with three key features:
(i) quantifying information gain to encourage targeted exploration,
(ii) integrating model-based posterior sampling with model-free value learning process, and
(iii) supporting both point estimate and distribution learning approaches for stronger uncertainty quantification.

\subsection{Optimistic Policy Iteration with Information Integration}
Policy iteration (PI) is a fundamental framework in RL \citep{sutton2018reinforcement}. 
Its core idea is the alternating implementation between two key steps: policy evaluation and policy improvement. 
The policy evaluation step involves generating interaction experience under a fixed policy. 
The resulting observations are used to update the state value functions.
In the policy improvement step, the updated value function estimates are used to define a better policy, supporting further refinement.
This alternating process drives the estimated value functions $V(s)$ towards their optimal values $V^*(s)$.

Unlike traditional PI that requires full convergence of value function estimates before proceeding to the policy improvement, 
our base policy learning uses the optimistic policy iteration (OPI) framework that allows partial updates, making it more computationally efficient and adaptive.
Applying it in our problem setting, we perform a single round of value function updates before policy improvement.
The key innovation in our approach is the direct incorporation of information gain built upon mutual information,
and modify the Bellman optimality equation into:
$$V^*(s)=\min_{d\in\mathcal{D}(s)}\mathcal{C}(s,d)+\E\left[V^*(s')|s,d\right]-G(d),$$
where $s'$ represents the resulting state after implementing $d$ at state $s$.
Subtracting $G(d)$ effectively rewards decisions that provide valuable information,
encouraging exploration of high-information paths during the learning process.
The theoretical advantage is established in Section \ref{sec: offline_property}.
For notational simplicity, in all subsequent sections
we denote $\mathcal{C}^{\scriptscriptstyle IG}(s,d)$ as the cost after adjusting for $G(d)$.

\subsection{Model-Based and Model-Free Reinforcement Learning Structure}
Within this information-guided OPI framework, 
we adopt a hybrid structure.
During policy evaluation, the model-based component uses posterior sampling to generate probabilistic environment models that guide exploration.
The model-free component complements this by directly learning value functions from simulated experiences, 
refining estimates while adapting to sequentially updated posterior distributions. 
The interplay of these two components creates a robust learning environment, 
effectively addressing uncertainty in obstacle status (i.e., blockage probabilities).

Following value function updates, 
an improved policy minimizing expected traversal cost is defined in the policy improvement step, completing one iteration. 
This process repeats until convergence. 
We adopt a partial convergence approach focusing on stabilizing value estimates only for states in current decision set, 
rather than requiring full convergence across the entire state space.
This approach is similar to value iteration with restricted backups, a technique commonly used in approximate dynamic programming. 
It not only accelerates the overall convergence of the process,
but also ensures that value estimates for non-converged states serve as effective initializations for the subsequent online rollout phase.

Our approach aligns with Dyna-based hybrid RL \citep{sutton1991dyna, silver2008sample}, 
where the model-based component provides dynamic information to enhance model-free learning and improve sampling efficiency.
By incorporating posterior sampling through Bayesian updating,
we emphasize a systematic approach for handling uncertainty and optimizing exploration-exploitation trade-off, 
resulting in more robust base policy learning.
The steps are detailed in Algorithm \ref{alg: OPI}.

\begin{algorithm}[ht]
	\SetAlgoLined
	\KwIn{Policy $\pi_0$, value function estimate $V_0(s)$, posterior distribution $f(\mathbf{y}|\mathbf{\hat{y}})$, convergence threshold $\eta$}
	\KwOut{Estimates of optimal value function $V^*(s)$}
	\textbf{Initialize:} $i\gets1$;
	
	\While{not $\text{converged}$}{
		\tcc{\textbf{Policy Evaluation}}
		Apply Algorithm \ref{alg:sampling_method} to sample environment models from posterior distribution $f(\mathbf{y}|\mathbf{\hat{y}})$\;
		Simulate experience under $\pi_i$ and update $V_{i+1}(s)$ using $\mathcal{C}^{\scriptscriptstyle IG}(s,d)$\;
		
		\tcc{\textbf{Policy Improvement}}
		Define $\pi_{i+1}$ using $V_{i+1}(s)$\;
		
		\tcc{\textbf{Convergence Check}}
		Compute the difference between $V_{i+1}(s)$ and the updated value function $V_i(s)$\;
		\If{$\max_{s}|V_{i+1}(s) - V_i(s)| < \eta$}{Return $V^*(s)\gets V_{i+1}(s)$}
	}
	\caption{Information-Guided OPI for Base Policy Learning}
	\label{alg: OPI}
\end{algorithm}

\paragraph{Convergence Property:}
Our problem is an undiscounted, finite horizon problem,
and prior work has established convergence results for OPI using single round value function updates per iteration under static probabilistic information
for such problems (\cite{tsitsiklis2002convergence}, \cite{chen2018convergence}, \cite{winnicki2023convergence}).
However, these results do not directly cover our proposed method due to the incorporation of posterior sampling. 
In Section \ref{sec: offline_property}, we prove that the convergence can still be established in a similar way, 
complementing with extensive empirical results presented in Section \ref{sec: simulations}. 

\subsection{The Model-Free Component: Value Function Approximation}
The model-free component of our offline learning process directly estimates the optimal state value functions from
simulated experience.
We explore two distinct strategies to generate robust and efficient value estimation:
(i) Monte Carlo approach for point estimates and 
(ii) distributional reinforcement learning (distributional RL) approach for estimating full distributions, 
capturing not only the mean but also the uncertainty.

\subsubsection{Monte Carlo Approximation}
The Monte Carlo (MC) approach estimates value function by using cumulative observed costs from simulated trajectories.
In a sampled traversal environment, with a given policy, 
a complete trajectory is simulated from a starting state randomly selected from decision set $\mathcal{D}(S)$.
Denote the cumulative traversal costs with information adjustment from any state $s$ on this trajectory in $i^{th}$ iteration as 
$\mathcal{C}^{\scriptscriptstyle IG}_{i}(s)$, the value function is updated using the following rule:
\begin{equation} \label{equation: MC update}
	V_i(s) = (1 - \alpha_i) V_{i-1}(s) + \alpha_i \mathcal{C}^{\scriptscriptstyle IG}_i(s), \quad \text{with } \alpha_i = \frac{1}{N_i(s)} 
\end{equation}
where $N_i(s)$ represents the total number of times state $s$ has been experienced in all simulations up to $i^{th}$ iteration.
For efficiency, we simultaneously update the value function for all states experienced along the trajectory.

After updating the state value function, we obtain a better policy in the policy improvement step to move towards the optimal direction. 
Combining with the Monte Carlo method, we consider three strategies for comparison. 
\paragraph{Greedy method.}
The greedy method is most commonly used, which always selects the decision that minimizes expected cost based on current value estimates:
$$d_{\text{greedy}} = \arg\min_{d \in \mathcal{D}(s)} \left\{ \mathcal{C}^{\scriptscriptstyle IG}(s,d) + V_i(s') \right\}.$$

\paragraph{Decaying $\epsilon$-greedy method.}
Due to the uncertainty related to single trajectory simulation and transition probability sampling, we consider the $\epsilon$-greedy strategy. 
It introduces controlled exploration by selecting random decisions with probability $\epsilon$ while following the greedy option with probability $1-\epsilon$:
\[
d_{\epsilon} =
\begin{cases}
	d_{\text{greedy}}, \text{with probability } 1 - \epsilon, \\
	d\in\mathcal{D}(s)\backslash d_{\text{greedy}}, \text{with probability } \frac{\epsilon}{|\mathcal{D}(s)\backslash d_{\text{greedy}}|}.
\end{cases}
\]
Specifically, we gradually decrease $\epsilon$ after each decision step to reflect that exploration is more critical during early stages of the traversal.

\paragraph{Softmax exploration method.}
Although decaying $\epsilon$-greedy strategy encourages exploration, it treats all non-greedy decisions equally. 
Hence, we consider the softmax strategy using Boltzmann distribution, assigning the selection probability proportional to the updated value function. 
This strategy maintains relative preferences among decisions as below:
\[
P(d \mid s) = \frac{\exp\left(-\beta \left(\mathcal{C}^{\scriptscriptstyle IG}(s, d) + V_i(s')\right)\right)}{\sum_{d' \in \mathcal{D}(s)} \exp\left(-\beta \left(\mathcal{C}^{\scriptscriptstyle IG}(s, d') + V_i(s')\right)\right)},
\]
where we set $\beta=1$ as a default parameter. 

The process of using Monte Carlo approach for estimating value function and updating policy is outlined in Algorithm \ref{alg:OPI_MC}.

\begin{algorithm}[ht]
	\SetAlgoLined
	\KwIn{Traversal region $\tilde{\mathcal{G}}$, policy $\pi_i$, value estimates $V_i(s)$ for all $s$}
	\KwOut{Improved policy $\pi_{i+1}$, updated value estimates $V_{i+1}(s)$}
	\tcc{\textbf{Trajectory Simulation}}
	Start from a random state $s_0 \in \mathcal{D}$\;
	\While{trajectory not terminated}{
		Take action $d$ according to $\pi_i$ and observe $\mathcal{C}^{\scriptscriptstyle IG}(s,d)$, next state $s'$\;
		$s \leftarrow s'$\;
	}
	\tcc{\textbf{Value Update}}
	\ForEach{\text{state} $s$ \text{along trajectory}}{
		Accumulate $\mathcal{C}^{\scriptscriptstyle IG}_i(s)$ for each visited $s$\;
		$V_{i+1}(s) \leftarrow (1-\alpha)V_i(s) + \alpha\mathcal{C}^{\scriptscriptstyle IG}_i(s)$\;
	}
	\tcc{\textbf{Policy Improvement}}
	\uIf{selection is greedy}{
		$\pi_{i+1}(s) \gets \arg\min_{d \in \mathcal{D}(s)} \,\bigl\{ \mathcal{C}^{\scriptscriptstyle IG}(s,d)+V_{i+1}(s')\bigr\}$.
	}
	\uElseIf{selection is $\epsilon$-greedy}{
		% mention random w. prob e
		$\pi_{i+1}(s)$ chooses a random $d$ w.p. $\epsilon$, else the greedy decision.  
	}
	\uElseIf{selection is softmax}{
		$\pi_{i+1}(s)$ selects $d$ w.p. proportional to $\exp\bigl(-\beta[\mathcal{C}^{\scriptscriptstyle IG}(s,d)+V_{i+1}(s')]\bigr)$. 
	}
	\caption{Monte Carlo approach for value function and policy update}
	\label{alg:OPI_MC}
\end{algorithm}

\subsubsection{Distributional RL Approach}\label{sec:DRL}
While Monte-Carlo approach generates point estimates (i.e., expected values) of $V(s)$,
it does not capture the inherent uncertainty in estimations arising from transition and posterior sampling.
To handle this uncertainty,
we propose a distributional RL learning approach, a more robust method capturing the full distributions of $V(s)$.
Moreover, the exploration strategies paired with point estimate methods often require careful exploration parameter tuning, 
which may lead to inconsistent performance compared to distribution-based approaches like posterior sampling \citep{osband2013more, osband2017posterior}.

Using the distributional Bellman equation, \cite{bellemare2017distributional} introduced a distributional reinforcement learning (distributional RL) algorithm using categorical distribution, called C51 algorithm. 
This algorithm approximates the discrete distribution of $V(s)$, demonstrating significant improvements over point estimate approaches. 
Expanding on this foundation, \cite{dabney2018distributional} proposed Quantile Regression DQN (QR-DQN), 
which utilizes quantile regression to represent continuous distributions, enhancing the flexibility of value function representation. 
Focusing on addressing exploration challenges, 
\cite{mavrin2019distributional} designed an exploration bonus using quantile-based distributional RL techniques,
while \cite{tang2018exploration} unified posterior sampling with distributional RL to create an exploration framework, shown to improve exploration performance in policy learning. 

Given that we have a finite number of obstacle statuses in the environment, 
the value function $V(s)$ can only take a finite number of possible values. 
Therefore, we adopt the categorical distribution to represent the probability distribution over $V(s)$. 
Extending the approach of \cite{tang2018exploration}, we propose a framework using Bayesian updates to learn the distribution of $V(s)$. 
Posterior sampling from sequentially updated distributions drives the policy improvement, 
which is demonstrated to be more stable and efficient compared to the classic exploration strategies described in the previous section. 
The complete process is outlined in Algorithm \ref{alg: DRL}.

\paragraph{Bayesian update.}
Let $V(s)$ be a random variable following a categorical distribution over finite support set $\mathcal{X}=\{x_1,\dots,x_k\}$:
$$V(s) \sim \text{Cat}(\mathbf{p}) \quad \text{where } \mathbf{p} = (p_1, \dots, p_k),$$
and $p_i=p(V(s)=x_i), i=1, 2, ..., k$ represents the probability that $V(s)$ takes the value $x_i\in\mathcal{X}$.
We place a Dirichlet prior over probabilities:
$$\mathbf{p}=\left(p_1, p_2, ..., p_k\right)\sim \text{Dir}(\boldsymbol{\alpha}=\left(\alpha_1, \alpha_2, ...\alpha_k\right)),$$
where $\alpha_1=\alpha_2=\dots=\alpha_k=1$ corresponds to a uniform prior, indicating no preference among the possible values.
Given observed costs from state $s$ over the first $i$ iterations, 
$\{\mathcal{C}^{\scriptscriptstyle IG}_1(s), \mathcal{C}^{\scriptscriptstyle IG}_2(s), ..., \mathcal{C}^{\scriptscriptstyle IG}_i(s)\}$, 
the posterior of $\mathbf{p}$ is derived using Bayes' rule:
$$f(\mathbf{p}|\{\mathcal{C}^{\scriptscriptstyle IG}_1(s), \mathcal{C}^{\scriptscriptstyle IG}_2(s), ..., \mathcal{C}^{\scriptscriptstyle IG}_i(s)\})\propto f(\mathbf{p};\boldsymbol{\alpha})\prod_{i}f(\mathcal{C}^{\scriptscriptstyle IG}_i(s)|\mathbf{p}).$$ 
By conjugacy, the resulting posterior remains Dirichlet:
\begin{equation} \label{equation: posterior-DRL}
	\mathbf{p} \mid \{ \mathcal{C}^{\scriptscriptstyle IG}_1(s), \dots, \mathcal{C}^{\scriptscriptstyle IG}_i(s) \}\sim \text{Dir}(\boldsymbol{\alpha}'=\left(\alpha'_1, \alpha'_2, ...\alpha'_k\right)),
\end{equation}
where $\alpha'_k=\alpha_k+\sum\mathbbm{1}\{\mathcal{C}^{\scriptscriptstyle IG}_i(s)=x_k\}$ with $\mathbbm{1}\{\cdot\}$ being the indicator function.

\paragraph{Recursive update using distributional Bellman equation.} 
Following a policy $\pi$ at state $s$, we observe immediate information-adjusted cost $\mathcal{C}^{\scriptscriptstyle IG}(s,d)$ and next state $s'$.
The core idea of distributional RL relies on the distributional Bellman equation of random variables $V(s)$:
$$V^{\pi}(s)=_{D}\mathcal{C}^{\scriptscriptstyle IG}(s,d)+V^{\pi}(s'),$$
where $d$ is determined by policy $\pi$.
Based on the distribution of $V^{\pi}(s')$, let $\mathcal{X}(V(s'))=\{x'_1, x'_2,... x'_{k'}\}$ denote the support of $V(s')$,
we update $V(s)$ by treating $\mathcal{C}^{\scriptscriptstyle IG}(s,d)+x'_i$, $i=1,2,\dots, k'$ as observations.
However, the shifted values $\mathcal{C}^{\scriptscriptstyle IG}(s,d)+\mathcal{X}(V(s'))$ often does not align with the support $\mathcal{X}(V(s))$. 
Following \cite{bellemare2017distributional}, we resolve this discrepancy using weighted interpolation. 
For each $x'_i \in \mathcal{X}(V(s'))$ with associated probability $p'_i$, 
we compute the shifted value $x'_i+\mathcal{C}^{\scriptscriptstyle IG}(s,d)$,
which falls in the interval $[x_j, x_{j+1}]$ (i.e., $x_j \le x'_i + \mathcal{C}^{\scriptscriptstyle IG}(s,d) < x_{j+1}$).
Then using Bayes updating rule we derived, we update the Dirichlet parameters for $V(s)$: 
$$\alpha_j=\frac{|x_{j+1}-(x'_i+\mathcal{C}^{\scriptscriptstyle IG}(s,d))|}{x_{j+1}-x_{j}}p'_i+\alpha_j, \; 
\alpha_{j+1}=\frac{|x_{j}-(x'_i+\mathcal{C}^{\scriptscriptstyle IG}(s,d))|}{x_{j+1}-x_{j}}p'_i+\alpha_{j+1}.$$

\paragraph{Support refinement through value contraction.}
To initialize the distribution of $V(s)$, we determine their upper and lower bounds using the method described in Section \ref{sec: identify Dt}, 
then discretize the interval into an equally spaced grid with spacing $\delta$ 
(i.e., $\mathcal{X}=\{x_\text{L}, x_\text{L}+\delta, x_\text{L}+2\delta, \dots, x_\text{U}\}$).
While practical and commonly used in distributional RL algorithms, 
this fixed discretization method does not guarantee accurate representation of the true value distribution. 
To address this shortcoming, we propose a support refinement step using a value contraction approach, 
which adds observed values or replaces the unobserved ones.

During the simulation, the distribution of $V(s)$ is updated iteratively using the distributional Bellman equation at each step. 
After completing an interaction trajectory, we refine the support of $V(s)$ using the observed information-adjusted costs.
Let $\mathcal{C}^{\scriptscriptstyle IG}_i(s)$ represent the cumulative information-adjusted cost from state $s$ observed during the $i^{th}$ iteration, 
we update the distribution support by locating neighbors of $\mathcal{C}^{\scriptscriptstyle IG}_i(s)$ in $\mathcal{X}(V(s))$.
Specifically, we identify two consecutive support points $x_j, x_{j+1}$ (i.e., $x_j\leq\mathcal{C}^{\scriptscriptstyle IG}_i(s)\leq x_{j+1}$) such that
$$|x_j-\mathcal{C}^{\scriptscriptstyle IG}_i(s)|\leq \delta, \; |x_{j+1}-\mathcal{C}^{\scriptscriptstyle IG}_i(s)|\leq \delta.$$
If these neighboring points have not been observed in previous trajectories, 
we replace them with $\mathcal{C}^{\scriptscriptstyle IG}_i(s)$ and update its probability distribution parameter:
$$\alpha_j=\alpha_j+\alpha_{j+1} +1.$$
Otherwise, we directly insert $\mathcal{C}^{\scriptscriptstyle IG}_i(s)$ into the support set and assign $\alpha_{j+1}=1$, since it is the first observation of that value.
If $\mathcal{C}^{\scriptscriptstyle IG}_i(s)$ is already included in the support, we update using the Bayesian update rule.
This refinement guarantees that the support of $V(s)$ aligns with the actual values which can be observed, enhancing the accuracy of distribution representation. 
Additionally, this support refinement procedure maintains stability in the mean estimate, with bounded changes that shrink as more observations accumulate, which is formally established in Lemma \ref{lem:refinement_stability} (Section \ref{sec: offline_property}). 

The theoretical analysis of distributional RL to our approximate Bayesian setting, and evaluation of distributional calibration (e.g., prediction interval coverage, CRPS)
remains an important direction for future investigation.

\begin{algorithm}[H]
	\SetAlgoLined
	\KwIn{Environment $\tilde{\mathcal{G}}$, initial policy $\pi_i$, value distribution $V_i(s)$ for all $s$}
	\KwOut{Improved policy $\pi_{i+1}$, updated value distribution $V_{i+1}(s)$}
	\tcc{\textbf{Trajectory Simulation}}
	Start from a random state $s_0 \in \mathcal{D}$;
	
	\While{trajectory not terminated}{
		%\tcc{\textbf{Simulation Step}}
		Take action $d \sim \pi_i(s)$, Observe $\mathcal{C}^{\scriptscriptstyle IG}(s,d)$ and next state $s'$\;
		
		\tcc{\textbf{Distributional Bellman Update}}
		Compute distribution $V_{i+1}(s)\gets\mathcal{C}^{\scriptscriptstyle IG}(s,d) + V_i(s')$\;
		Update $V_i(s)$ using weighted interpolation: $\alpha'_j=\frac{|x_{j+1}-(x'_i+\mathcal{C}^{\scriptscriptstyle IG}(s,d))|}{x_{j+1}-x_{j}}p'_i+\alpha_j, \; \alpha'_{j+1}=\frac{|x_{j}-(x'_i+\mathcal{C}^{\scriptscriptstyle IG}(s,d))|}{x_{j+1}-x_{j}}p'_i+\alpha_{j+1}$\;
		
		$s \gets s'$\;
	}
	
	\tcc{\textbf{Support Refinement}}
	\ForEach{\text{state} $s$ \text{along trajectory}}{
		Compute cumulative cost $\mathcal{C}^{\scriptscriptstyle IG}_i(s) = \sum_{t=0}^T \mathcal{C}^{\scriptscriptstyle IG}(s_t,d_t)$\;
		Identify neighbors $x_j, x_{j+1}$ in $\mathcal{X}(V_{i+1}(s))$: $|x_j - \mathcal{C}^{\scriptscriptstyle IG}_i(s)| \leq \delta,\; |x_{j+1} - \mathcal{C}^{\scriptscriptstyle IG}_i(s)| \leq \delta$\;
		\eIf{$x_j, x_{j+1}$ not observed and $\mathcal{C}^{\scriptscriptstyle IG}_i(s)\notin\mathcal{X}(V_i(s))$}{
			Replace $x_j, x_{j+1}$ with $\mathcal{C}^{\scriptscriptstyle IG}_i(s)$\;
			Update $\alpha_j \leftarrow \alpha_j + \alpha_{j+1} + 1$\;
		}{
			\eIf{$\mathcal{C}^{\scriptscriptstyle IG}_i(s)\in\mathcal{X}(V_i(s))$}{
				Update $\alpha_j$ using Equation (\ref{equation: posterior-DRL})\;
			}{
				$\alpha_j\gets1$
			}
		}
	}
	
	\tcc{\textbf{Policy Improvement}}
	\ForEach{\text{state} $s$}{
		Sample $\mathbf{p}$ from posterior distribution and compute $\E\left[V(s)\right]$\;
		%Compute expected value $\E(V(s))$\;
		$\pi_{i+1}\gets\argmin_d\mathcal{C}^{\scriptscriptstyle IG}(s,d)+\E\left[V(s)\right]$
	}
	
	\caption{Distributional RL with Bayesian update and support refinement for value function and policy update}
	\label{alg: DRL}
\end{algorithm}

\subsection{Offline Learning Framework Properties} \label{sec: offline_property}
In this section, we establish theoretical results to justify our offline learning strategy.
First, we demonstrate the benefit of incorporating information gain into value function updating process,
then show that the information gain should maintain diminishing impact on the decision making due to submodularity property.
We also prove the convergence of OPI under posterior sampling, and show that our support refinement procedure for distributional learning maintains stability with bounded, diminishing changes in mean estimates.

\subsubsection{Benefit of Nonnegative Bonus Shaping} \label{sec:nonnegativ-bonus}
We consider an undiscounted, episodic stochastic shortest path (SSP) MDP with finite state set $\mathcal S$ and finite feasible decision sets $\mathcal D(s)$.
There is an absorbing goal $g\in\mathcal S$ with $\mathcal C(g,d)=0$ and $P(g\mid g,d)=1$ for all $d\in\mathcal D(g)$.
A policy $\pi$ is \emph{proper} if the hitting time $\tau:=\inf\{t\ge0:S_t=g\}$ is finite a.s. from any start state.

Assume:
\begin{enumerate}[label=(A\arabic*),leftmargin=2em]
	\item \label{ass:proper} All admissible policies are proper, and base one–stage costs are nonnegative and bounded:
	$0\le \mathcal C(s,d)\le \overline{\mathcal{C}}<\infty$.
	\item \label{ass:G} The (exploration) bonus is bounded and nonnegative:
	$G:\{(s,d):d\in\mathcal D(s)\}\to[0,\infty)$ with $\sup_{s,d}G(s,d)<\infty$.
\end{enumerate}

%For a (history-dependent) policy $\pi$, define the \emph{base} and \emph{bonus-shaped} performance starting from $s$:
%\[
%J^{\pi}_{\mathrm b}(s)
%:=\mathbb E_{\pi}\!\Big[\textstyle\sum_{t=0}^{\tau-1}\mathcal C(S_t,D_t)\ \Big|\ S_0=s\Big],\qquad
%J^{\pi}_{*}(s)
%:=\mathbb E_{\pi}\!\Big[\textstyle\sum_{t=0}^{\tau-1}\big(\mathcal C(S_t,D_t)-G(S_t,D_t)\big)\ \Big|\ S_0=s\Big].
%\]
%Let $J^{\mathrm b}(s):=\inf_{\pi}J^{\pi}_{\mathrm b}(s)$ and $J^{*}(s):=\inf_{\pi}J^{\pi}_{*}(s)$ be the optimal values.
%(Under \ref{ass:proper}–\ref{ass:G}, these values are finite and coincide with the minimal nonnegative solutions of the corresponding optimality equations for SSP.)
%

For a (history-dependent) policy $\pi$, define the \emph{base} and \emph{bonus-shaped} performances from $s$:
\[
{\small
J^{\pi}_{\mathrm b}(s)
:=\mathbb E_{\pi}\!\Big[\textstyle\sum_{t=0}^{\tau-1}\mathcal C(S_t,D_t)\ \Big|\ S_0=s\Big],\quad
J^{\pi}_{*}(s)
:=\mathbb E_{\pi}\!\Big[\textstyle\sum_{t=0}^{\tau-1}\big(\mathcal C(S_t,D_t)-G(S_t,D_t)\big)\ \Big|\ S_0=s\Big].}
\]
Let $J^{\mathrm b}(s):=\inf_{\pi}J^{\pi}_{\mathrm b}(s)$ and $J^{*}(s):=\inf_{\pi}J^{\pi}_{*}(s)$.
(Under \ref{ass:proper}–\ref{ass:G}, $J^{\mathrm b}$ is finite and equals the minimal nonnegative solution of the SSP optimality equations.
The shaped value $J^*$ is well defined as the optimal value of the shaped problem; 
the ordering $J^*\le J^{\mathrm b}$ in Theorem~\ref{thm:benefit_G_ssp} does not require the shaped costs to be nonnegative.)

\begin{theorem}[Benefit of nonnegative bonus]\label{thm:benefit_G_ssp}
	Under \ref{ass:proper}–\ref{ass:G}:
	\begin{enumerate}[label=(\roman*),leftmargin=2em]
		\item \label{it:order}
		For all $s\in\mathcal S$, $\displaystyle J^*(s)\ \le\ J^{\mathrm b}(s)$.
		\item \label{it:gap}
		Let $\pi^{\mathrm b}$ be an optimal policy for the base problem, i.e., $J^{\mathrm b}(s)=J^{\pi^{\mathrm b}}_{\mathrm b}(s)$.
		Then for any start state $s_0$,
		\[
		J^{\mathrm b}(s_0) - J^{*}(s_0)
		\;\;\ge\;\;
		\mathbb E_{\pi^{\mathrm b}}\!\Big[\textstyle\sum_{t=0}^{\tau-1} G(S_t,D_t)\ \Big|\ S_0=s_0\Big].
		\]
	\end{enumerate}
\end{theorem}

\begin{proof}
	\emph{Step 1 (policywise identity).}
	For any fixed policy $\pi$ and any start state $s$,
	\[
	J^{\pi}_{\mathrm b}(s) - J^{\pi}_{*}(s)
	=\mathbb E_{\pi}\!\Big[\textstyle\sum_{t=0}^{\tau-1} G(S_t,D_t)\ \Big|\ S_0=s\Big],
	\]
	which follows directly from the definitions by linearity of expectation.
	
	\emph{Step 2 (ordering of optimal values).}
	By Step 1, 
	since $G\ge0$, $J^{\pi}_{*}(s)\le J^{\pi}_{\mathrm b}(s)$ for all $\pi$ and $s$.
	Taking infimum over $\pi$ on both sides yields $J^{*}(s)\le J^{\mathrm b}(s)$, proving \ref{it:order}.
	
	\emph{Step 3 (gap bound at a base-optimal policy).}
	For the base-optimal policy $\pi^{\mathrm b}$,
	\[
	J^{\mathrm b}(s_0) - J^{*}(s_0)
	\;\;\ge\;\;
	J^{\pi^{\mathrm b}}_{\mathrm b}(s_0) - J^{\pi^{\mathrm b}}_{*}(s_0)
	\;=\;
	\mathbb E_{\pi^{\mathrm b}}\!\Big[\textstyle\sum_{t=0}^{\tau-1} G(S_t,D_t)\ \Big|\ S_0=s_0\Big],
	\]
	where the inequality uses $J^{*}(s_0)\le J^{\pi^{\mathrm b}}_{*}(s_0)$ (optimality of $J^*$), and the equality is Step 1.
	This proves \ref{it:gap}.
\end{proof}

\begin{remark}[Bellman-operator view]
	Define the dynamic-programming operators on bounded $V:\mathcal S\to\mathbb R$ by
	\begin{align*}
		(\mathbf T_{*}V)(s)&:=\min_{d\in\mathcal D(s)}\Big\{\mathcal C(s,d)-G(s,d)+\sum_{s'}P(s'|s,d)V(s')\Big\},\\
		(\mathbf T_{\mathrm b}V)(s)&:=\min_{d\in\mathcal D(s)}\Big\{\mathcal C(s,d)+\sum_{s'}P(s'|s,d)V(s')\Big\}.
	\end{align*}
	Then $(\mathbf T_{*}V)(s)\le(\mathbf T_{\mathrm b}V)(s)$ for all $V,s$ because $G\ge0$.
	Under SSP with nonnegative base costs, the minimal nonnegative fixed point of $\mathbf T_{\mathrm b}$ is $J^{\mathrm b}$.
	If, in addition, $G(s,d)\le \mathcal C(s,d)$ for all $(s,d)$, then $\mathbf T_{*}$ also has a minimal nonnegative fixed point equal to $J^{*}$; otherwise $J^{*}$ should be interpreted as the optimal value solving the shaped problem without the nonnegativity qualifier.
	The ordering $J^{*}\le J^{\mathrm b}$ follows from the pointwise operator inequality and a standard telescoping argument along trajectories until absorption.
\end{remark}

For a fixed finite horizon $T<\infty$, replace $\tau$ by $T$ throughout; all conclusions remain valid without additional assumptions.

\subsubsection{Submodularity of Information Gain} \label{sec:submodular}
\paragraph{Linear–Gaussian Sensing Model.}
Let $Y=(Y_x)_{x\in X}\sim\mathcal N(0,K)$ with $K\succ 0$.
For any $A\subseteq X$, define noisy observations
\[
\widetilde Y_A \ :=\ Y_A + \varepsilon_A,\qquad
\varepsilon_A\sim \mathcal N(0,\Sigma_A),
\]
with $\varepsilon_A$ independent of $Y$ and of $\varepsilon_B$ for disjoint $A,B$.
Assume throughout the independent-noise case $\Sigma_A=\mathrm{diag}\big((\sigma_x^2)_{x\in A}\big)\succ 0$.
Define the set function
\[
\mathbf I(A)\ :=\ \mathrm{MI}\big(Y_A;\ \widetilde Y_A\big)
\ =\ \tfrac12\log\det\!\big(I + \Sigma_A^{-1/2}K_{AA}\Sigma_A^{-1/2}\big)
\ =\ \tfrac12\bigl[\log\det(K_{AA}+\Sigma_A)-\log\det(\Sigma_A)\bigr].
\]

\begin{proposition}%[Information gain: normalization, monotonicity, submodularity]
	\label{prop:MI-submod}
	Under the model above, $\mathbf I:2^X\to\mathbb R_{\ge 0}$ satisfies:
	\begin{enumerate}[label=(\roman*),leftmargin=2em]
		\item \emph{Normalization:} $\mathbf I(\varnothing)=0$.
		\item \emph{Monotonicity:} If $A\subseteq B$, then $\mathbf I(A)\le \mathbf I(B)$.
		\item \emph{Submodularity (diminishing returns):} For all $A\subseteq B\subseteq X$ and $i\in X\setminus B$,
		\[
		\mathbf I(A\cup\{i\})-\mathbf I(A)\ \ge\ \mathbf I(B\cup\{i\})-\mathbf I(B).
		\]
	\end{enumerate}
\end{proposition}

\begin{proof}
	\emph{(i) Normalization.} $\mathbf I(\varnothing)=\tfrac12\log\det(I)=0$.
	
	\smallskip
	\noindent
	\emph{(ii) Monotonicity.} By the chain rule and independence of sensor noise across indices,
	\[
	\mathbf I(B)-\mathbf I(A)
	=\mathrm{MI}(Y_{B\setminus A};\,\widetilde Y_{B\setminus A}\mid \widetilde Y_A)\ \ge 0.
	\]
	For a singleton addition $i\notin A$ this becomes
	\[
	\mathbf I(A\cup\{i\})-\mathbf I(A)
	=\tfrac12\log\!\Big(1+\frac{s_{i\mid A}}{\sigma_i^2}\Big)\ \ge 0,
	\]
	where, by Gaussian conditioning (i.e., the Gaussian posterior variance of $Y_i$ given $\widetilde Y_A$ is),
	\[
	s_{i\mid A}:=\mathrm{Var}(Y_i\mid \widetilde Y_A)
	= k_{ii}-k_{iA}(K_{AA}+\Sigma_A)^{-1}k_{Ai}\ \in [0,k_{ii}].
	\]
	
	\smallskip
	\noindent
	\emph{(iii) Submodularity.} Let $A\subseteq B$ and $i\notin B$.
	Using the singleton marginal form above,
	\[
	\mathbf I(A\cup\{i\})-\mathbf I(A)
	=\tfrac12\log\!\Big(1+\frac{s_{i\mid A}}{\sigma_i^2}\Big),\qquad
	\mathbf I(B\cup\{i\})-\mathbf I(B)
	=\tfrac12\log\!\Big(1+\frac{s_{i\mid B}}{\sigma_i^2}\Big).
	\]
	It therefore suffices to show $s_{i\mid A}\ge s_{i\mid B}$.
	Write $C:=B\setminus A$ and apply block Gaussian conditioning (Schur complements):
	\[
	s_{i\mid A}-s_{i\mid B}
	=\ k_{iC\mid A}\,\bigl(K_{CC\mid A}+\Sigma_C\bigr)^{-1}\,k_{Ci\mid A}\ \ \ge\ 0,
	\]
	where $K_{CC\mid A}:=K_{CC}-K_{CA}(K_{AA}+\Sigma_A)^{-1}K_{AC}$ and
	$k_{iC\mid A}:=k_{iC}-k_{iA}(K_{AA}+\Sigma_A)^{-1}K_{AC}$.
	Hence $s_{i\mid B}\le s_{i\mid A}$.
	Since $u\mapsto \tfrac12\log(1+u/\sigma_i^2)$ is increasing and concave on $[0,\infty)$, the marginal gain decreases with the context, proving submodularity.
\end{proof}

\begin{remark}[Noise models]\label{rem:noise-models}
	The proposition assumes independent per-index sensor noise (diagonal $\Sigma_A$), which is the sensing model used in the paper.
	If a \emph{fixed}, subset-independent noise covariance $\Sigma_0\succ0$ applies to all measurements simultaneously via a selection operator (observations $S_A Y+\varepsilon$ with $\varepsilon\sim\mathcal N(0,\Sigma_0)$), one can pre-whiten by $\Sigma_0^{-1/2}$ and derive an equivalent form with identity noise. In that case, a submodularity proof can be carried out analogously.
	For \emph{arbitrary} subset-dependent correlated noise blocks $\Sigma_A$ that are not induced by restricting a single global $\Sigma_0$, submodularity of $\mathbf I$ need not hold in general, and extra care is needed.
\end{remark}

%\begin{remark}[Connection to your notation]
%If one keeps the diagonal-noise case, $\Sigma_A=\mathrm{diag}((\sigma_x^2)_{x\in A})$, the singleton marginal simplifies to
%$\mathbf I(A\cup\{i\})-\mathbf I(A)=\tfrac12\log\!\big(1+s_{i\mid A}/\sigma_i^2\big)$
%with $s_{i\mid A}=k_{ii}-k_{iA}(K_{AA}+\Sigma_A)^{-1}k_{Ai}$.
%This matches the form used in your proof and makes the diminishing-returns step explicit via $s_{i\mid A}\ge s_{i\mid B}$ for $A\subseteq B$.
%\end{remark}

Because $\mathbf I$ is normalized, monotone, and submodular, the standard greedy selection achieves a $(1-1/e)$ approximation for maximizing $\mathbf I$ under a cardinality (or budget) constraint; see, e.g., \citep{NemhauserWolseyFisher1978, FisherNemhauserWolsey1978} and subsequent sensor selection results.

\subsubsection{Convergence under Posterior Sampling}
\label{sec:converge_sampling}
We consider a finite MDP with state space $\mathcal S$, feasible decision set $\mathcal D(s)$, and bounded one–stage costs
$\sup_{s,d}|\mathcal C(s,d)|<\infty$.
Let $P_\rho(\cdot\mid s,d)$ denote the transition kernel parameterized by $\rho$.
For a bounded $V:\mathcal S\to\mathbb R$, define the optimal Bellman operator under model $\rho$ by
\[
(\mathbf T^{\rho}V)(s)\;:=\;\min_{d\in\mathcal D(s)}\Big\{\mathcal C(s,d)+\sum_{s'}P_\rho(s'\!\mid s,d)\,V(s')\Big\}.
\]
Assume one of the following settings holds:

\begin{enumerate}[leftmargin=2em,label=(S\arabic*)]
	\item \label{ass:disc}
	\emph{Discounted case:} A discount $\gamma\in(0,1)$ is incorporated in $\mathcal C$ (or in the Bellman operator), so $\mathbf T^{\rho}$ is a $\gamma$–contraction in $\|\cdot\|_\infty$ with a unique fixed point $V^{\!*}(\rho)$.
	\item \label{ass:ssp}
	\emph{Stochastic shortest path (SSP):} There exists an absorbing goal state $g$ with zero cost; all policies considered are \emph{proper} (hit $g$ a.s. from any start), and the standard SSP conditions hold so that $\mathbf T^{\rho}$ has a minimal nonnegative fixed point $V^{\!*}(\rho)$ and value iteration converges to it.
\end{enumerate}

\noindent
Let $\{\mathcal F_i\}$ be the natural filtration generated by all randomness up to iteration $i$.
At iteration $i$ we:
\begin{itemize}
	\item form the posterior and its mean $\bar\rho_i:=\mathbb E[\rho_i\mid \mathcal F_{i-1}]$;
	\item compute a greedy policy $\pi_i(s)\in\arg\min_{d}\{\mathcal C(s,d)+\sum_{s'}P_{\bar\rho_i}(s'\!\mid s,d)\,V_{i-1}(s')\}$;
	\item sample a dynamics parameter $\rho_i$ from the current posterior (posterior sampling);
	\item form a \emph{sampled policy-evaluation target}
	$\widehat{\mathbf T}_i V_{i-1}(s)$, which is an unbiased estimator of $(\mathbf T^{\rho_i}_{\pi_i}V_{i-1})(s)$ (e.g., draw $S'\!\sim P_{\rho_i}(\cdot\mid s,\pi_i(s))$ and use $\mathcal C(s,\pi_i(s)) + V_{i-1}(S')$).
\end{itemize}
We update per state with stepsizes $\alpha_i(s)\in(0,1]$:
\[
V_i(s)\ :=\ (1-\alpha_i(s))\,V_{i-1}(s)\ +\ \alpha_i(s)\,\widehat{\mathbf T}_i V_{i-1}(s).
\]
Assume (Robbins–Monro) for each $s$: $\sum_i \alpha_i(s)=\infty$, $\sum_i \alpha_i^2(s)<\infty$, and every state is updated infinitely often.

Let $\bar\rho_i:=\mathbb E[\rho_i\mid \mathcal F_{i-1}]$ denote the posterior mean at iteration $i$ and $\bar\rho$ a limit parameter (either the fixed mean if the dataset is frozen, or $\bar\rho_i\to\bar\rho$ a.s. if the posterior concentrates).

\begin{theorem}[Convergence under posterior sampling]\label{thm:converge_sampling_rigorous}
	Under \ref{ass:disc} or \ref{ass:ssp}, and the conditions above, suppose in addition that:
	\begin{enumerate}[leftmargin=2em,label=(A\arabic*)]
		\item \label{ass:md}
		(\emph{Martingale-difference noise}) The target noise
		\[
		\xi_i(s)\ :=\ \widehat{\mathbf T}_i V_{i-1}(s)\ -\ (\mathbf T^{\rho_i}V_{i-1})(s)
		\]
		satisfies $\mathbb E[\xi_i(s)\mid \mathcal F_{i-1}]=0$ and $\mathbb E[\xi_i(s)^2\mid \mathcal F_{i-1}] \le C(1+\|V_{i-1}\|_\infty^2)$ a.s. for some $C<\infty$.
		\item \label{ass:ps-fixed}
		(\emph{Posterior sampling unbiasedness for fixed policy}) The model mismatch term
		\[
		\eta_i(s)\ :=\ (\mathbf T^{\rho_i}V_{i-1})(s)\ -\ (\mathbf T^{\bar\rho_i}V_{i-1})(s)
		\]
		satisfies $\mathbb E[\eta_i(s)\mid \mathcal F_{i-1}] = 0$ for all $s$.
		\item \label{ass:meanlimit}
		(\emph{Stable target}) Either $\bar\rho_i\equiv\bar\rho$ for all $i$ (frozen posterior mean) or $\bar\rho_i\to\bar\rho$ a.s.,
		and $\mathbf T^{\bar\rho_i}\to \mathbf T^{\bar\rho}$ uniformly on bounded sets.
	\end{enumerate}
	Then $V_i\to V^{\!*}(\bar\rho)$ almost surely as $i\to\infty$.
\end{theorem}

\begin{proof}
	Write the update as
	\[
	V_i(s)\ =\ (1-\alpha_i(s))\,V_{i-1}(s)\ +\ \alpha_i(s)\Big\{\ (\mathbf T^{\bar\rho}V_{i-1})(s)\ +\ \epsilon_i(s)\ +\ \zeta_i(s)\ \Big\},
	\]
	where we have decomposed the perturbation into
	\[
	\epsilon_i(s):=\underbrace{(\mathbf T^{\bar\rho_i}V_{i-1})(s)-(\mathbf T^{\bar\rho}V_{i-1})(s)}_{\to 0}
	\ +\ \underbrace{\big[(\mathbf T^{\bar\rho_i}_{\pi_i}V_{i-1})(s)-(\mathbf T^{\bar\rho_i}V_{i-1})(s)\big]}_{\le 0},
	\qquad
	\zeta_i(s):=\eta_i(s)+\xi_i(s).
	\]
	By \ref{ass:md}–\ref{ass:ps-fixed}, $\{\zeta_i(s),\mathcal F_i\}$ is a square-integrable martingale-difference sequence.
	By \ref{ass:meanlimit}, $\|\epsilon_i\|_\infty\to 0$ a.s. (the bracketed term is nonpositive and vanishes as $V_{i-1}\to V^{\!*}(\bar\rho)$).
	Under \ref{ass:disc}, $\mathbf T^{\bar\rho}$ is a contraction; under \ref{ass:ssp}, the standard monotone SSP convergence applies.
	Stochastic approximation for asynchronous value iteration with martingale noise then yields $V_i\to V^{\!*}(\bar\rho)$ a.s.
\end{proof}

\ref{ass:md} holds for unbiased Monte Carlo targets with bounded second moments.
\ref{ass:ps-fixed} holds because, conditional on $\mathcal F_{i-1}$, $\rho_i$ is sampled from a posterior with mean $\bar\rho_i$ and $\mathbf T^{\rho}_{\pi}$ is \emph{linear in $P_\rho$} for fixed $\pi$.
If the posterior is frozen, \ref{ass:meanlimit} is trivial; if it concentrates, it holds when $P_{\bar\rho_i}\to P_{\bar\rho}$.
If one updates only a subset of states per iteration, require that every state is visited infinitely often and apply the same argument componentwise (asynchronous SA).
If an information–gain shaped cost $\mathcal C^{\scriptscriptstyle IG}$ is used, replace $\mathcal C$ by $\mathcal C^{\scriptscriptstyle IG}$ throughout.

\subsubsection{Stability of Support Refinement}
The distribution of value function at state $s$, $V(s)$, is categorical with support $\mathcal{X}(V(s))=\{x_i\}_{i=1}^k$, 
where the probabilities have a Dirichlet distribution with parameters $\{\alpha_i\}_{i=1}^k$. 
The mean is
$
\mu = \E\left[V(s)\right] = \frac{\sum_{i=1}^K \alpha_i x_i}{\sum_{i=1}^k \alpha_i}.
$

Given a new observed cumulative cost $\mathcal{C}^{\scriptscriptstyle IG}(s)$ which is not in support and $\delta>0$ defining a neighborhood,
we implement support refinement strategy through value contract.
Specifically,
we locate consecutive support points $x_j, x_{j+1}$ within $\delta$ of $\mathcal{C}^{\scriptscriptstyle IG}_i(s)$. 
If these points have not been observed in previous trajectories, 
we replace them with $\mathcal{C}^{\scriptscriptstyle IG}_i(s)$ and merge their Dirichlet parameters. 
If no neighbors exist within $\delta$, we insert $\mathcal{C}^{\scriptscriptstyle IG}_i(s)$ as a new support point with parameter $\alpha=1$.

\begin{lemma}\label{lem:refinement_stability}
	Given a new observed cumulative cost $\mathcal{C}^{\scriptscriptstyle IG}(s)$ and $\delta>0$ defining a neighborhood,
	the support refinement through value contraction leads to a bounded change in the mean that shrinks as the total count grows, 
	introducing no consistent upward or downward bias.
\end{lemma}

\begin{proof}
	We consider two cases based on the neighborhood structure.
	
	(i) \emph{Two neighbors within $\delta$ distance.}
	After replacing $x_j,x_{j+1}$ with $\mathcal{C}^{\scriptscriptstyle IG}(s)$,
	the corresponding Dirichlet parameter is updated as $\alpha_j+\alpha_{j+1}+1$. 
	Define $\alpha_{\mathrm{total}} = \sum_{i=1}^{k} \alpha_i$ and 
	$\bar{x}_{j,j+1} = \frac{\alpha_j x_j + \alpha_{j+1} x_{j+1}}{\alpha_j + \alpha_{j+1}}$.
	The new mean becomes
	\[
	\mu' = \frac{\alpha_{\mathrm{total}} \mu - (\alpha_j x_j + \alpha_{j+1} x_{j+1}) 
		+ (\alpha_j + \alpha_{j+1} + 1)\mathcal{C}^{\scriptscriptstyle IG}(s)}
	{\alpha_{\mathrm{total}} + 1}.
	\]
	The difference between $\mu$ and $\mu'$ is
	\begin{align*}
		\mu' - \mu
		&= \frac{\mathcal{C}^{\scriptscriptstyle IG}(s) - \mu}{\alpha_{\mathrm{total}} + 1}
		+ \frac{(\alpha_j + \alpha_{j+1})(\mathcal{C}^{\scriptscriptstyle IG}(s) - \bar{x}_{j,j+1})}{\alpha_{\mathrm{total}} + 1}.
	\end{align*}
	The first term shrinks as the total count grows.
	For the second term, since $x_j$ and $x_{j+1}$ are within $\delta$ of $\mathcal{C}^{\scriptscriptstyle IG}(s)$,
	we have $|\mathcal{C}^{\scriptscriptstyle IG}(s) - \bar{x}_{j,j+1}| \leq \delta$.
	Therefore,
	\[
	|\mu' - \mu| \leq \frac{|\mathcal{C}^{\scriptscriptstyle IG}(s) - \mu| + (\alpha_j + \alpha_{j+1})\delta}{\alpha_{\mathrm{total}} + 1},
	\]
	leading to a bounded change that decreases with increasing $\alpha_{\mathrm{total}}$.
	
	(ii) \emph{No neighbors within $\delta$ distance.}
	A new support point is added with $\alpha = 1$, giving
	$
	\mu' = \frac{\alpha_{\mathrm{total}} \mu + \mathcal{C}^{\scriptscriptstyle IG}(s)}{\alpha_{\mathrm{total}} + 1},
	$
	and thus
	\[
	\mu' - \mu = \frac{\mathcal{C}^{\scriptscriptstyle IG}(s) - \mu}{\alpha_{\mathrm{total}} + 1},
	\]
	which also shrinks as $\alpha_{\mathrm{total}}$ increases, introducing no consistent upward or downward bias.
\end{proof}

\section{Online Rollout with Base Policy Update}\label{sec:online}
Building on the high-quality base policy learned offline, 
our online phase implements rollout policy for making real-time decisions.
Rollout policies are theoretically guaranteed to perform no worse than their underlying base policy \citep{bertsekas2021rollout}, 
making the quality of the base policy crucial for overall performance.
Formally, the rollout policy makes decisions following the rule:
$$d_t^{\text{roll}}=
\argmin_{d_t \in \mathcal{D}_t(S_t)}\left(\mathcal{C}(S_t, d_t)+
\E_{W_{t+1} | S_t, d_t}\left[V^{\text{base}}(S_{t+1})\right]\right),$$
where $V^{\text{base}}(S_{t+1})=\min\E\left[\sum_{t'=t+1}^{T}\mathcal{C}(S_{t'},d_{t'}^{\text{base}})|S_{t+1}\right]$ 
is the estimated optimal state value function following the base policy.

At $t=0$, the decision can be made directly using the pre-computed $V^{\text{base}}$ from offline stage.
From $t=1$ onward, simulation is required to incorporate the updated belief as new information is collected.
Given the performance guarantee of rollout policies, 
maintaining base policy quality becomes critical as belief information evolves.

\paragraph{Base Policy Update Mechanism}
A unique challenge in the SCOS setting is that the belief over obstacle blockage statuses is not static.
Instead, it evolves as new information is collected through sensor readings and disambiguations.
If left unadjusted, $\pi^{\text{base}}$ may become misaligned with the improved belief,
decreasing the effectiveness of the rollout policy.
To address this challenge, we update $\pi^{\text{base}}$ after each rollout step,
ensuring that our two-stage framework remains robust and adaptive throughout the traversal process.

The updating process varies depending on the value function approximation method used in the offline phase.
Under point estimation approach, we update $V^{\text{base}}$ using the incremental learning rule from Equation (\ref{equation: MC update}).
Under the distributional RL approach, we update the value distribution using the support refinement process outlined in Section \ref{sec:DRL}.

\section{Monte Carlo Experiments and Comparison} \label{sec: simulations}
We evaluate policies using 8-adjacency grid graphs that
align with the discretized setting of the real-world COBRA minefield dataset,
which is commonly used in path planning studies (\cite{fishkind2005}, \cite{ye2011sensor}, \cite{aksakalli2016based}, \cite{aslan2019any}).

Formally, the traversal region is a two-dimensional plane, $\Omega=[0,I]\times[0,J]$, 
represented by an undirected 8-adjacency integer lattice graph $\mathcal{G}$. 
$\mathcal{V(G)}$ is the set of all vertices corresponding to grid intersections, 
and $\mathcal{E(G)}$ is the set of edges connecting each vertex to its horizontal, vertical and diagonal neighbors, 
corresponding to the sides and diagonals of the grid squares.
Each vertex $v\in \mathcal{V(G)}$ has a pair of integer coordinates $(i,j)$, where $i=1, 2, ..., I$ and $j=1, 2, ..., J$.
Each edge $e\in \mathcal{E(G)}$ is defined between vertices of the following forms: $(i,j)$ and $(i+1,j)$, $(i,j)$ and $(i, j+1)$, $(i,j)$ and $(i+1, j+1)$.
We set the start vertex at $s=(\lfloor I/2 \rfloor,J)$ and the goal vertex as $g=(\lfloor I/2 \rfloor,1)$,
encouraging strategic navigation around obstacles rather than detouring along an unnecessarily long obstacle-free path.
Grid sizes range from compact (i.e, 50×25) to larger area operations (i.e, 100×50),
reflecting different problem scales.

Disk-shaped obstacles of fixed radius are randomly positioned within a subregion of $\Omega$ to ensure
that at least one long traversable path exits between $s$ and $g$,
with obstacle density in the traversal region span from lightly defended (i.e, 20-30 obstacles) to heavily filled (i.e, 40-60 obstacles).
To reflect realistic minefield characteristics,
we impose a spatial correlation pattern between obstacles based on two practical defensive strategies:
(i) obstacles closer to the goal has higher log-odds of being blocked (i.e., true threats),
simulating defensive strategy for protecting objectives, and
(ii) isolated obstacles (with fewer neighboring obstacles) have higher log-odds compared to clustered ones,
reflecting surveillance patterns where decoy clusters hide real threat locations.
The resulting setup ensures obstacles with similar tactical importance exhibit correlated log-odds values,
with additional noise generated according to multivariate normal distribution for capturing environmental uncertainties.  

The sensor readings are generated from $Beta(\alpha, \beta)$, 
where parameters vary depending on sensor noise and the actual obstacle status:
\begin{align*}
	\alpha_O=4+\lambda, \; \beta_O=4-\lambda, \; \text{if } z_i = 1 \\
	\alpha_F=4-\lambda, \; \beta_F=4+\lambda, \; \text{if } z_i = 0
\end{align*}
$\lambda\in(0,4)$ controls the level of sensor noise,
with lower values (e.g., $\lambda=0.35, 0.75$) simulating basic detectors with high false alarm rates,
and higher values (e.g., $\lambda=1.5, 2.5$) representing advanced radar with improved discrimination capabilities.
Sensor ranges vary from restricted capabilities (i.e., 10-15 grid units) to extended detection ranges (i.e., 20-30 grid units),
reflecting different operational constraints.

\begin{table}[H]
	\centering
	\resizebox{0.9\textwidth}{!}{
		\begin{tabular}{@{}c|c|c|c|c@{}}
			\toprule
			Graph Size                     & Number of Obstacles ($N$) & \begin{tabular}[c]{@{}c@{}}Disambiguation Cost\\ (= Radius)\end{tabular} & \begin{tabular}[c]{@{}c@{}}Sensor Range ($R$)\\ (to obstacle center)\end{tabular} & Sensor Parameter ($\lambda$)          \\ \midrule
			\multirow{2}{*}{$50\times25$}  & $20$                & $3.5$                                                                      & \multirow{2}{*}{$10, 12.5, 15$}                                            & \multirow{2}{*}{$0.35, 0.75, 1.5, 2.5$} \\
			& $40$                & $3$                                                                    &                                                                             &                                   \\ \midrule
			\multirow{2}{*}{$100\times50$} & $30$                & $5.5$                                                                      & \multirow{2}{*}{$20, 25, 30$}                                               & \multirow{2}{*}{$0.35, 0.75, 1.5, 2.5$} \\
			& $60$                & $5$                                                                    & \multicolumn{1}{l|}{}                                                       & \multicolumn{1}{l}{}              \\ \bottomrule
		\end{tabular}
	}
	\caption{Parameters of Simulation Setting}
	\label{tab: envir_parameter}
\end{table}

The complete simulation parameter combinations are summarized in Table \ref{tab: envir_parameter}.
For each combination, we generate 50 distinct environments, 
each includes 10 random replicates according to the probability information,
yielding 500 total simulation runs per parameter setting.
Example traversal regions of different obstacle density levels are presented in Figure \ref{fig:Example Environment}.

\begin{figure}[H]
	\centering
	\begin{tabular}{cc}
		\begin{subfigure}[b]{0.45\textwidth}
			\centering
			\includegraphics[width=\textwidth]{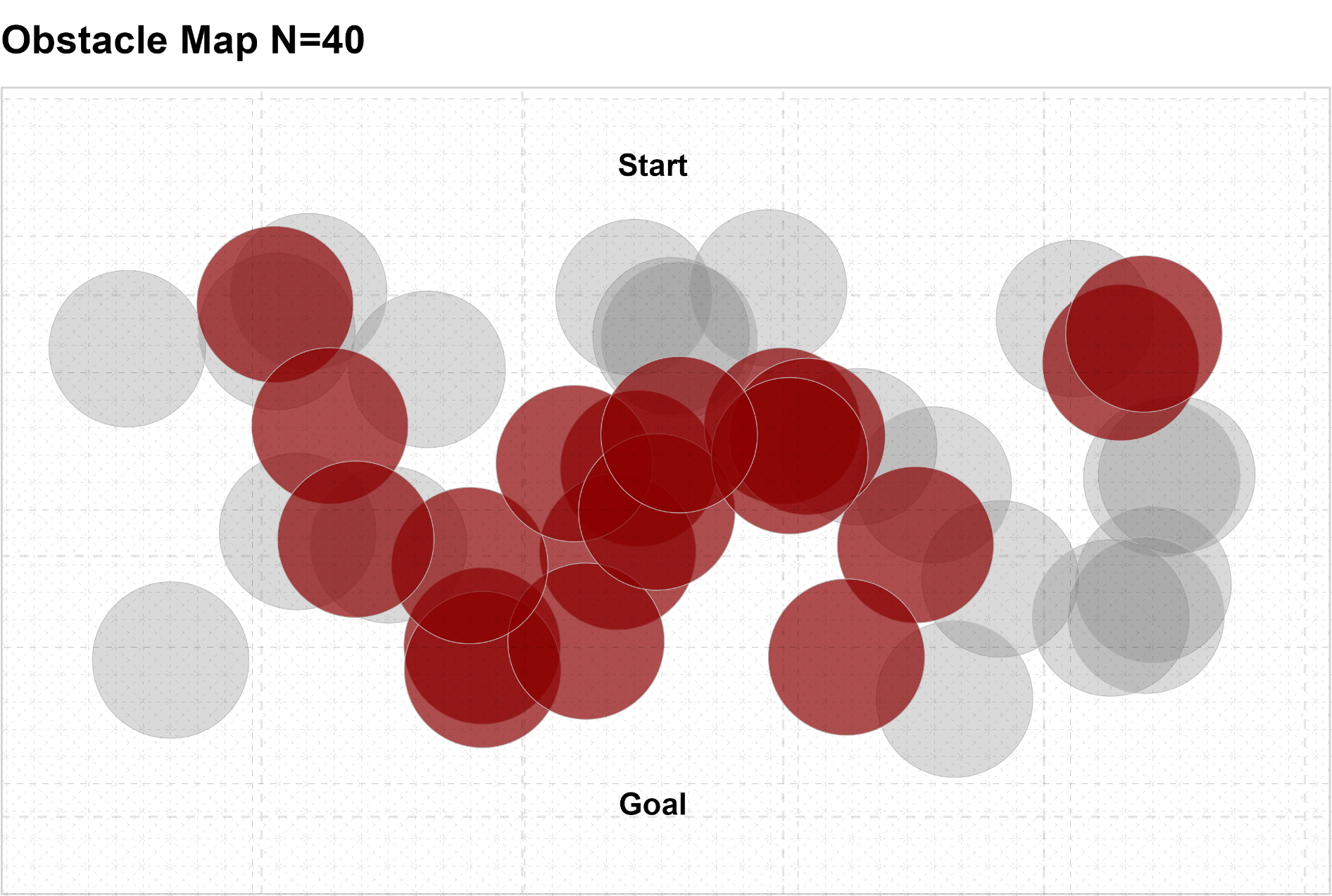}
		\end{subfigure} &
		\begin{subfigure}[b]{0.45\textwidth}
			\centering
			\includegraphics[width=\textwidth]{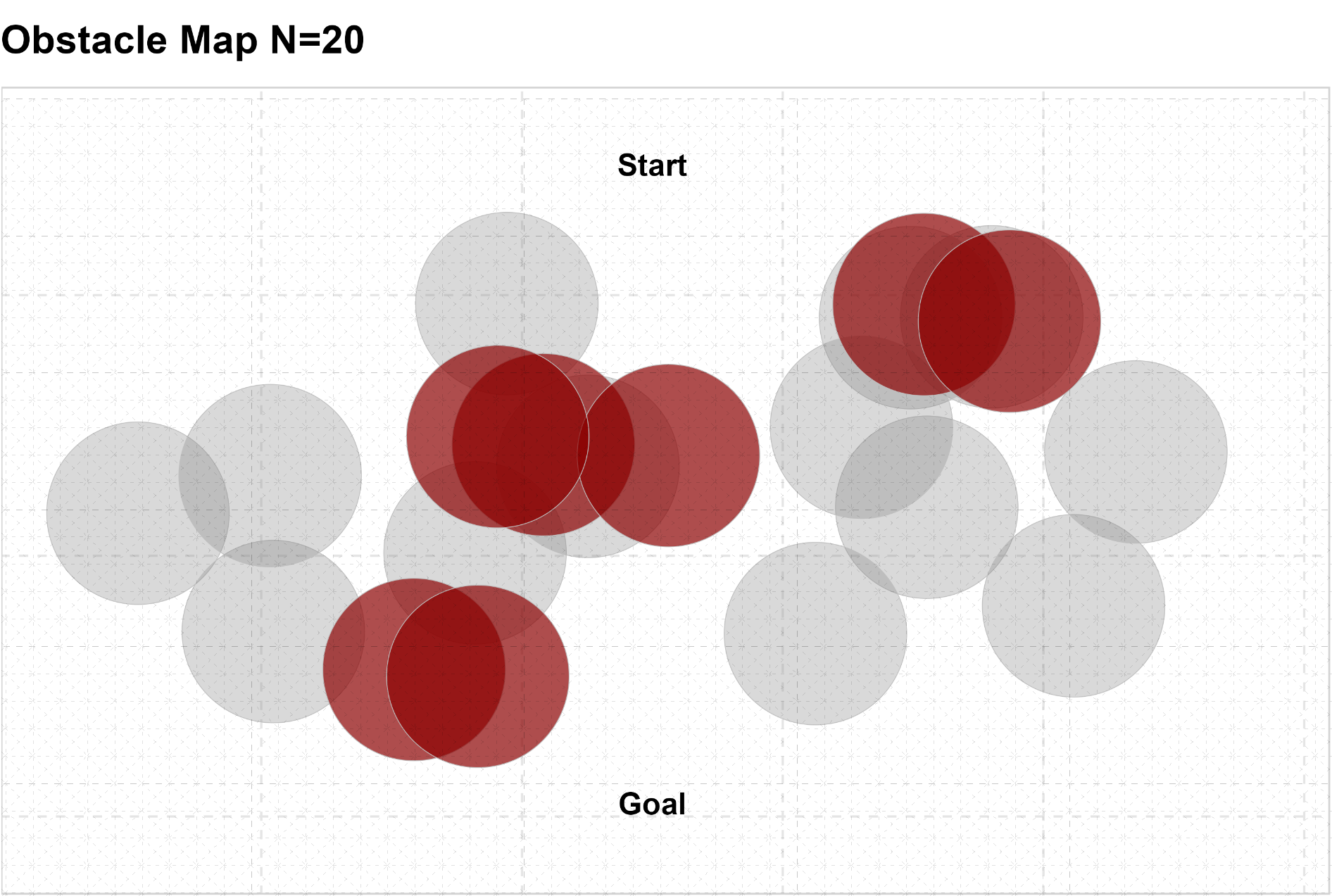}
		\end{subfigure} 
	\end{tabular}
	\caption{Example traversal environments containing $N=40$ and $N=20$ obstacles (red, gray disks present obstacles that are actual threats and false alarms, respectively)}
	\label{fig:Example Environment}
\end{figure}

\subsection{Evaluation Metrics} \label{sec: comparison_evaluation}
To comprehensively compare policy performance, we use the following evaluation metrics: 

\paragraph{Average performance.} 
We compute the mean and median traversal costs across all simulation environments,
providing direct comparison of policy efficiency and typical performance.

\paragraph{Optimality gap.}
We assess each policy's deviation from optimality by calculating the difference between achieved costs and an offline-optimal benchmark: $\mathcal{C}-\mathcal{C}_{\text{optimal}}$,
quantifying how closely each policy approaches the optimal solution.
This offline benchmark assumes perfect knowledge of obstacle actual status, which is an unrealistic assumption in practice.

\paragraph{Consistency and robustness.}
We evaluate (i) within-environment consistency by calculating the standard deviation of traversal costs across replicates within each environment, 
and (ii) cross-environment robustness by calculating the standard deviation of mean traversal costs across environments,
assessing policy reliability under stochastic conditions.

\paragraph{Convergence speed.}
We record the average offline and online simulation time per complete traversal,
indicating the policy appropriateness for real-time decision making. 

\subsection{Competing Policies} \label{sec: comparison_policy}
We consider four competing baseline approaches which represent diverse strategies for solving navigation problems,
ranging from computationally efficient heuristics to approaches effectively incorporating probabilistic information,
leading to a balanced evaluation of our proposed framework.

\paragraph{Penalty-based policies.}
These approaches assign a deterministic value to each path by penalizing high-risk ones in addition to the actual traversal length,
then apply classic shortest path algorithms (e.g., Dijkstra's algorithm) with replanning when encountering an ambiguous obstacle.
We consider two variants \citep{sahin2015comparison, alkaya2021heuristics}:
(i) RD policy penalizes cost of path using $\tilde{\mathcal{C}}_p=\ell_p+\sum_{x: x\cap p \neq\emptyset}\frac{c_x}{1-\rho_x}$,
and (ii) DT policy incorporates distance-to-goal term $d(x,g)$, shown to have comparable performance to UCT-based methods: $\tilde{\mathcal{C}}_p=\ell_p+\sum_{x: x\cap p \neq\emptyset}\left[c_x+\left(\frac{d(x,g)}{1-\rho_x}\right)^{-\log(1-\rho_x)}\right]$.

\paragraph{Rollout-based policies.}
These approaches use simulation to evaluate future trajectories,
but differ in the base policy strategies and assumptions compared to our two-stage learning framework.
While we learn high-quality base policy offline and continuously adapt them to evolving beliefs, 
these baselines rely on simple heuristic assumptions that remain static throughout traversal.
We consider two approaches \citep{eyerich2010high, hou2022dynamic}:
(i) hindsight policy using rollout with a base policy that assumes perfect information about sampled environment during simulation,
serving as a powerful benchmark for comparison in the literature, and 
(ii) optimistic rollout policy assuming all ambiguous obstacles in sampled environment are traversable (i.e., optimistic assumption) during simulation,
encouraging exploration in uncertain regions.

\subsection{Illustrative Example} \label{sec: toy-example}
We use one example traversal region including 40 potentially blocked disks to
illustrate how correlation modeling and the information gain bonus impact decisions and traversal costs.

Figure \ref{fig:example_all}(a) presents the traversal from our proposed two-stage policy learning framework.
The path takes a short, direct route after the information collection through sensing and disambiguation in the region near the starting vertex,
reaching the goal with total cost 46.21.
\ref{fig:example_all}(b) shows the result when the information gain bonus is removed from our two-stage learning framework.
This leads to a more conservative decision since the policy does not consider
the potential value of new information that might benefit future decisions,
instead directly choosing an obstacle-free path of cost 57.63.
Figure \ref{fig:example_all}(c)-(d) include the traversal results using the baseline policy, DT policy,
under two modeling assumptions, correlation-aware and independent.
Compared with our two-stage strategy,
the penalty policy is myopic even with correlation.
It enters the central clustered region without adequate strategic planning, 
and must detour around the obstacles when additional information reveals the risk,
resulting in a route with cost 90.36.
The performance is worse when correlation is ignored.
The policy cannot refine beliefs on statuses of nearby obstacles,
repeatedly disambiguates locally with additional costs,
increasing the cost to 120.43.

\begin{figure}[t]
	\centering
	\begin{tabular}{cc}
		\includegraphics[width=0.4\textwidth]{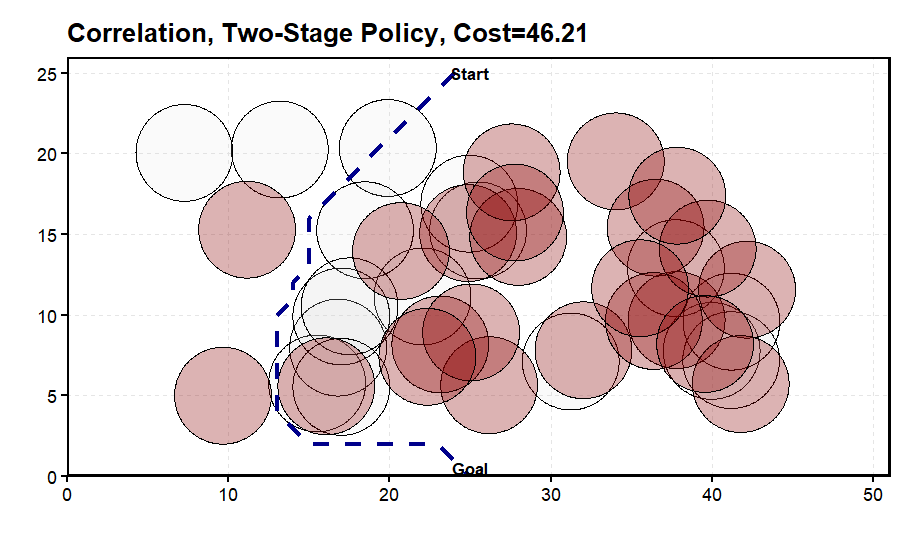} &
		\includegraphics[width=0.4\textwidth]{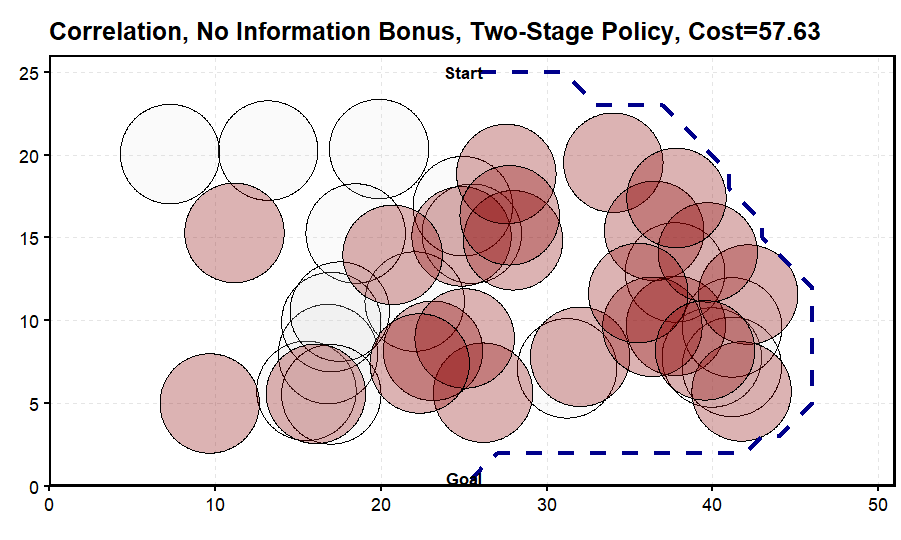} \\
		\small (a) Two-Stage Policy Learning &
		\small (b) Two-Stage Policy (no information gain bonus) \\[0.8cm]
		
		\includegraphics[width=0.4\textwidth]{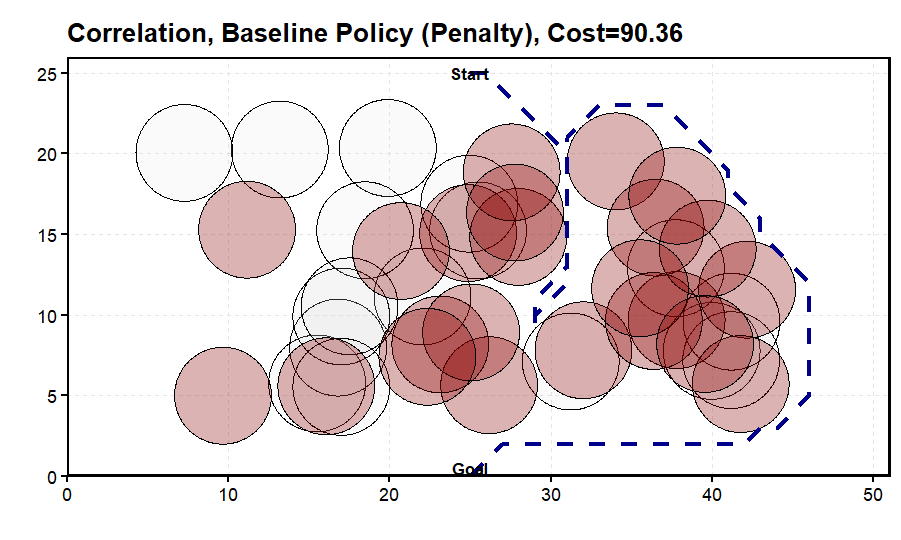} &
		\includegraphics[width=0.4\textwidth]{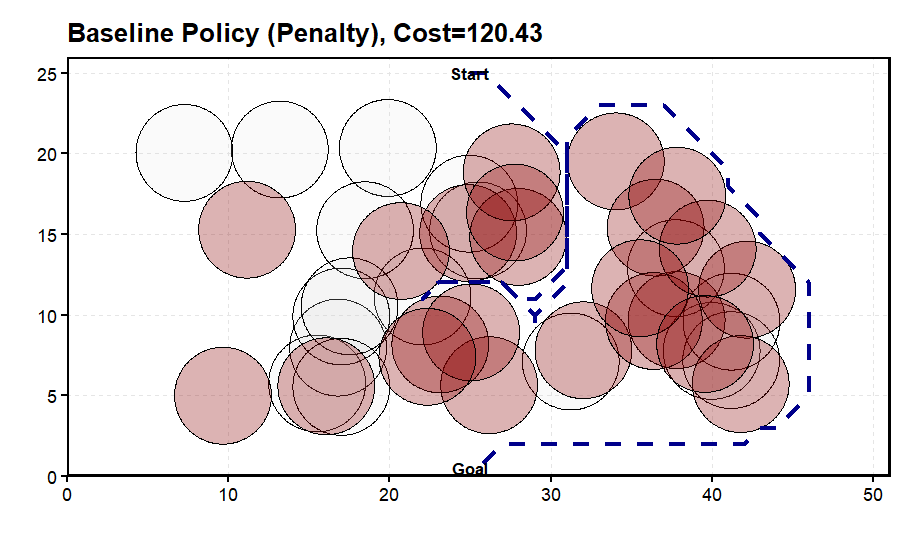} \\
		\small (c) Penalty Policy (correlation) &
		\small (d) Penalty Policy (independent) \\
	\end{tabular}
	\caption{Illustration of different policy approaches for a region with 40 potentially blocked disks under various correlation assumptions (disk background color shows the ground truth, red and gray disks present the actual threats and false alarms, respectively).}
	\label{fig:example_all}
\end{figure}

\subsection{Empirical Results}\label{sec: simulation results}
We present key numerical results and insights in this section,
with additional results for different combinations of environmental parameters included in the Appendix.

Figure \ref{fig: avg cost_lambda} displays the mean traversal costs 
with 95\% confidence intervals across environments. 
In general, traversal costs decrease monotonically as $\lambda$ increases (i.e., sensor precision increases), 
graph size reduces and obstacle density decreases,
confirming the expected relationship between environmental complexity and path planning difficulty. 
Across policies,
our proposed two-stage policy framework consistently yields lower mean traversal costs and 
tighter confidence interval than rollout and penalty baselines,
especially in more challenging environments with high noise and dense obstacles,
showing the sensitivity of baselines to environmental factors.
Exceptions are observed when sensors are highly accurate,
where the baselines appear to be comparably effective. 
These exceptions likely occur because high sensor accuracy condition creates a relatively straightforward setting 
that does not require sophisticated uncertainty handling techniques.
In such simple scenarios, iterative learning approaches tend to be conservative,
but performance can be improved via tuning exploration parameter (e.g., weight of information gain). 

Within the two-stage framework,
the distributional RL (DRL) base exhibits the lowest traversal costs in most cases,
with particularly strong performance in challenging environments due to its superior uncertainty quantification.
Due to better exploration mechanisms,
the Monte Carlo bases using softmax and $\epsilon$-greedy exploration strategies, due to better exploration mechanisms
demonstrate competitive performance and establish advantage over the greedy base.
Among baseline policies,
the optimistic rollout policy appears as the strongest competitor,
while the hindsight policy performs surprisingly poorly.
Comparing two penalty policies,
the RD policy consistently exhibits high traversal costs, particularly under conditions of high obstacle density and high noise, 
while DT policy shows better performance across most settings.

Median cost comparisons (see Figure \ref{fig: median_cost_lambda}) show similar performance ranking pattern,
with more pronounced advantages of using the proposed policy learning framework,
further validates its benefits.

\begin{figure}[H]
	\centering
	\begin{tabular}{cc}
		\includegraphics[width=0.44\textwidth]{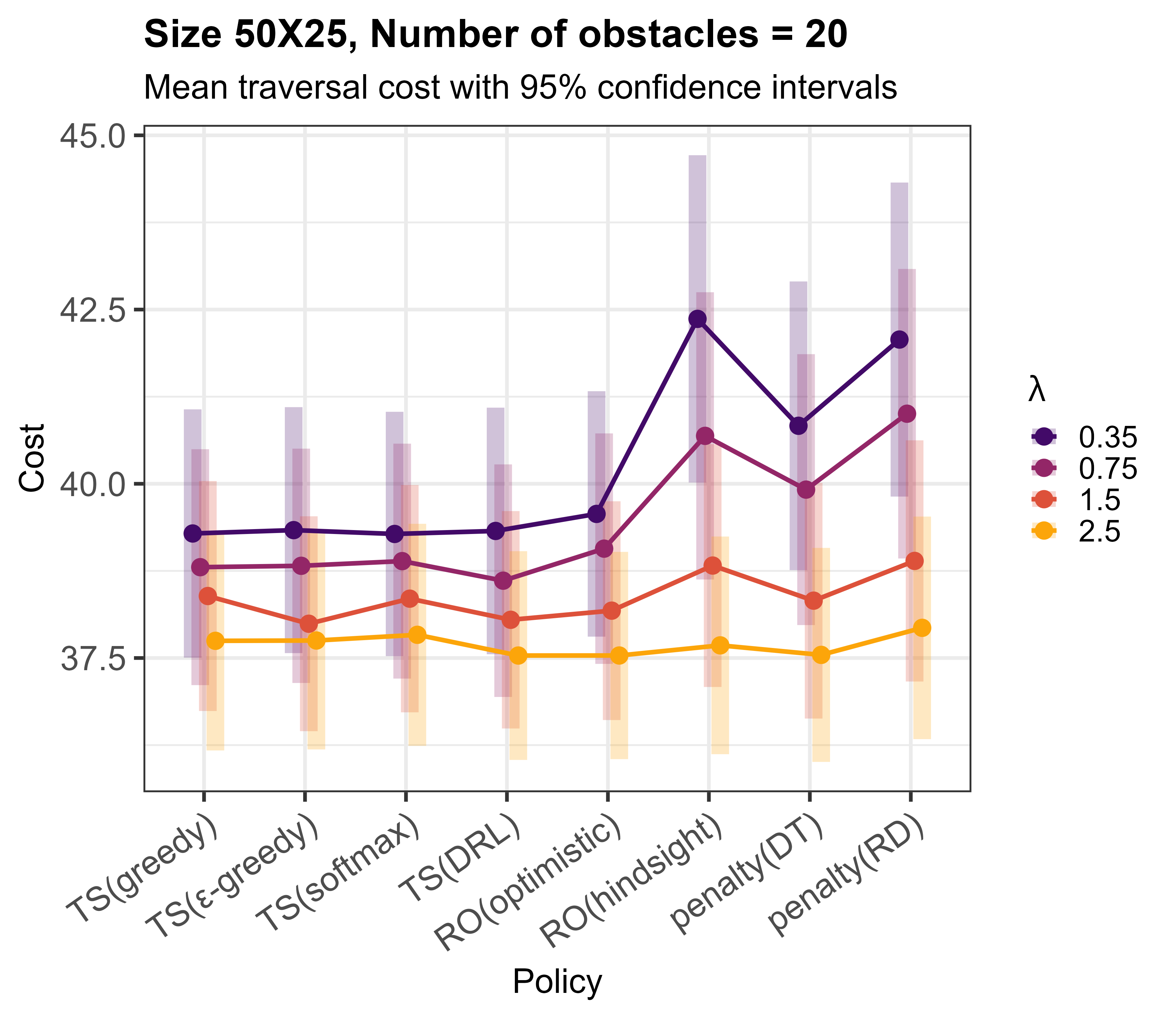} &
		\includegraphics[width=0.44\textwidth]{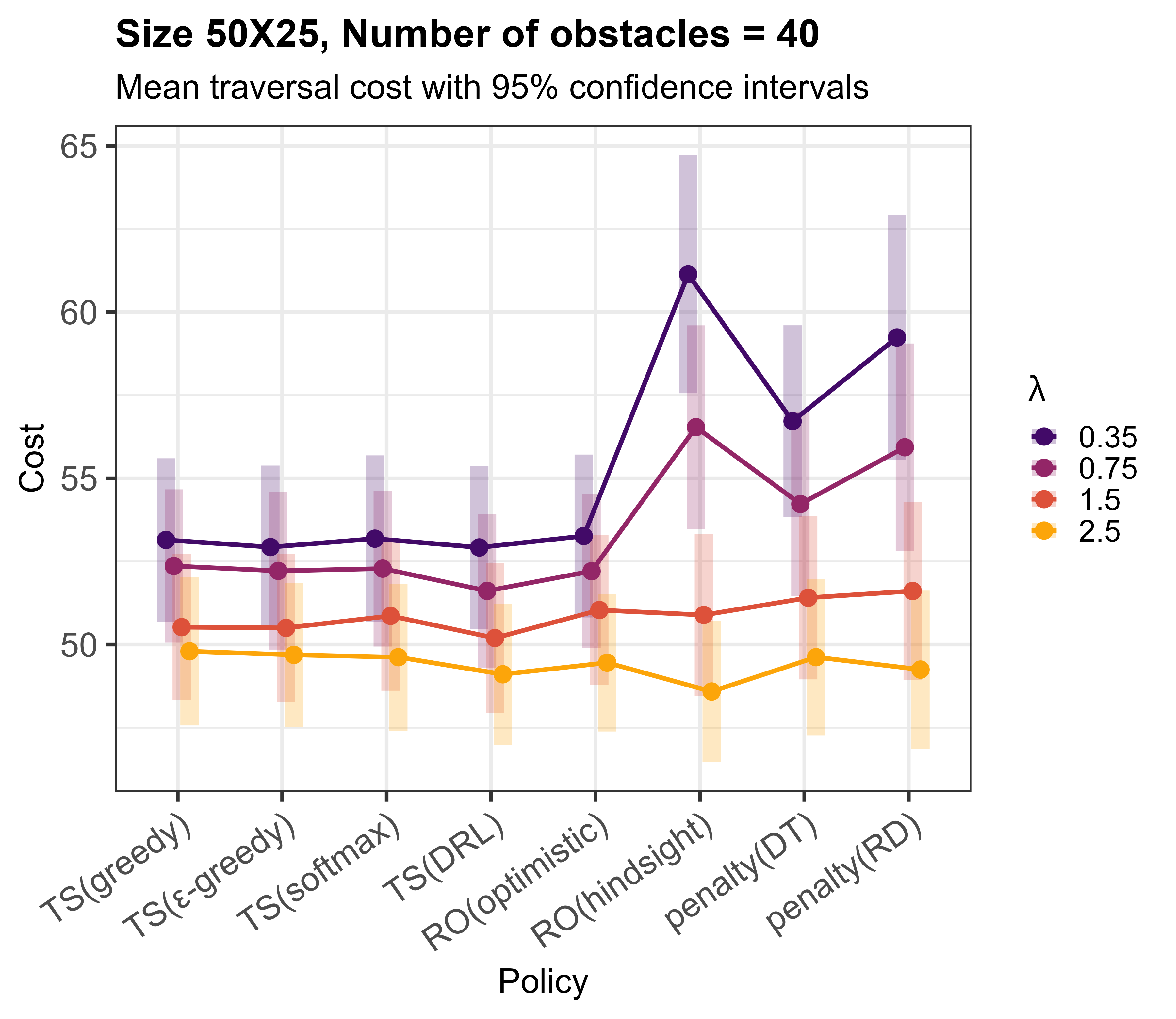} \\
		\includegraphics[width=0.44\textwidth]{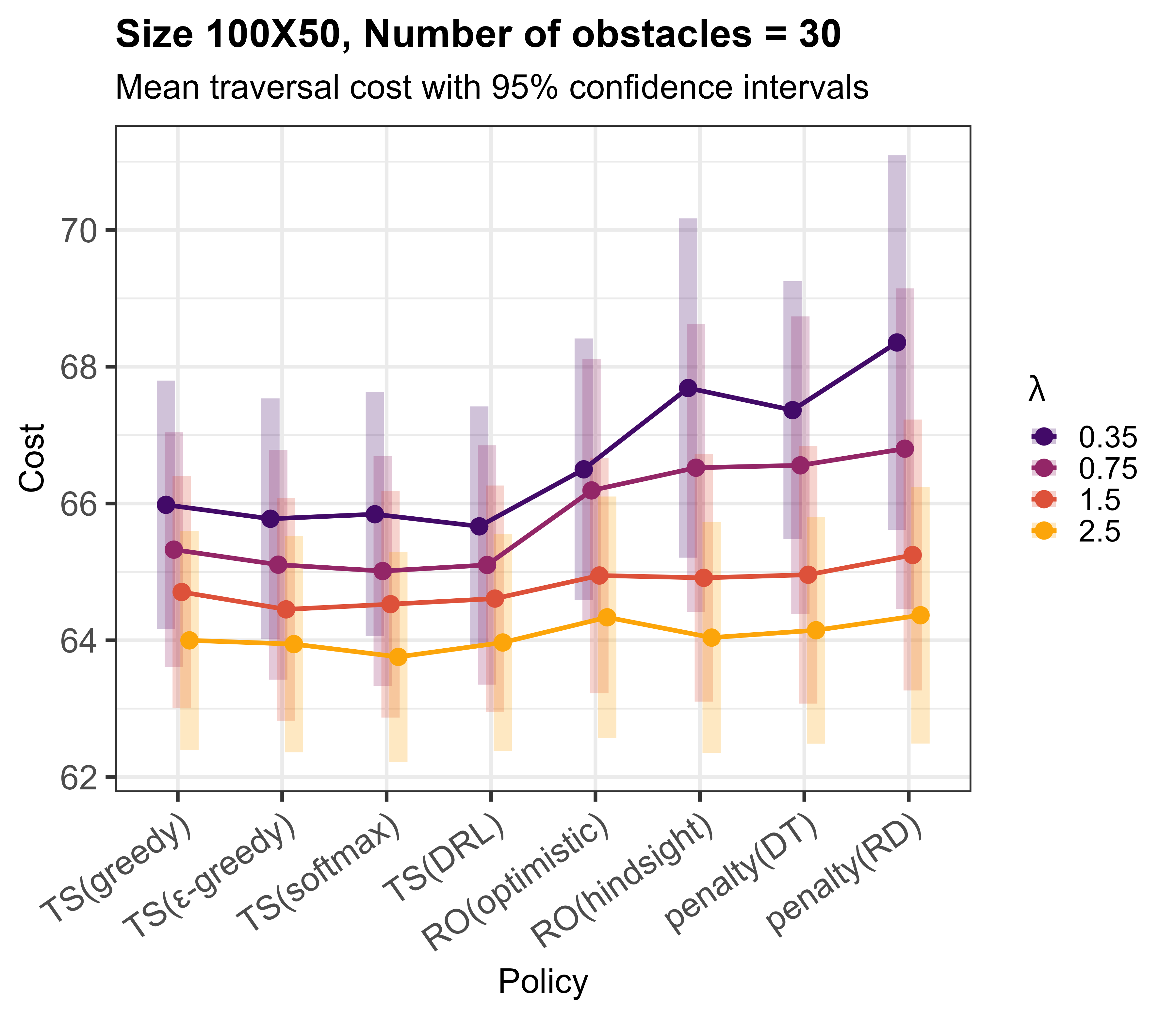} &
		\includegraphics[width=0.44\textwidth]{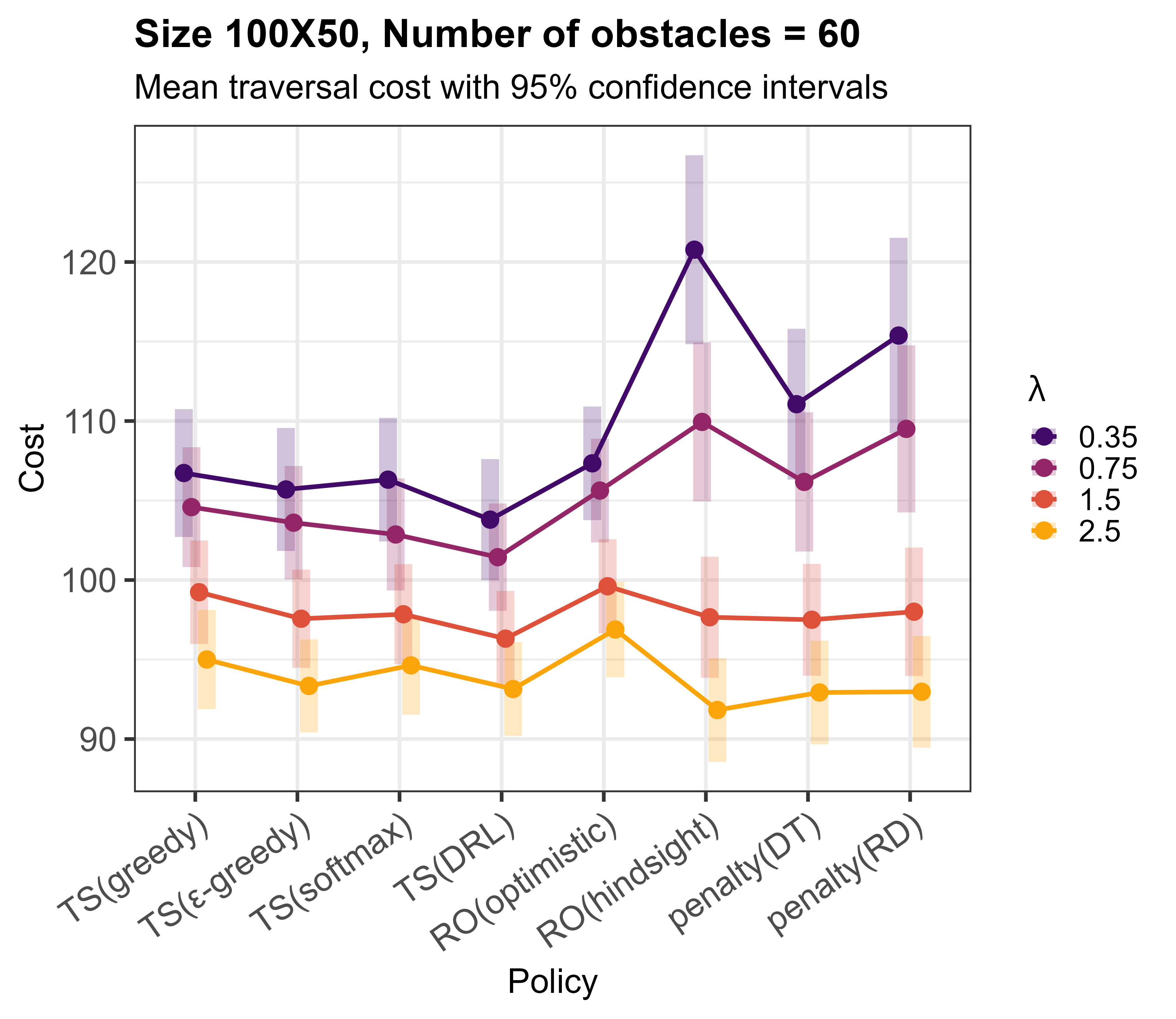}
	\end{tabular}
	\caption{The mean traversal cost with 95\% confidence intervals for proposed policy variants and baselines}
	\label{fig: avg cost_lambda}
\end{figure}

Figures \ref{fig: avg cost_range} and \ref{fig: median_cost_range} show the mean and median traversal costs by sensor range.
The traversal cost decreases as sensor range increases, with reduction magnitude growing with sensor range.
This validates the effectiveness of our Bayesian updating framework,
it converts additional observations into better decisions, 
and the gains are further amplified by incorporating correlation information. 

Beyond mean and median cost, we assess policy performance relative to the offline-optimal benchmark that assumes perfect knowledge of obstacle status.
Figure \ref{fig: avg deviation} presents the average optimality gap (i.e., $\mathcal{C}-\mathcal{C}_\text{optimal}$) across environments,
revealing consistent performance rankings.
The two-stage framework with DRL base achieves the smallest optimality gap across most environments,
with this advantage most pronounced in larger, denser and more noised environments where uncertainty management is harder.
The Monte Carlo bases, decaying $\epsilon$-greedy and softmax approaches, rank following the DRL base,
maintaining competitive performance.
The optimistic rollout policy shows comparable performance in smaller settings,
and other baselines show comparable performance only when sensor accuracy is highest.
Results aggregated by sensor range 
shows similar pattern and are presented in the Appendix (Figure \ref{fig: avg deviation_range}).

\begin{figure}[H]
	\centering
	\begin{tabular}{cc}
		\includegraphics[width=0.44\textwidth]{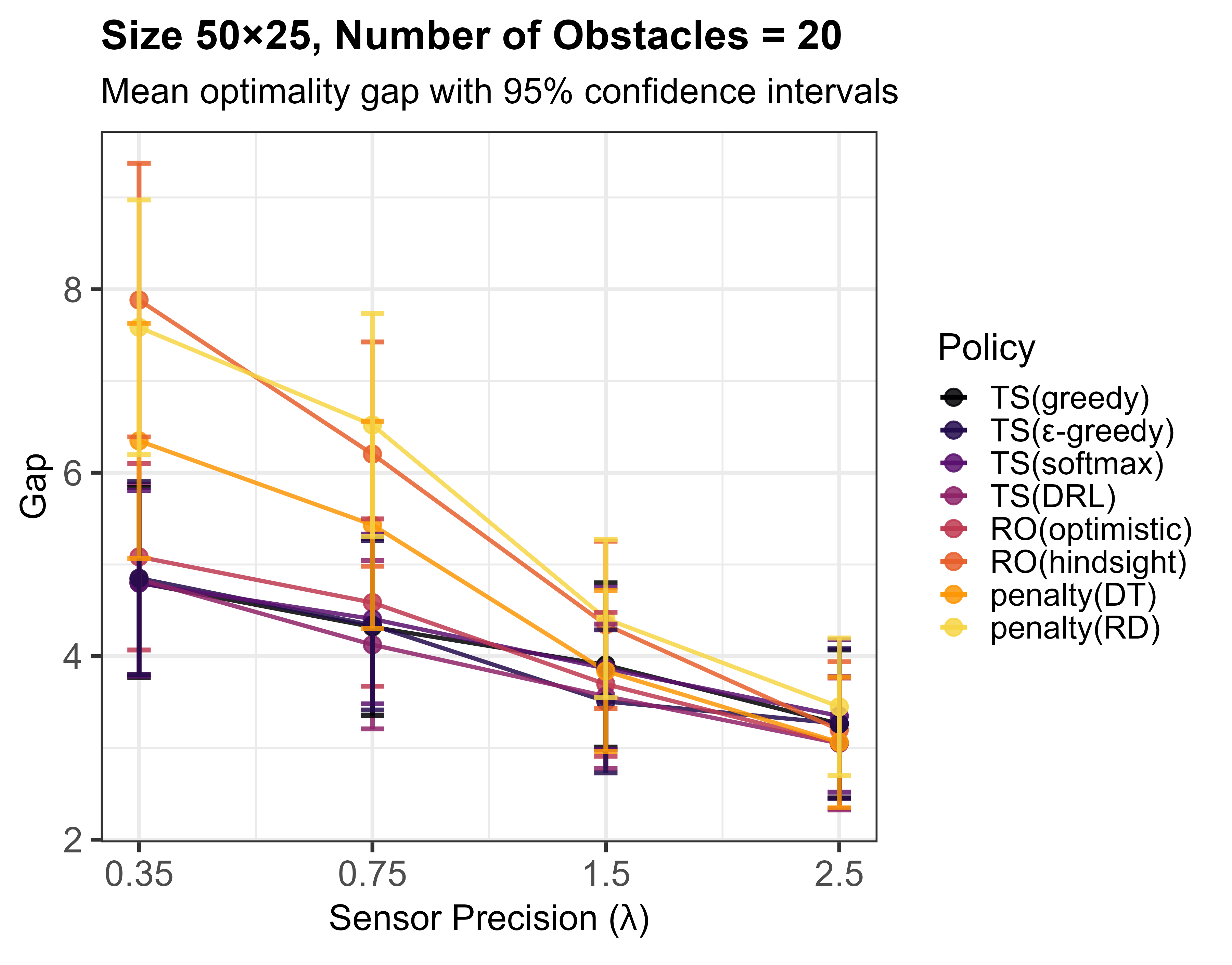} &
		\includegraphics[width=0.44\textwidth]{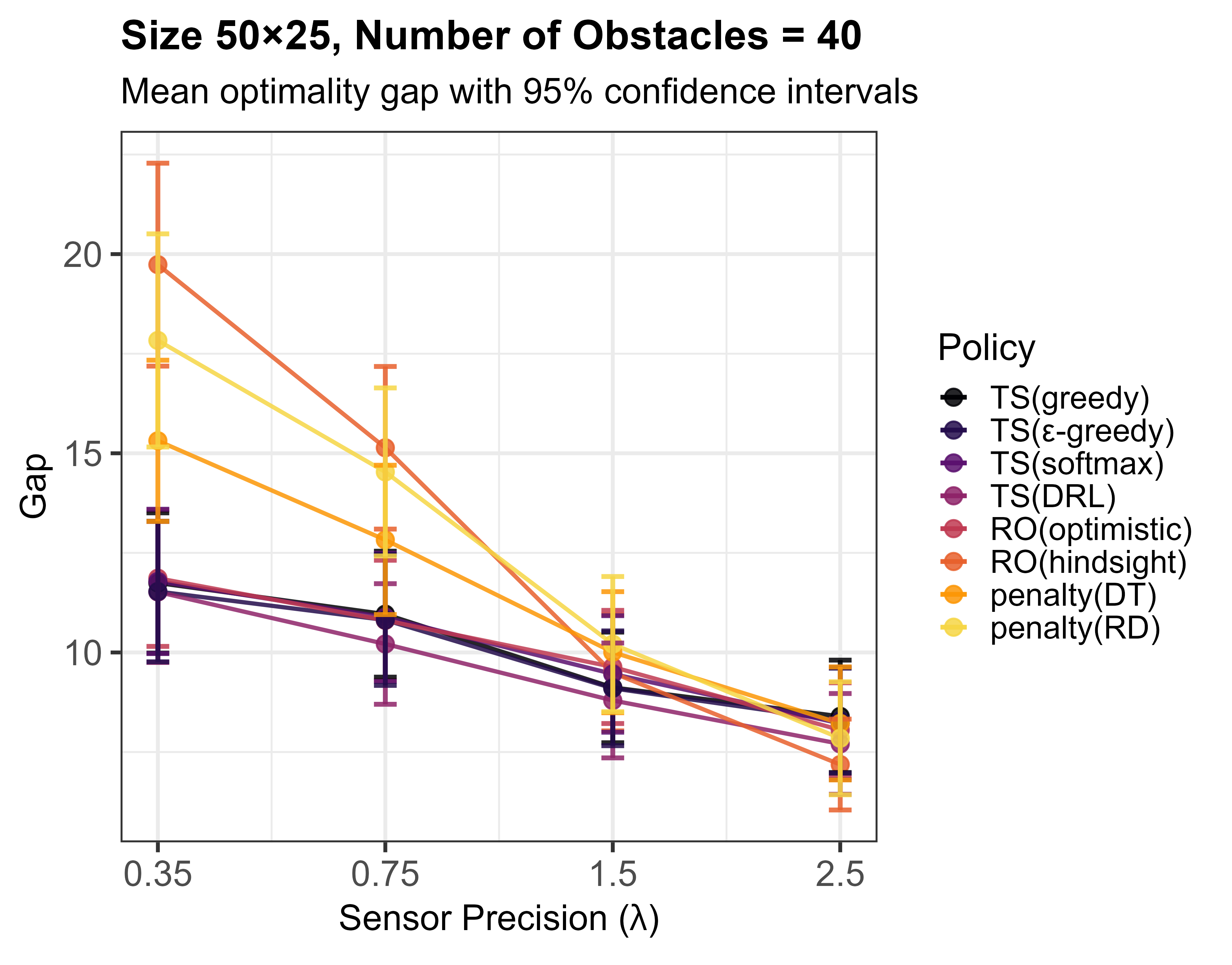}\\
		\includegraphics[width=0.44\textwidth]{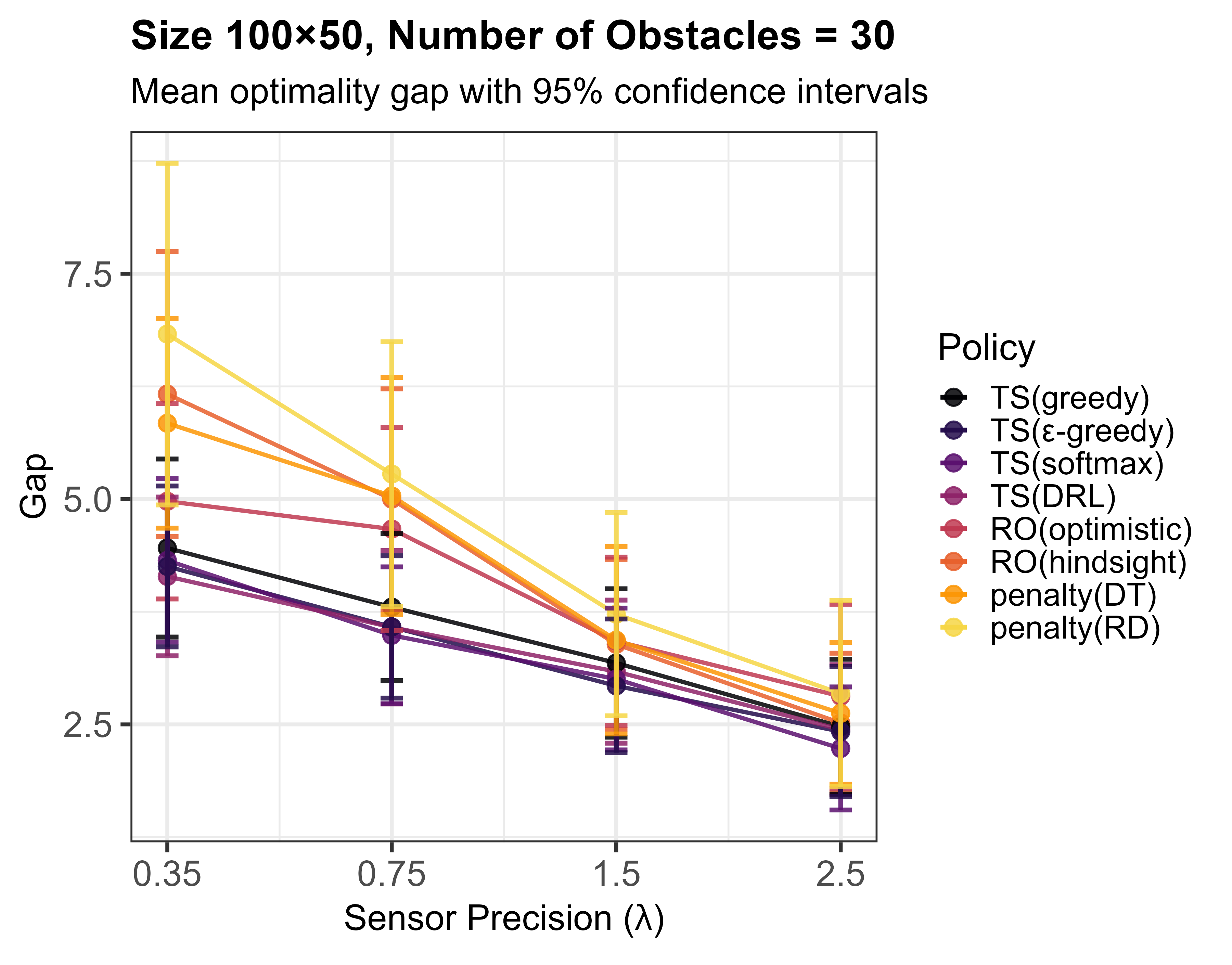} &
		\includegraphics[width=0.44\textwidth]{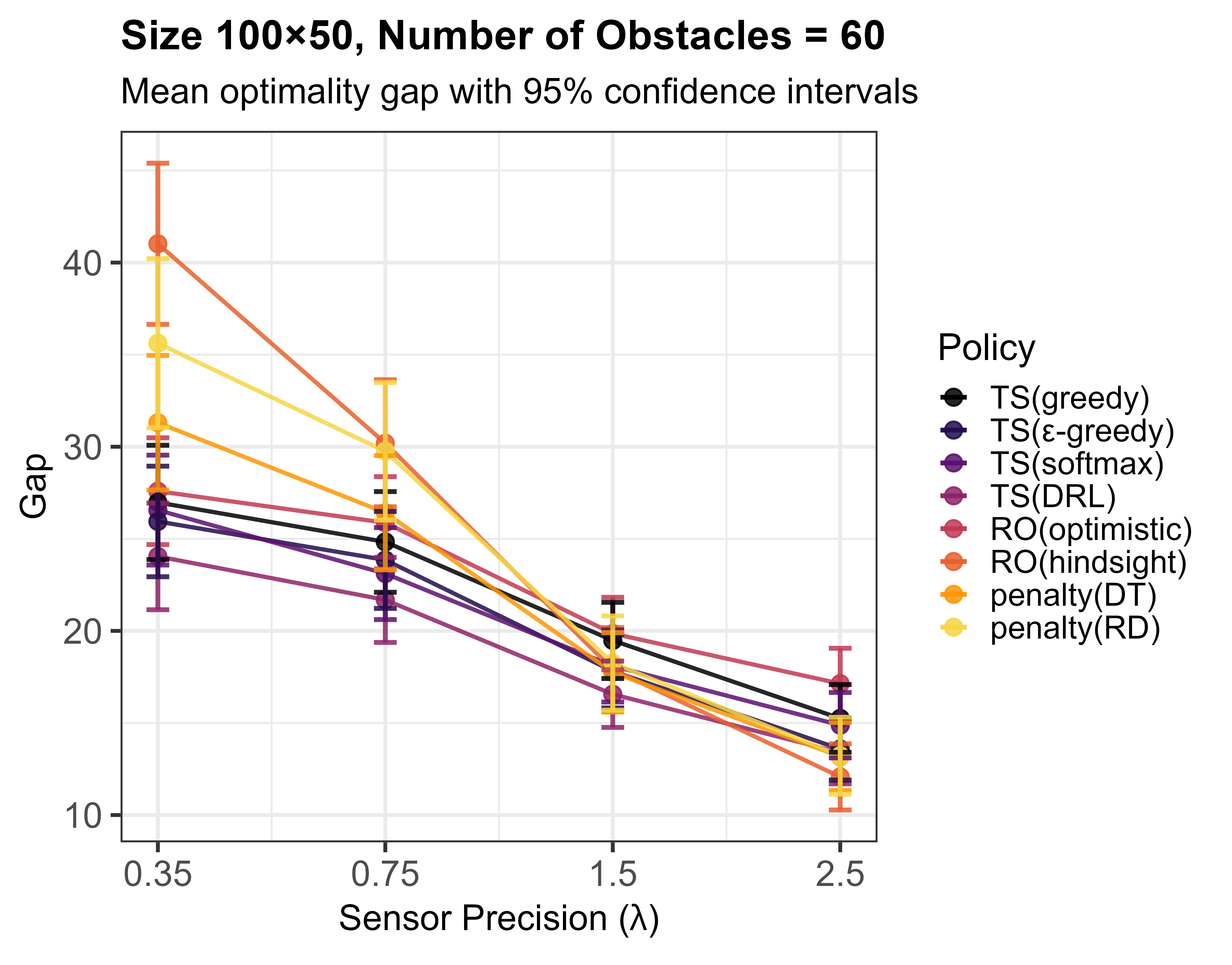}\\
	\end{tabular}
	\caption{The average deviation from optimal solutions with 95\% confidence intervals for proposed policy variants and baselines}
	\label{fig: avg deviation}
\end{figure}

Figure \ref{fig:sd by env} presents the standard deviation of traversal costs 
across replicates within each environment and across environments,
providing insights about policy's robustness.
DRL demonstrates higher consistency with lower variance,
which is particularly crucial in applications requiring predictable policy behavior.
All Monte Carlo bases show comparable robustness with slightly higher variance.
Among baselines,
only the optimistic rollout policy maintains reasonable consistency, 
while other three baselines exhibit substantially higher variance, which is getting worse in more challenging environments.

\begin{figure}[t]
	\centering
	\begin{tabular}{cc}
		\includegraphics[width=0.49\textwidth]{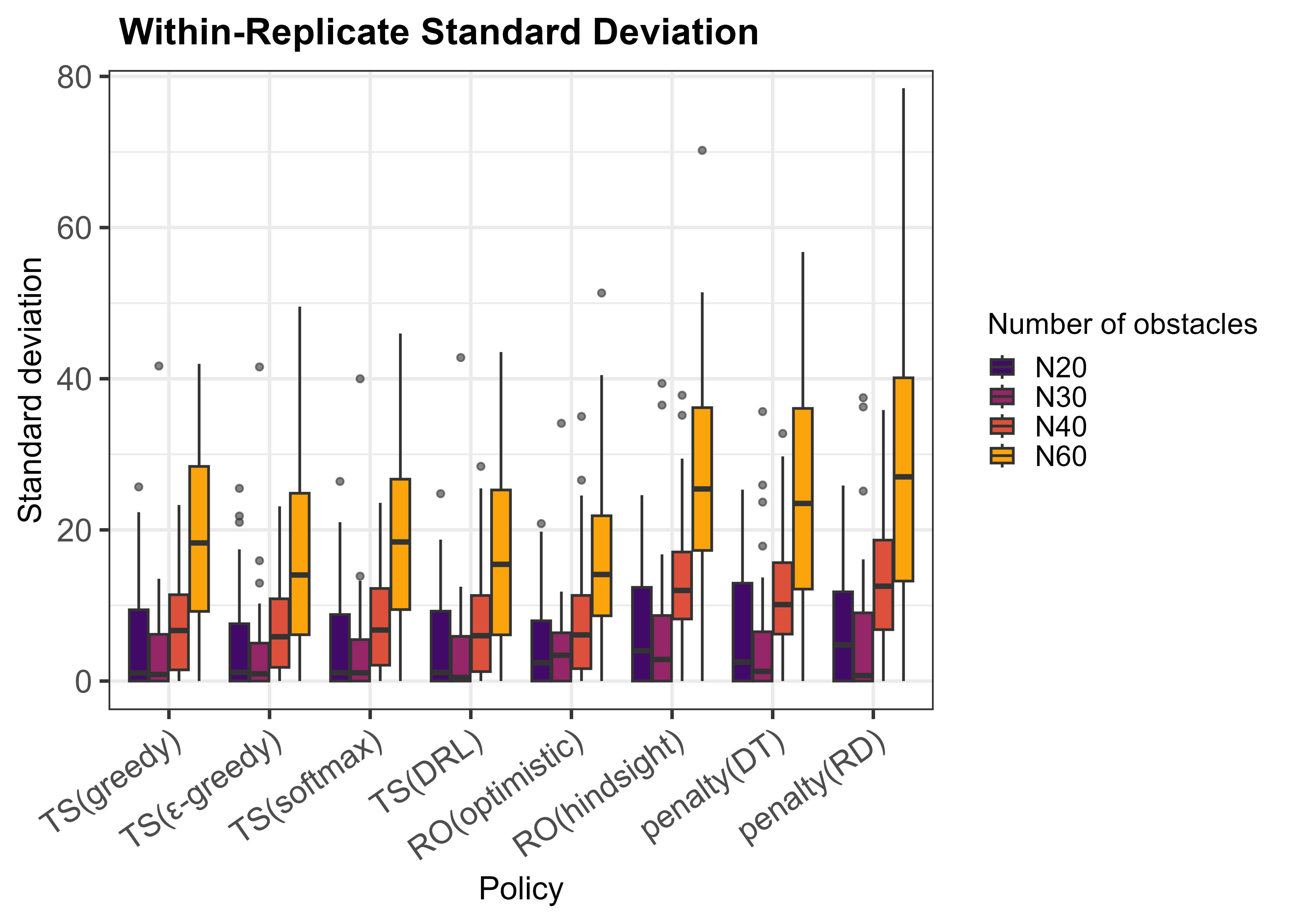} &
		\includegraphics[width=0.49\textwidth]{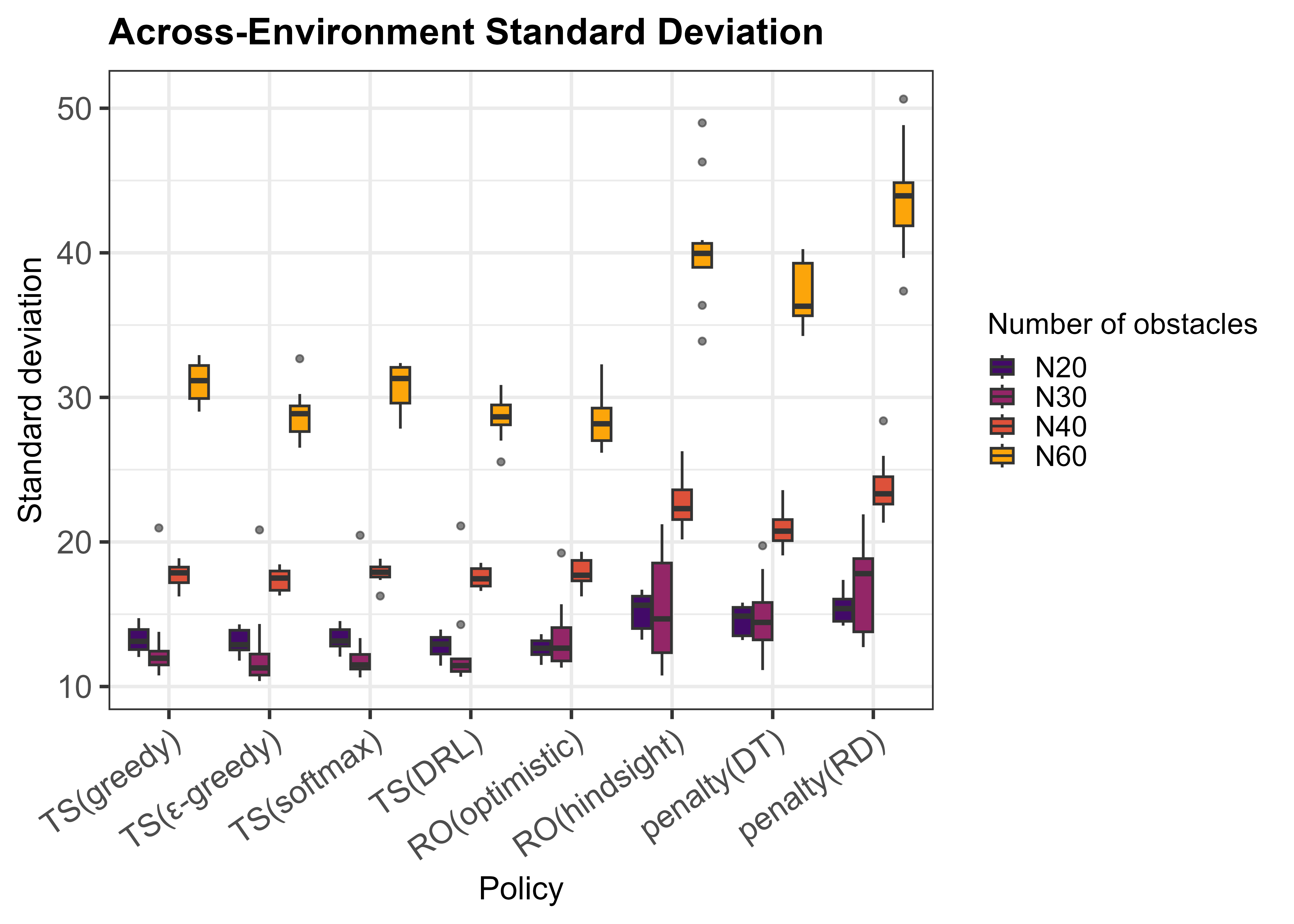}
	\end{tabular}
	\caption{The standard deviation of traversal cost across replicates and environments for proposed policy variants and baselines}
	\label{fig:sd by env}
\end{figure}

Table \ref{tab: simulation_time} summarizes the average simulation time per complete traversal.
Both online simulation time and offline training time grow with obstacle density and traversal region size for all policies.
Within the two-stage policy learning framework,
the DRL base takes longer computation time than Monte Carlo bases,
reflecting the additional computational cost associated with distributional updates,
with the smaller gap in small instances.
By contrast, the decay $\epsilon$-greedy and softmax bases achieve comparable traversal performance at substantially lower computation time.
Their online time is often comparable to, or shorter than, rollout baselines due to the offline training effort.
Rollout baselines require online simulation time on the similar scale (or slightly faster),
but come with higher traversal costs and variability.
Penalty policies, which use direct cost approximation without iterative learning,
run extremely fast but associate with much worse and less stable performance.

\begin{table}[H]
	\centering
	\caption{Average simulation time (in seconds) for one complete traversal}
	\label{tab: simulation_time}
	
	\renewcommand{\arraystretch}{0.95}   % tighter row height (optional)
	\setlength{\tabcolsep}{3pt}          % tighter column padding (optional)
	
	\begin{subtable}{\linewidth}
		\centering
		\caption{Online}
		\footnotesize
		\resizebox{\linewidth}{!}{%
			\begin{tabular}{lcccccccc}
				\toprule
				$N$ & TS(greedy) & TS($\varepsilon$-greedy) & TS(softmax) & TS(DRL) &
				RO(optimism) & RO(hindsight) & penalty(RD) & penalty(DT) \\
				\midrule
				20 & 9.22 (0.65)  & 7.92 (0.56)  & 9.56 (0.67)   & 18.09 (1.26) & 13.28 (0.77) & 16.49 (1.03) & 0.05 (0.01) & 0.04 (0.01) \\
				30 & 20.02 (2.29) & 17.13 (2.13) & 18.86 (2.13)  & 43.04 (4.92) & 40.24 (5.25) & 17.81 (1.72) & 0.59 (0.03) & 0.48 (0.04) \\
				40 & 71.16 (5.70) & 57.16 (5.66) & 72.48 (5.12)  & 208.67 (15.49) & 74.29 (3.28) & 70.61 (4.85) & 0.36 (0.02) & 0.22 (0.01) \\
				60 & 316.57 (17.02) & 195.13 (15.27) & 311.72 (15.38) & 557.36 (30.04) &
				194.80 (10.32) & 135.31 (5.05) & 0.43 (0.01) & 0.26 (0.01) \\
				\bottomrule
		\end{tabular}}
	\end{subtable}
	
	\vspace{0.6em}
	
	\begin{subtable}{0.52\linewidth}
		\centering
		\caption{Offline}
		\footnotesize
		\resizebox{\linewidth}{!}{%
			\begin{tabular}{lcccc}
				\toprule
				$N$ & TS(greedy) & TS($\varepsilon$-greedy) & TS(softmax) & TS(DRL) \\
				\midrule
				20 & 19.03 (0.93) & 19.75 (0.95) & 19.77 (0.96) & 28.97 (1.35) \\
				30 & 56.49 (4.38) & 61.40 (4.75) & 57.10 (4.41) & 99.58 (7.38) \\
				40 & 89.02 (3.49) & 85.15 (3.37) & 92.87 (3.53) & 252.08 (7.75) \\
				60 & 347.95 (12.86) & 282.38 (11.18) & 343.15 (12.49) & 477.80 (17.97) \\
				\bottomrule
		\end{tabular}}
	\end{subtable}
\end{table}

\subsection{Summary of Key Empirical Findings}
\paragraph{Bayesian updating framework.}
The proposed Bayesian updating process demonstrates effectiveness in using sensor readings to improve decisions,
with greater benefits under challenging conditions where
uncertainty handling and information efficiency is crucial.
Consistent with Theorem \ref{thm:coarsening-dominance} and Corollary \ref{cor:monotone},
which show that added observations reduce the expected cost
and correlation-aware updating process amplify these gains.

\paragraph{Two-stage policy learning framework.}
Across environments, the two-stage framework outperforms rollout and penalty baselines on mean, median cost, optimality gap, and standard deviations,
with benefits most pronounced under more challenging conditions (high uncertainty and noise, dense obstacles).
This aligned with our theoretical analysis in Section \ref{sec: offline_property}, 
where we show the convergence property and the exploration benefit by incorporating information gain.

\paragraph{Two-stage framework with distributional RL base.}
DRL base appears to be the best performer, achieving the lowest mean traversal costs,
smallest optimality gaps and great consistency across majority of environmental settings.
This shows that its ability to handle uncertainty through distribution learning is particularly beneficial in high-noise environments where sensor information is unreliable.
Its computational requirement is justified by the robust gains in accuracy and stability.

\paragraph{Two-stage framework with Monte Carlo bases.}
While showing slightly worse performance compared to DRL,
Monte Carlo bases provide a balance between performance and computational efficiency,
making them suitable for settings where computational constraints are tight.
With the flexibility offered by exploration parameters in decaying $\epsilon$-greedy and softmax approaches,
they have the potential to achieve performance improvement through strategic tuning,
which is a promising direction for future research.  

\paragraph{Comparison baselines.}
The optimistic rollout policy shows promising performance in simpler cases but shows higher costs and variability in challenging settings,
limiting its practical applicability.
The penalty policies provide great implementation simplicity but underperform other iterative learning approaches,
making them suitable only when the computation cost is the dominant constraint or the environment is simple with accurate sensors.
The hindsight policy consistently shows poor performance across all metrics and is not ideal in practical applications.

\section{Discussion and Conclusions}\label{sec: conclusion}
We address a path-planning problem in complex environments of limited and uncertain information. 
We introduce the Stochastic Correlated Obstacle Scene (SCOS) problem,
extending the Stochastic Obstacle Scene (SOS) problem by incorporating realistic obstacle spatial patterns and practical sensor constraints. 
To overcome limitations in existing planning policies, 
we propose a two-stage policy learning framework that integrates an offline training phase guided by information gain and an online decision phase.
In the offline phase, 
we learn a robust base policy via optimistic policy iteration augmented with information bonus to encourage exploration in uncertainty regions.
while efficient online rollout policy is applied for real-time decision making, followed by base policy adjustment.
This framework systematically balances exploration-exploitation trade-offs and is supported by theoretical analysis.

Our contributions can be summarized across three key aspects.
First, we formulate the SCOS problem as a more applicable framework for path-planing problems involving adversarial interruption and information uncertainty.
Second, we develop an novel two-stage policy learning framework:
offline learning enhanced with information bonus built upon mutual information for better exploration, 
with online rollout with periodic base updates for new information adaptation.
This strategy yields robust policies with theoretical guarantees,
supporting both Monte Carlo estimates and distributional RL for full distribution approximation.
Third, using Gaussian random field model, 
we incorporate a Bayesian updating framework for information refinement,
which not only enhances the decision-making,
but also supports the search space reduction step to improve the computational efficiency.

Comprehensive empirical results demonstrate substantial performance improvements over existing baseline policies.
In terms of solution quality,
the proposed two-stage strategy achieves lower traversal costs and smaller optimality gaps across environments.
The distributional RL base shows strongest performance in challenging environments with high noise and more obstacles,
and its advantage grows as environmental complexity (i.e., noise and obstacle dense level) increases.
The Monte Carlo bases are often close to DRL in traversal performance.
Regarding computation time,
empirical results show a clear trade-off between solution quality and speed.
While distribution RL approach requires longer computation time, it offers consistency and robustness of performance,
whereas Monte Carlo approaches provides alternatives for applications with computational constraints.

Despite these improvements,
we have various directions for future research.
(i) \emph{Computational scalability:}
the computational requirements of distribution RL approach may still limit its applicability in time critical scenarios,
a more efficient distribution approximation techniques can be explored to enhance scalability.
(ii) \emph{Exploration parameter tuning:}
the incorporation of information gain $G$ in our framework requires careful tuning of its weight in decision making to 
avoid being conservative and missing beneficial but uncertain paths.
(iii) \emph{Dynamic environment extension:}
the current framework assumes static obstacle locations and status,
while real-world applications often involve dynamic environments where obstacles status or locations evolve over time, 
posing additional planning challenge.
(iv) \emph{Multi-agent extension:} 
extending to multi-agent scenarios can greatly increase its applicability 
but which would require effective coordination or competing strategies between agents.

\bibliographystyle{apalike}
\bibliography{LZbib}

\begin{thebibliography}{}

\bibitem[Aksakalli and Ari, 2014]{aksakalliari2014}
Aksakalli, V. and Ari, I. (2014).
\newblock Penalty-based algorithms for the stochastic obstacle scene problem.
\newblock {\em INFORMS Journal on Computing}, 26:370--384.

\bibitem[Aksakalli et~al., 2011]{aksakalli2011}
Aksakalli, V., Fishkind, D.~E., Priebe, C.~E., and Ye, X. (2011).
\newblock The reset disambiguation policy for navigating stochastic obstacle
  fields.
\newblock {\em Naval Research Logistics}, 58(4):389--399.

\bibitem[Aksakalli et~al., 2016]{aksakalli2016based}
Aksakalli, V., Sahin, O.~F., and Ari, I. (2016).
\newblock An {AO*} based exact algorithm for the {Canadian} traveler problem.
\newblock {\em INFORMS Journal on Computing}, 28(1):96--111.

\bibitem[Alkaya et~al., 2015]{alkaya2015penalty}
Alkaya, A.~F., Aksakalli, V., and Priebe, C.~E. (2015).
\newblock A penalty search algorithm for the obstacle neutralization problem.
\newblock {\em Computers \& Operations Research}, 53:165--175.

\bibitem[Alkaya and Oz, 2017]{alkaya2017optimal}
Alkaya, A.~F. and Oz, D. (2017).
\newblock An optimal algorithm for the obstacle neutralization problem.
\newblock {\em Journal of Industrial \& Management Optimization}, 13(2).

\bibitem[Alkaya et~al., 2021]{alkaya2021heuristics}
Alkaya, A.~F., Yildirim, S., and Aksakalli, V. (2021).
\newblock Heuristics for the {Canadian} traveler problem with neutralizations.
\newblock {\em Computers \& Industrial Engineering}, 159:107488.

\bibitem[Aslan et~al., 2020]{aslan2019any}
Aslan, U., Alkaya, A.~F., Yildirim, S., and Aksakalli, V. (2020).
\newblock Any angle path finding in stochastic obstacle scenes.
\newblock In {\em Proceedings of the 3rd International Conference on Advances
  in Artificial Intelligence}, ICAAI '19, page 122–126, New York, NY, USA.
  Association for Computing Machinery.

\bibitem[Azizi and Seifi, 2024]{azizi2024shortest}
Azizi, E. and Seifi, A. (2024).
\newblock Shortest path network interdiction with incomplete information: a
  robust optimization approach.
\newblock {\em Annals of Operations Research}, 335(2):727--759.

\bibitem[Bar-Noy and Schieber, 1991]{bar1991canadian}
Bar-Noy, A. and Schieber, B. (1991).
\newblock The {Canadian} {Traveller} {Problem}.
\newblock In {\em Proceedings of the Second Annual ACM-SIAM Symposium on
  Discrete Algorithms}, SODA '91, page 261–270, USA. Society for Industrial
  and Applied Mathematics.

\bibitem[Bellemare et~al., 2017]{bellemare2017distributional}
Bellemare, M.~G., Dabney, W., and Munos, R. (2017).
\newblock A distributional perspective on reinforcement learning.
\newblock In {\em International {Conference} on {Machine} {Learning}}, pages
  449--458. PMLR.

\bibitem[Berger et~al., 2012]{berger2012new}
Berger, J., Boukhtouta, A., Benmoussa, A., and Kettani, O. (2012).
\newblock A new mixed-integer linear programming model for rescue path planning
  in uncertain adversarial environment.
\newblock {\em Computers \& Operations Research}, 39(12):3420--3430.

\bibitem[Bertsekas, 2021]{bertsekas2021rollout}
Bertsekas, D. (2021).
\newblock {\em Rollout, policy iteration, and distributed reinforcement
  learning}.
\newblock Athena Scientific.

\bibitem[Blumenthal and Shani, 2023]{blumenthal2023rollout}
Blumenthal, O. and Shani, G. (2023).
\newblock Rollout heuristics for online stochastic contingent planning.
\newblock {\em arXiv preprint arXiv:2310.02345}.

\bibitem[Bnaya et~al., 2009]{bnaya2009canadian}
Bnaya, Z., Felner, A., and Shimony, S.~E. (2009).
\newblock Canadian traveler problem with remote sensing.
\newblock In {\em IJCAI}, pages 437--442.

\bibitem[Bonet, 2012]{bonet2012deterministic}
Bonet, B. (2012).
\newblock Deterministic {POMDPs} revisited.
\newblock {\em arXiv preprint arXiv:1205.2659}.

\bibitem[Chaloner and Verdinelli, 1995]{chaloner1995bayesian}
Chaloner, K. and Verdinelli, I. (1995).
\newblock Bayesian experimental design: A review.
\newblock {\em Statistical Science}, pages 273--304.

\bibitem[Chen, 2018]{chen2018convergence}
Chen, Y. (2018).
\newblock On the convergence of optimistic policy iteration for stochastic
  shortest path problem.
\newblock {\em arXiv preprint arXiv:1808.08763}.

\bibitem[Contal et~al., 2014]{contal2014gaussian}
Contal, E., Perchet, V., and Vayatis, N. (2014).
\newblock Gaussian process optimization with mutual information.
\newblock In Xing, E.~P. and Jebara, T., editors, {\em Proceedings of the 31st
  International Conference on Machine Learning}, volume~32 of {\em Proceedings
  of Machine Learning Research}, pages 253--261, Bejing, China. PMLR.

\bibitem[Dabney et~al., 2018]{dabney2018distributional}
Dabney, W., Rowland, M., Bellemare, M., and Munos, R. (2018).
\newblock Distributional reinforcement learning with quantile regression.
\newblock In {\em Proceedings of the AAAI conference on Artificial
  Intelligence}, volume~32.

\bibitem[Dey et~al., 2014]{dey2014gauss}
Dey, D., Kolobov, A., Caruana, R., Kamar, E., Horvitz, E., and Kapoor, A.
  (2014).
\newblock Gauss meets {Canadian} traveler: Shortest-path problems with
  correlated natural dynamics.
\newblock In {\em AAMAS 2014}, pages 1101--1108. AAMAS.

\bibitem[Eyerich et~al., 2010]{eyerich2010high}
Eyerich, P., Keller, T., and Helmert, M. (2010).
\newblock High-quality policies for the {Canadian} traveler's problem.
\newblock In {\em Proceedings of the AAAI Conference on Artificial
  Intelligence}, volume 24(1), pages 51--58.

\bibitem[Feinberg et~al., 2018]{feinberg2018model}
Feinberg, V., Wan, A., Stoica, I., Jordan, M.~I., Gonzalez, J.~E., and Levine,
  S. (2018).
\newblock Model-based value estimation for efficient model-free reinforcement
  learning.
\newblock {\em arXiv preprint arXiv:1803.00101}.

\bibitem[Fisher et~al., 1978]{FisherNemhauserWolsey1978}
Fisher, M.~L., Nemhauser, G.~L., and Wolsey, L.~A. (1978).
\newblock An analysis of approximations for maximizing submodular set
  functions—{II}.
\newblock In Balinski, M.~L. and Hoffman, A.~J., editors, {\em Polyhedral
  Combinatorics}, volume~8 of {\em Mathematical Programming Studies}, pages
  73--87. Springer, Berlin, Heidelberg.

\bibitem[Hickling et~al., 2023]{hickling2023robust}
Hickling, T., Aouf, N., and Spencer, P. (2023).
\newblock Robust adversarial attacks detection based on explainable deep
  reinforcement learning for uav guidance and planning.
\newblock {\em IEEE Transactions on Intelligent Vehicles}.

\bibitem[Hou and Srinivasa, 2022]{hou2022dynamic}
Hou, B. and Srinivasa, S.~S. (2022).
\newblock Dynamic replanning with posterior sampling.
\newblock In {\em 2022 IEEE/RSJ International Conference on Intelligent Robots
  and Systems (IROS)}, pages 2938--2945. IEEE.

\bibitem[Kapoor et~al., 2010]{kapoor2010gaussian}
Kapoor, A., Grauman, K., Urtasun, R., and Darrell, T. (2010).
\newblock Gaussian processes for object categorization.
\newblock {\em International journal of computer vision}, 88:169--188.

\bibitem[Kocsis and Szepesv{\'a}ri, 2006]{kocsis2006bandit}
Kocsis, L. and Szepesv{\'a}ri, C. (2006).
\newblock Bandit based {Monte-Carlo} planning.
\newblock In {\em European conference on machine learning}, pages 282--293.
  Springer.

\bibitem[Koenig and Likhachev, 2002]{koenig2002d}
Koenig, S. and Likhachev, M. (2002).
\newblock D*lite.
\newblock In {\em Eighteenth National Conference on Artificial Intelligence},
  page 476–483, USA. American Association for Artificial Intelligence.

\bibitem[Lamarre and Kelly, 2025]{lamarre2025risk}
Lamarre, O. and Kelly, J. (2025).
\newblock Risk-averse traversal of graphs with stochastic and correlated edge
  costs for safe global planetary mobility.
\newblock {\em arXiv preprint arXiv:2505.13674}.

\bibitem[Li et~al., 2018]{li2018gaussian}
Li, H., Bar{\~a}o, M., and Rato, L. (2018).
\newblock Gaussian random field-based log odds occupancy mapping.
\newblock In {\em 2018 IEEE International Conference on Automation, Quality and
  Testing, Robotics (AQTR)}, pages 1--4. IEEE.

\bibitem[Lim et~al., 2017]{lim2017shortest}
Lim, Z.~W., Hsu, D., Lee, W.~S., and Sun, W. (2017).
\newblock Shortest {Path} under {Uncertainty}: {Exploration} versus
  {Exploitation}.
\newblock In {\em UAI}.

\bibitem[MacDonald and Smith, 2020]{macdonald2020reactive}
MacDonald, R.~A. and Smith, S.~L. (2020).
\newblock Reactive motion planning in uncertain environments via mutual
  information policies.
\newblock In {\em Algorithmic Foundations of Robotics XII: Proceedings of the
  Twelfth Workshop on the Algorithmic Foundations of Robotics}, pages 256--271.
  Springer.

\bibitem[Mavrin et~al., 2019]{mavrin2019distributional}
Mavrin, B., Yao, H., Kong, L., Wu, K., and Yu, Y. (2019).
\newblock Distributional reinforcement learning for efficient exploration.
\newblock In {\em International Conference on Machine Learning}, pages
  4424--4434. PMLR.

\bibitem[Nemhauser et~al., 1978]{NemhauserWolseyFisher1978}
Nemhauser, G.~L., Wolsey, L.~A., and Fisher, M.~L. (1978).
\newblock An analysis of approximations for maximizing submodular set
  functions—{I}.
\newblock {\em Mathematical Programming}, 14:265--294.

\bibitem[Nickisch et~al., 2008]{nickisch2008approximations}
Nickisch, H., Rasmussen, C.~E., et~al. (2008).
\newblock Approximations for binary gaussian process classification.
\newblock {\em Journal of Machine Learning Research}, 9(10):2035--2078.

\bibitem[Osband et~al., 2013]{osband2013more}
Osband, I., Russo, D., and Van~Roy, B. (2013).
\newblock (more) efficient reinforcement learning via posterior sampling.
\newblock In Burges, C., Bottou, L., Welling, M., Ghahramani, Z., and
  Weinberger, K., editors, {\em Advances in Neural Information Processing
  Systems}, volume~26. Curran Associates, Inc.

\bibitem[Osband and Van~Roy, 2017]{osband2017posterior}
Osband, I. and Van~Roy, B. (2017).
\newblock Why is posterior sampling better than optimism for reinforcement
  learning?
\newblock In {\em International conference on machine learning}, pages
  2701--2710. PMLR.

\bibitem[O’Callaghan and Ramos, 2012]{o2012gaussian}
O’Callaghan, S.~T. and Ramos, F.~T. (2012).
\newblock Gaussian process occupancy maps.
\newblock {\em The International Journal of Robotics Research}, 31(1):42--62.

\bibitem[Papadimitriou and Yannakakis, 1991]{papadimitriou1991shortest}
Papadimitriou, C.~H. and Yannakakis, M. (1991).
\newblock Shortest paths without a map.
\newblock {\em Theoretical Computer Science}, 84(1):127--150.

\bibitem[Pinosky et~al., 2023]{pinosky2023hybrid}
Pinosky, A., Abraham, I., Broad, A., Argall, B., and Murphey, T.~D. (2023).
\newblock Hybrid control for combining model-based and model-free reinforcement
  learning.
\newblock {\em The International Journal of Robotics Research}, 42(6):337--355.

\bibitem[Pitilakis et~al., 2016]{pitilakis2016systemic}
Pitilakis, K., Argyroudis, S., Kakderi, K., and Selva, J. (2016).
\newblock Systemic vulnerability and risk assessment of transportation systems
  under natural hazards towards more resilient and robust infrastructures.
\newblock {\em Transportation research procedia}, 14:1335--1344.

\bibitem[Polydoros and Nalpantidis, 2017]{polydoros2017survey}
Polydoros, A.~S. and Nalpantidis, L. (2017).
\newblock Survey of model-based reinforcement learning: Applications on
  robotics.
\newblock {\em Journal of Intelligent \& Robotic Systems}, 86(2):153--173.

\bibitem[Powell, 2019]{powell2019unified}
Powell, W.~B. (2019).
\newblock A unified framework for stochastic optimization.
\newblock {\em European Journal of Operational Research}, 275(3):795--821.

\bibitem[Powell, 2022]{powell2022designing}
Powell, W.~B. (2022).
\newblock Designing lookahead policies for sequential decision problems in
  transportation and logistics.
\newblock {\em IEEE Open Journal of Intelligent Transportation Systems},
  3:313--327.

\bibitem[Priebe et~al., 2005]{fishkind2005}
Priebe, C., Fishkind, D., Abrams, L., and Piatko, C. (2005).
\newblock Random disambiguation paths for traversing a mapped hazard field.
\newblock {\em Naval Research Logistics (NRL)}, 52:285 -- 292.

\bibitem[Sahin and Aksakalli, 2015]{sahin2015comparison}
Sahin, O.~F. and Aksakalli, V. (2015).
\newblock A comparison of penalty and rollout-based algorithms for the
  {Canadian} traveler problem.
\newblock {\em International Journal of Machine Learning and Computing},
  5(4):319.

\bibitem[Shiri and Salman, 2019]{shiri2019online}
Shiri, D. and Salman, F.~S. (2019).
\newblock Online optimization of first-responder routes in disaster response
  logistics.
\newblock {\em IBM Journal of Research and Development}, 64(1/2):13--1.

\bibitem[Silver et~al., 2008]{silver2008sample}
Silver, D., Sutton, R.~S., and M{\"u}ller, M. (2008).
\newblock Sample-based learning and search with permanent and transient
  memories.
\newblock In {\em Proceedings of the 25th international conference on Machine
  learning}, pages 968--975.

\bibitem[Smith and Song, 2020]{smith2020survey}
Smith, J.~C. and Song, Y. (2020).
\newblock A survey of network interdiction models and algorithms.
\newblock {\em European Journal of Operational Research}, 283(3):797--811.

\bibitem[Srinivas et~al., 2010]{srinivas2010gaussian}
Srinivas, N., Krause, A., Kakade, S., and Seeger, M. (2010).
\newblock Gaussian process optimization in the bandit setting: no regret and
  experimental design.
\newblock In {\em Proceedings of the 27th International Conference on
  International Conference on Machine Learning}, ICML'10, page 1015–1022,
  Madison, WI, USA. Omnipress.

\bibitem[Sutton, 1991]{sutton1991dyna}
Sutton, R.~S. (1991).
\newblock Dyna, an integrated architecture for learning, planning, and
  reacting.
\newblock {\em ACM Sigart Bulletin}, 2(4):160--163.

\bibitem[Sutton and Barto, 2018]{sutton2018reinforcement}
Sutton, R.~S. and Barto, A.~G. (2018).
\newblock {\em Reinforcement {Learning}: {A}n {I}ntroduction}.
\newblock MIT press.

\bibitem[Tang and Agrawal, 2018]{tang2018exploration}
Tang, Y. and Agrawal, S. (2018).
\newblock Exploration by distributional reinforcement learning.
\newblock {\em arXiv preprint arXiv:1805.01907}.

\bibitem[Tolpin and Shimony, 2012]{tolpin2012mcts}
Tolpin, D. and Shimony, S. (2012).
\newblock {MCTS} based on simple regret.
\newblock In {\em Proceedings of the AAAI Conference on Artificial
  Intelligence}, volume 26(1), pages 570--576.

\bibitem[Tsitsiklis, 2002]{tsitsiklis2002convergence}
Tsitsiklis, J.~N. (2002).
\newblock On the convergence of optimistic policy iteration.
\newblock {\em Journal of Machine Learning Research}, 3(Jul):59--72.

\bibitem[Wang et~al., 2020]{wang2020mobile}
Wang, B., Liu, Z., Li, Q., and Prorok, A. (2020).
\newblock Mobile robot path planning in dynamic environments through globally
  guided reinforcement learning.
\newblock {\em IEEE Robotics and Automation Letters}, 5(4):6932--6939.

\bibitem[Wang et~al., 2024]{wang2024survey}
Wang, N., Li, X., Zhang, K., Wang, J., and Xie, D. (2024).
\newblock A survey on path planning for autonomous ground vehicles in
  unstructured environments.
\newblock {\em Machines}, 12(1):31.

\bibitem[Winnicki and Srikant, 2023]{winnicki2023convergence}
Winnicki, A. and Srikant, R. (2023).
\newblock On the convergence of policy iteration-based reinforcement learning
  with {Monte} {Carlo} policy evaluation.
\newblock In {\em International Conference on Artificial Intelligence and
  Statistics}, pages 9852--9878. PMLR.

\bibitem[Ye et~al., 2011]{ye2011sensor}
Ye, X., Fishkind, D.~E., Abrams, L., and Priebe, C.~E. (2011).
\newblock Sensor information monotonicity in disambiguation protocols.
\newblock {\em Journal of the Operational Research Society}, 62(1):142--151.

\bibitem[Yu and Bertsekas, 2013]{yu2013q}
Yu, H. and Bertsekas, D.~P. (2013).
\newblock Q-learning and policy iteration algorithms for stochastic shortest
  path problems.
\newblock {\em Annals of Operations Research}, 208(1):95--132.

\end{thebibliography}

\clearpage
\appendix
\section*{Appendix}
\addcontentsline{toc}{section}{Appendix}
\setcounter{figure}{0}
\setcounter{table}{0}
\renewcommand{\thefigure}{A\arabic{figure}}
\renewcommand{\thetable}{A\arabic{table}}

\begin{table}[h]
	\centering
	\caption{Description of notations used in the manuscript}
	\begin{tabular}{|l|l|}
		\hline
		\textbf{Notation} & \textbf{Description} \\
		\hline
		\multicolumn{2}{|c|}{\textbf{Problem Setup}} \\
		\hline
		$\Omega$ & Two-dimensional traversal region \\
		$X$ & Set of all obstacle locations \\
		$X^F$ & Set of false obstacle locations \\
		$X^O$ & Set of true obstacle locations \\
		$X^U$ & Set of ambiguous/uncertain obstacle locations \\
		$\text{radius}(x)$ & Radius of obstacle at location $x$ \\
		$\mathcal{G}$ & Undirected graph imposed over $\Omega$ \\
		$\mathcal{V}(\mathcal{G})$ & Set of vertices (navigation locations) \\
		$\mathcal{E}(\mathcal{G})$ & Set of feasible edges \\
		$s, g$ & Starting and goal vertices \\
		\hline
		\multicolumn{2}{|c|}{\textbf{Sensor and Beliefs}} \\
		\hline
		$z_i\in\{0,1\}$ & Latent true status of obstacle $x_i$ ($1=$ true/blocked, $0=$ false/free)\\
		$\rho_i$ & Posterior blockage probability estimate for obstacle $x_i$ \\
		$\rho^*_i$ & True underlying blockage probability for obstacle $x_i$ \\
		$y_i=\log\frac{\rho_i^*}{1-\rho_i^*}$ & log-odds for obstacle $x_i$\\
		$\tilde{\rho}_i$ (or $\tilde{y}_i$) & Noisy observation for obstacle $x_i$\\
		$R$ & Sensor range \\
		$c(x)$ & Disambiguation cost for obstacle $x$ \\
		\hline
		\multicolumn{2}{|c|}{\textbf{Sequential Decision Formulation}} \\
		\hline
		$S_t = \{V_t, B_t\}$ & State variables at time $t$ \\
		$V_t$ & Agent's physical location at time $t$ \\
		$B_t$ & Belief state at time $t$ \\
		$d_t$ & Decision at time $t$ \\
		$W_t$ & Exogenous information at time $t$ \\
		$\mathcal{P}_{sg}$ & Set of all paths from $s$ to $g$ \\
		$L_p$ & Random variable for traversal length of path $p$ \\
		$C_p$ & Random variable for disambiguation costs of path $p$ \\
		$\pi^*$ & Optimal policy \\
		$T$ & Random arrival time at goal \\
		\hline
	\end{tabular}
    \label{tab:notations}
\end{table}

\begin{figure}[H]
	\centering
	\begin{tabular}{cc}
		\includegraphics[width=0.35\textwidth]{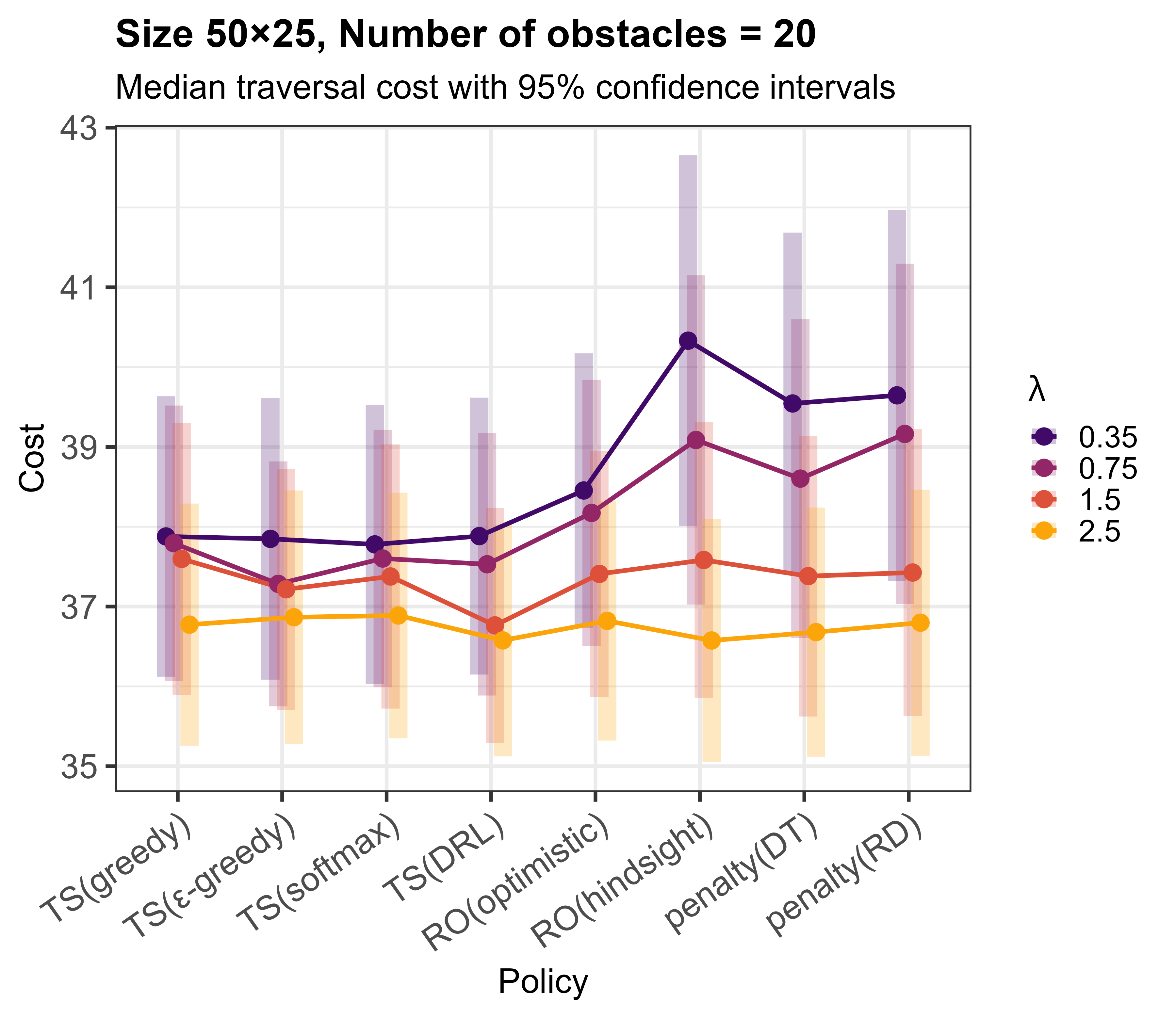} &
		\includegraphics[width=0.35\textwidth]{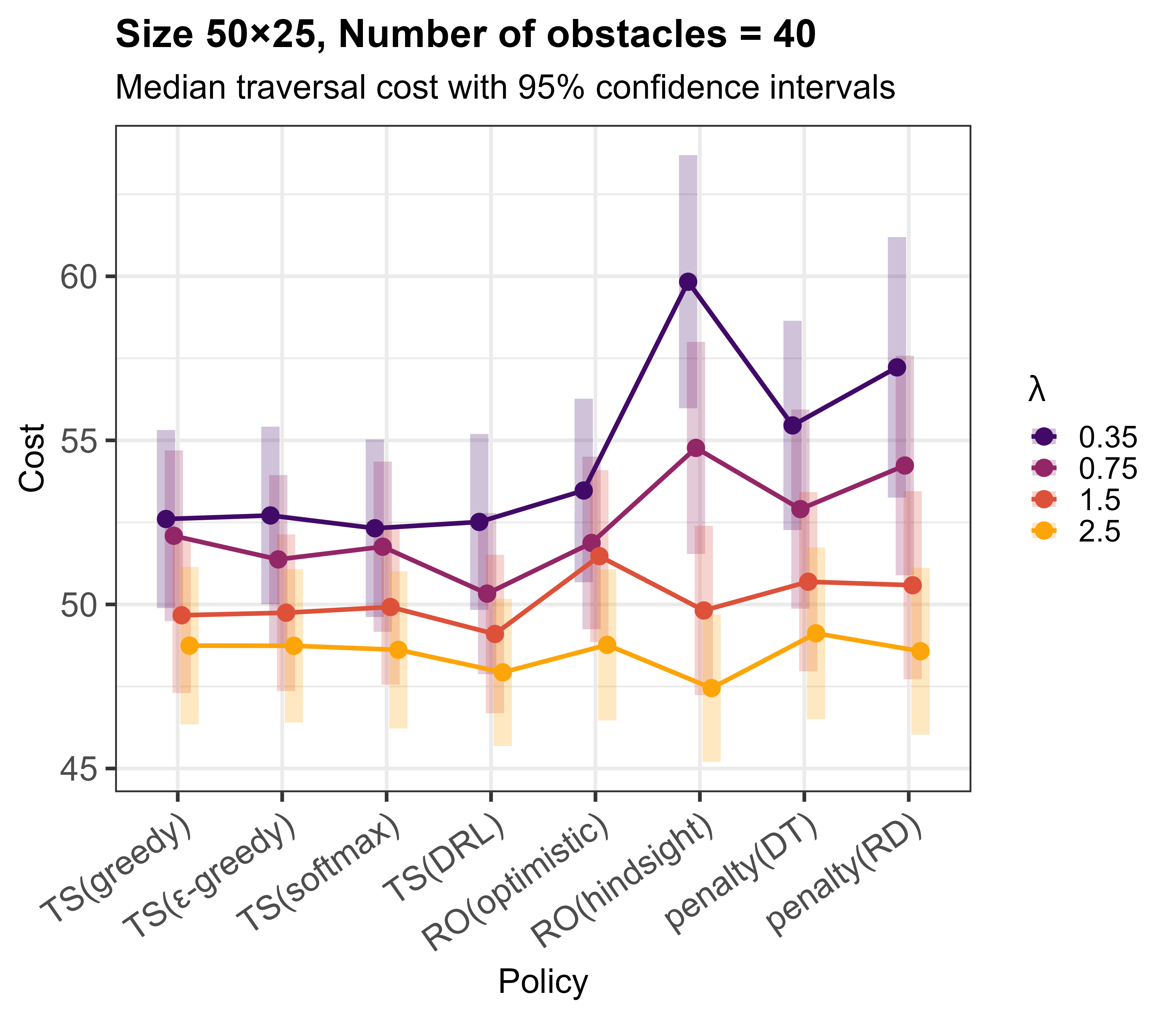} \\
		\includegraphics[width=0.35\textwidth]{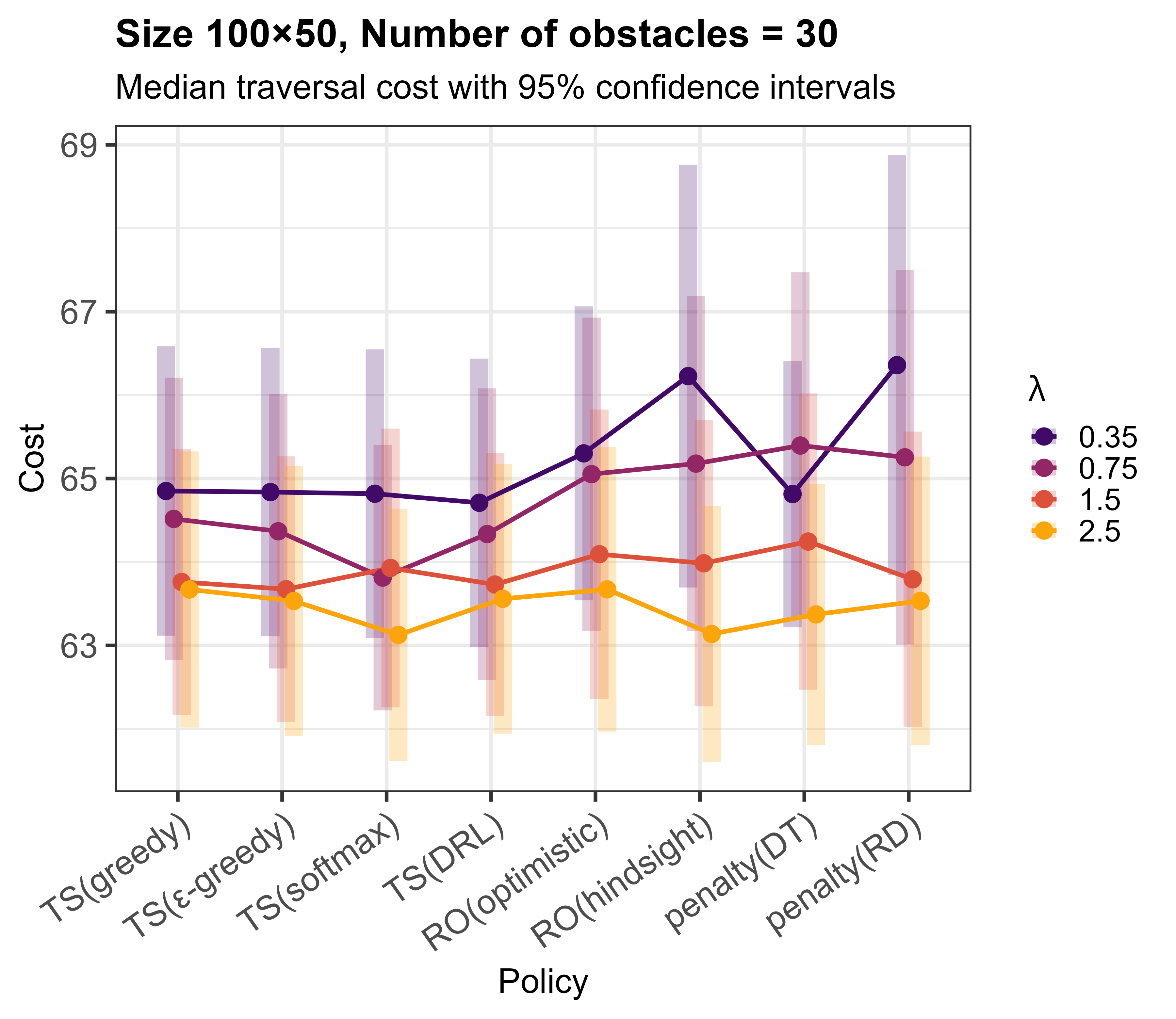} &
		\includegraphics[width=0.35\textwidth]{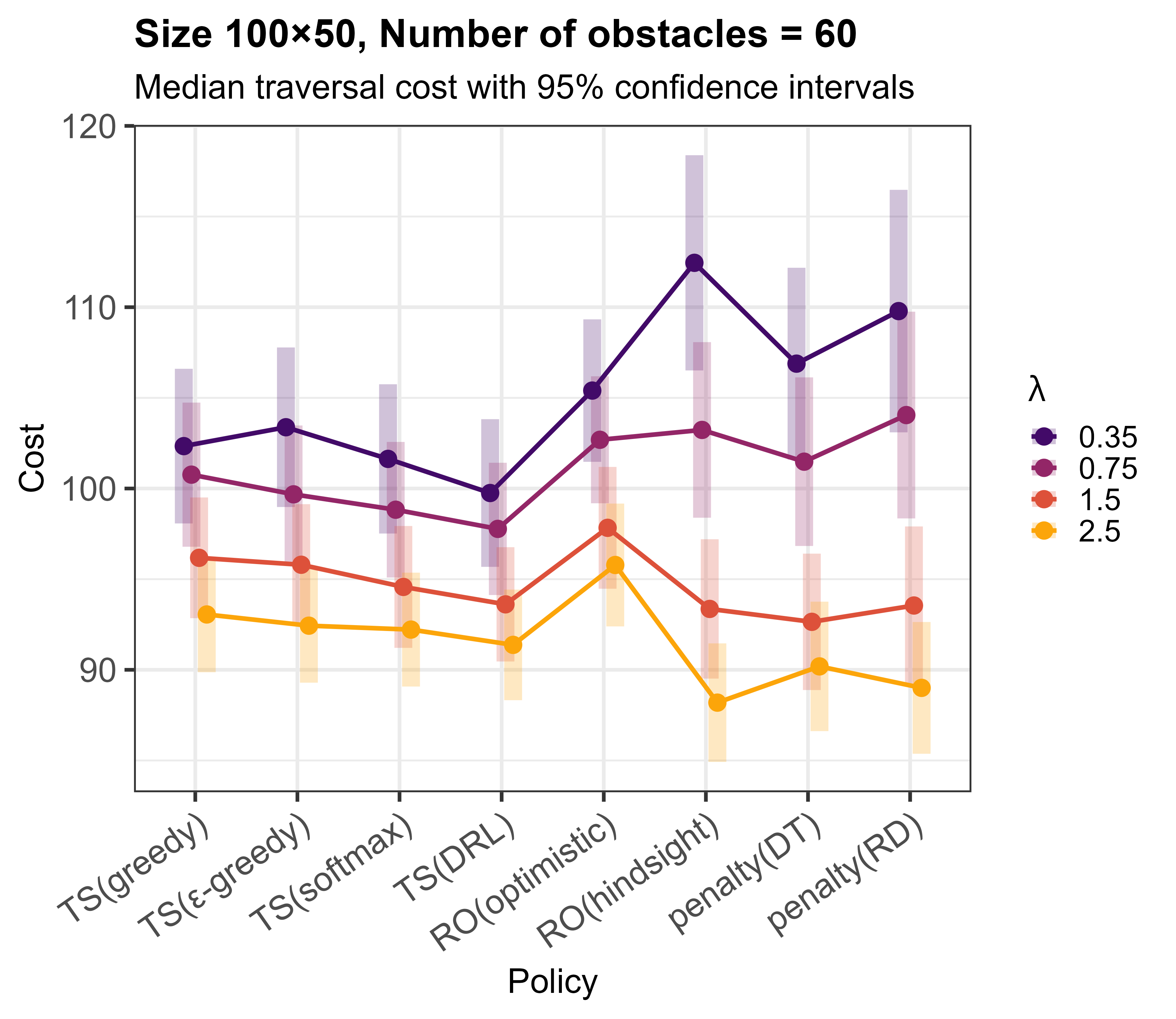}
	\end{tabular}
	\caption{The median traversal cost with 95\% confidence intervals for proposed policy variants and baselines by sensing precision}
	\label{fig: median_cost_lambda}
\end{figure}

\begin{figure}[H]
	\centering
	\begin{tabular}{cc}
		\includegraphics[width=0.35\textwidth]{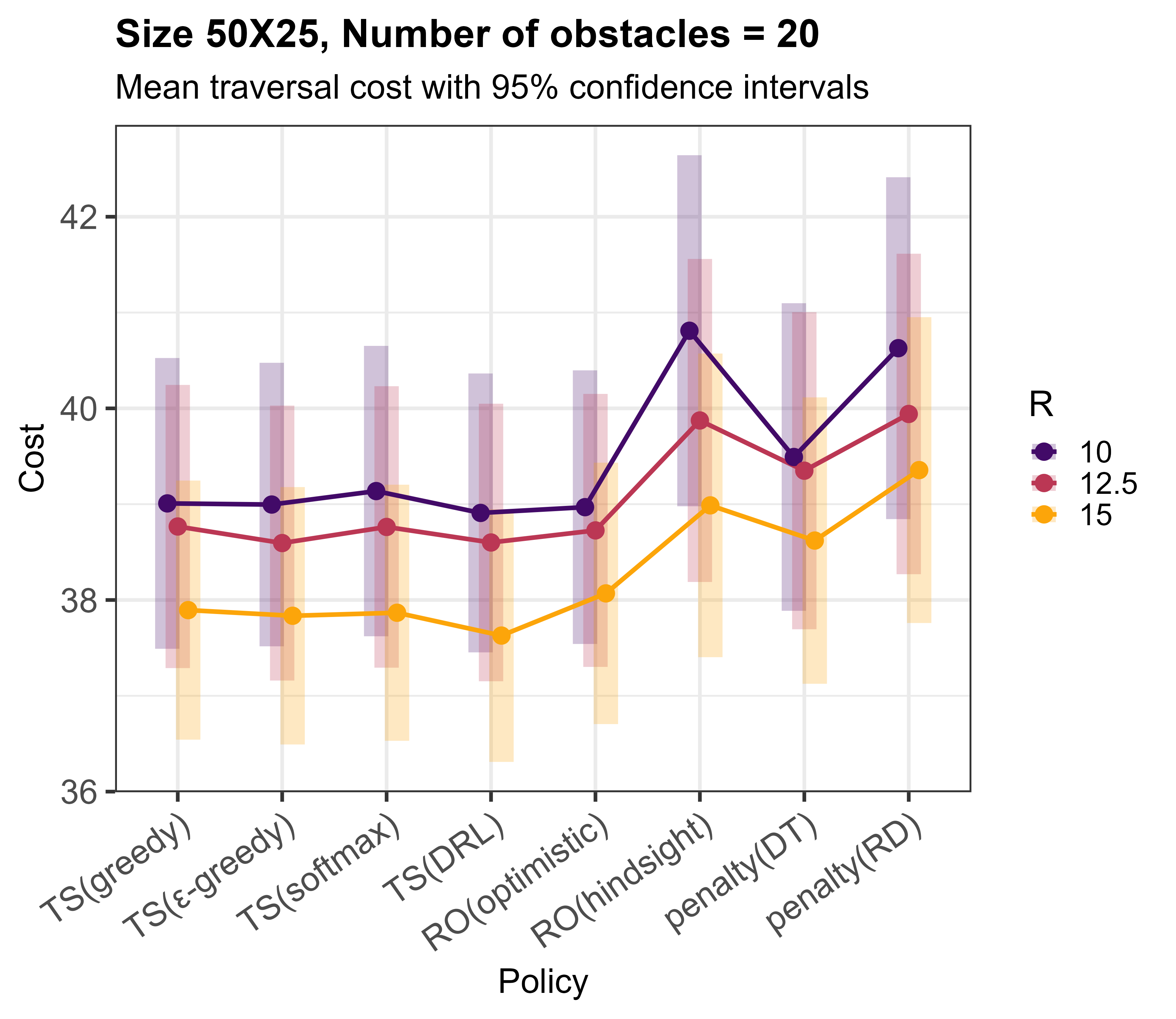} &
		\includegraphics[width=0.35\textwidth]{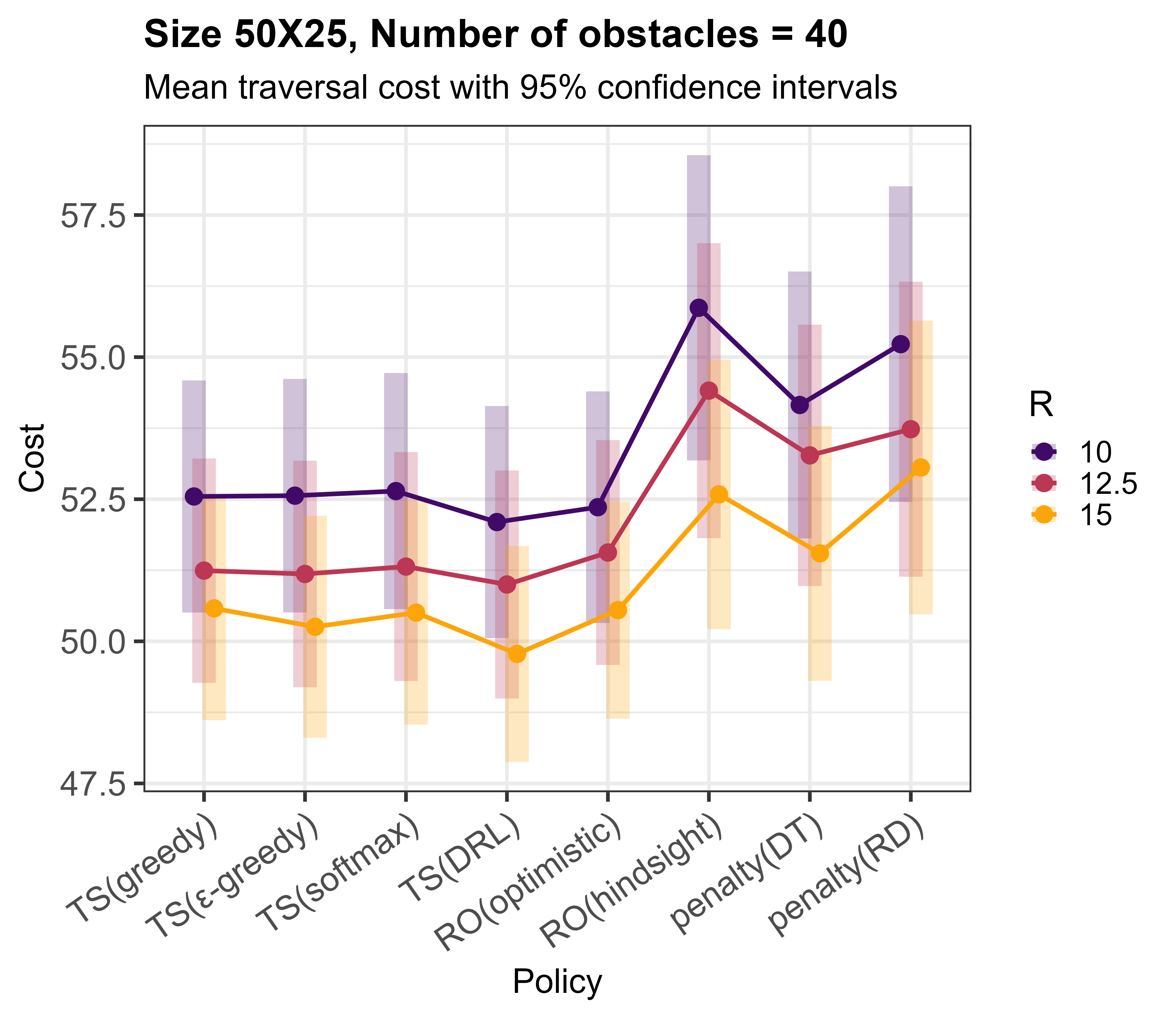} \\
		\includegraphics[width=0.35\textwidth]{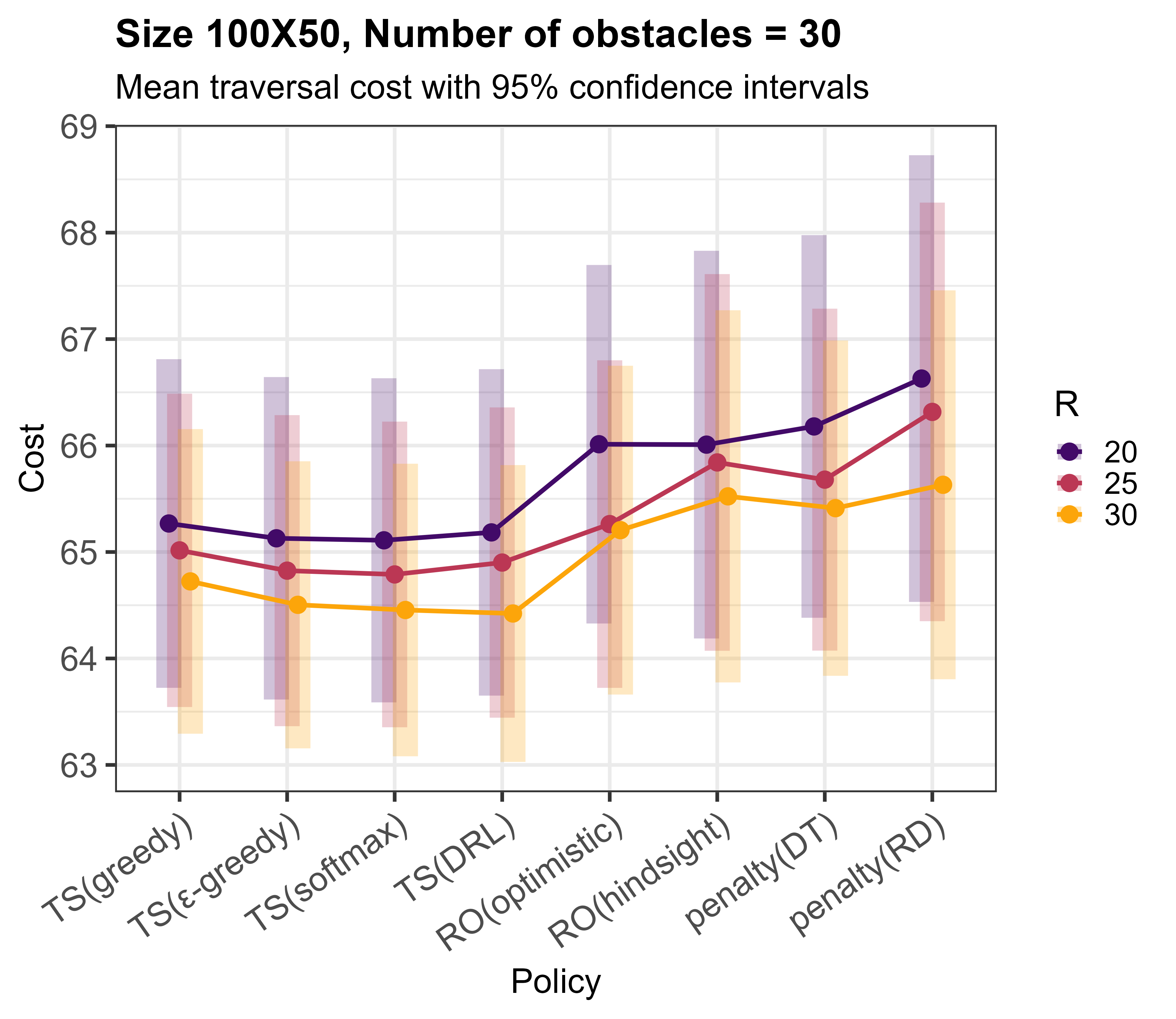} &
		\includegraphics[width=0.35\textwidth]{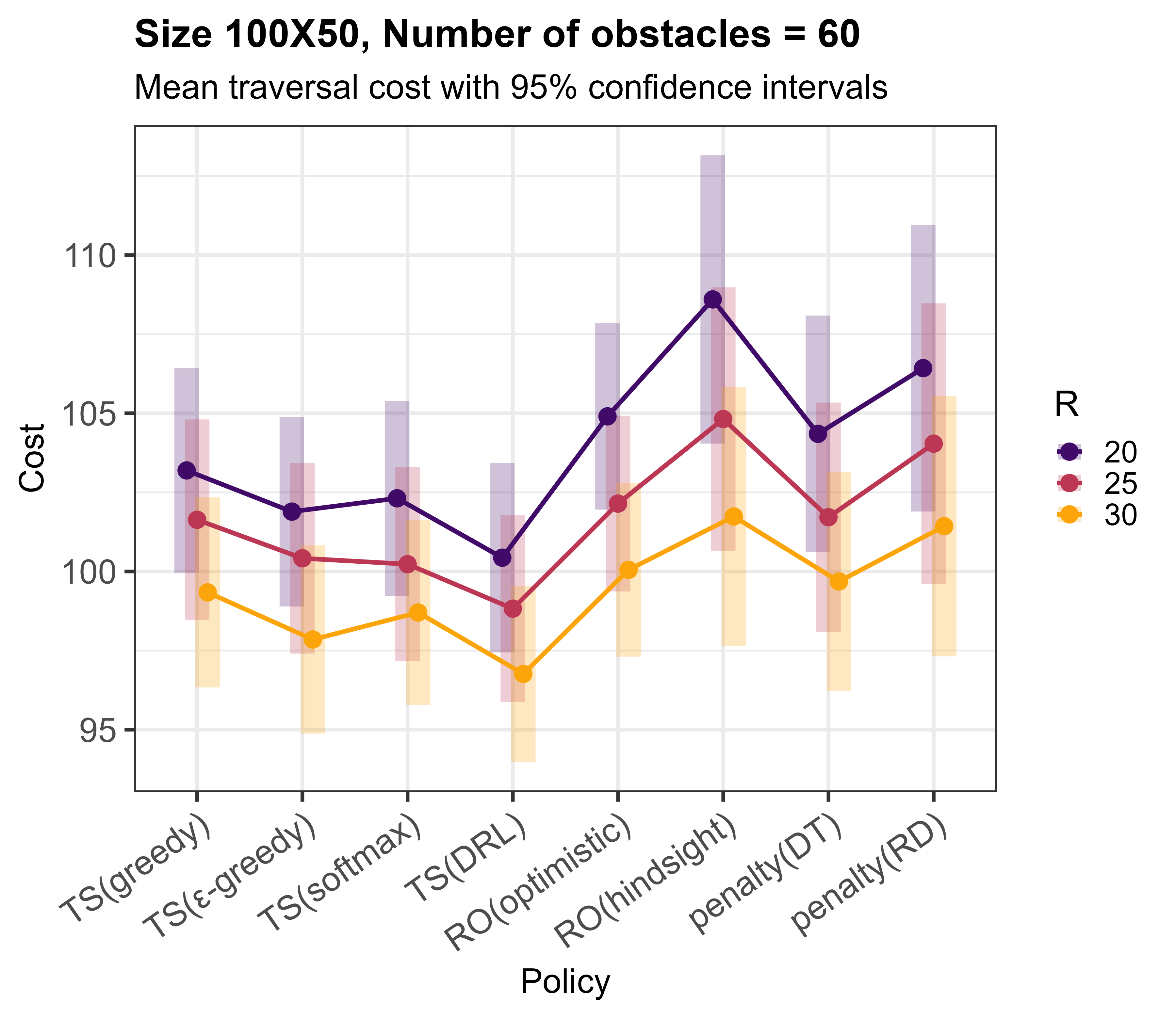}
	\end{tabular}
	\caption{The mean traversal cost with 95\% confidence intervals for proposed policy variants and baselines by sensing ranges}
	\label{fig: avg cost_range}
\end{figure}

\begin{figure}[H]
	\centering
	\begin{tabular}{cc}
		\includegraphics[width=0.35\textwidth]{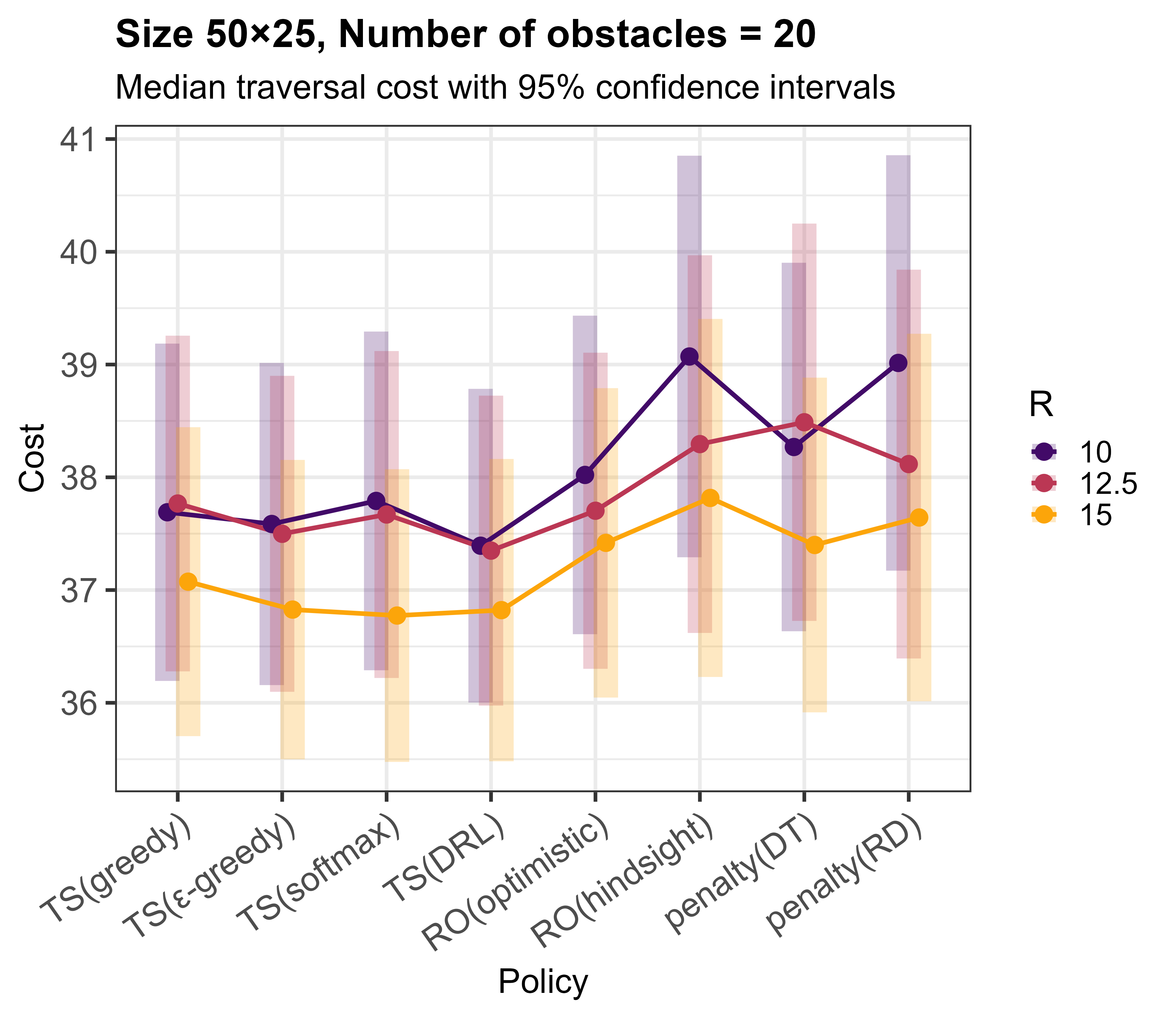} &
		\includegraphics[width=0.35\textwidth]{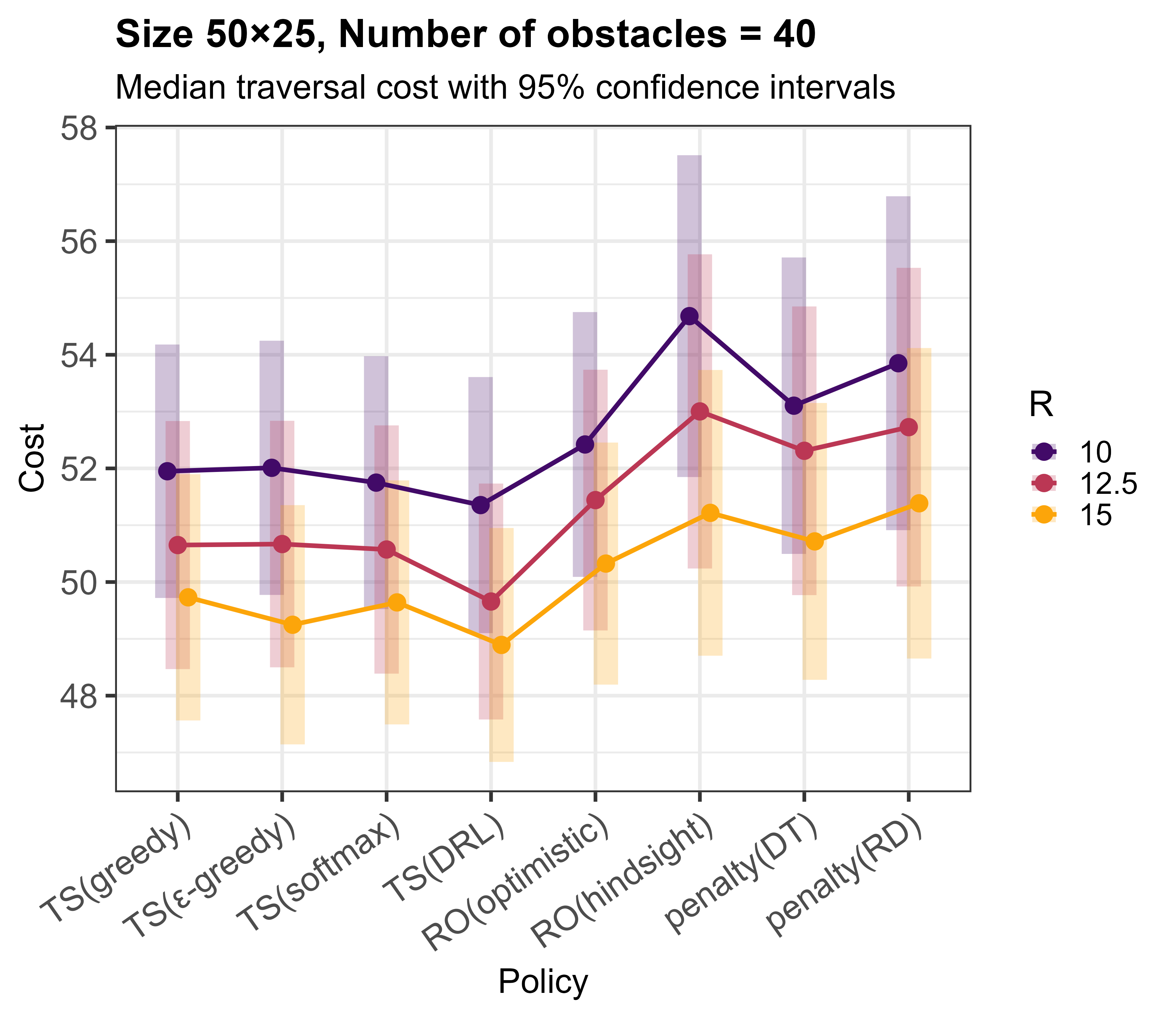} \\
		\includegraphics[width=0.35\textwidth]{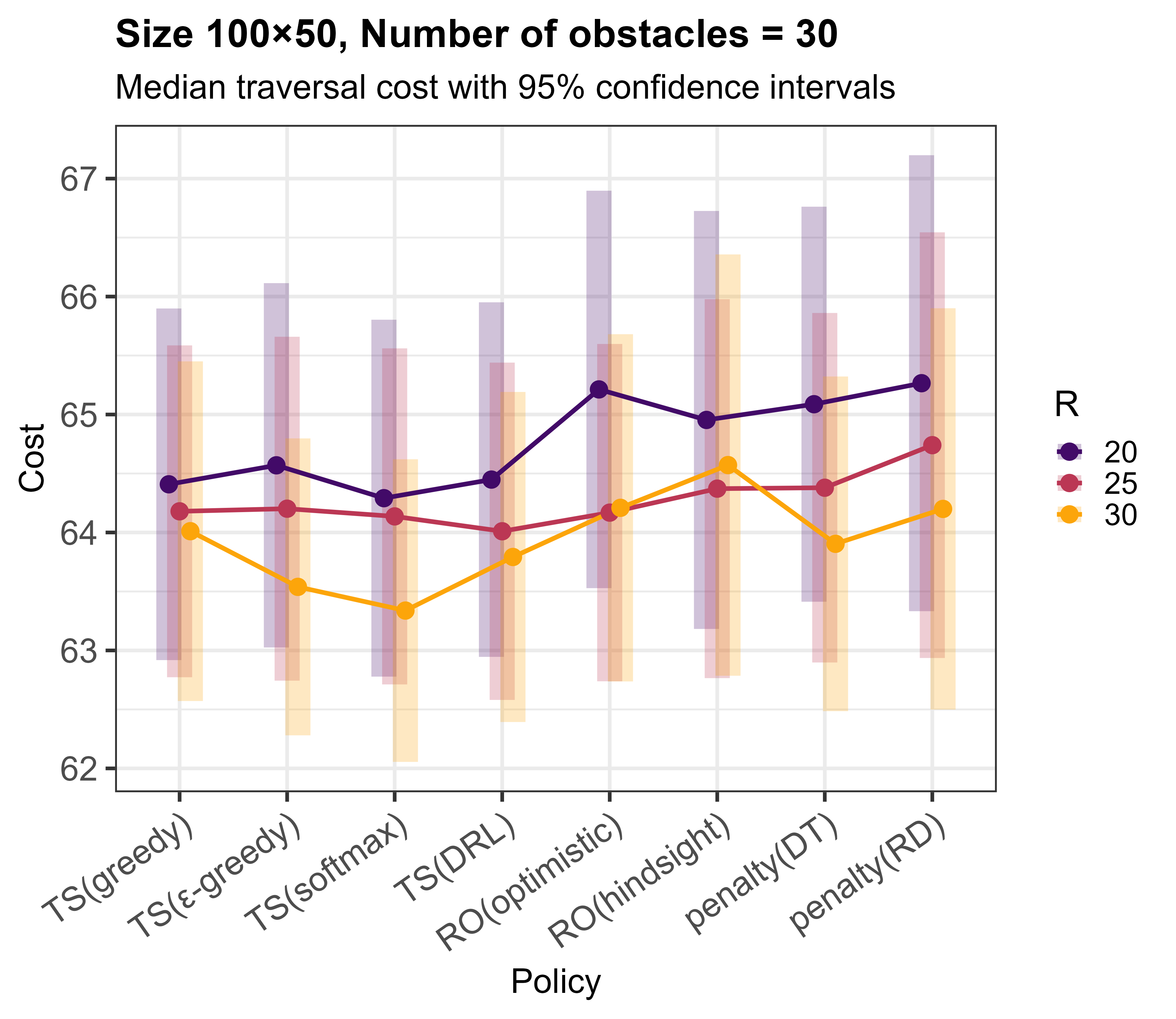} &
		\includegraphics[width=0.35\textwidth]{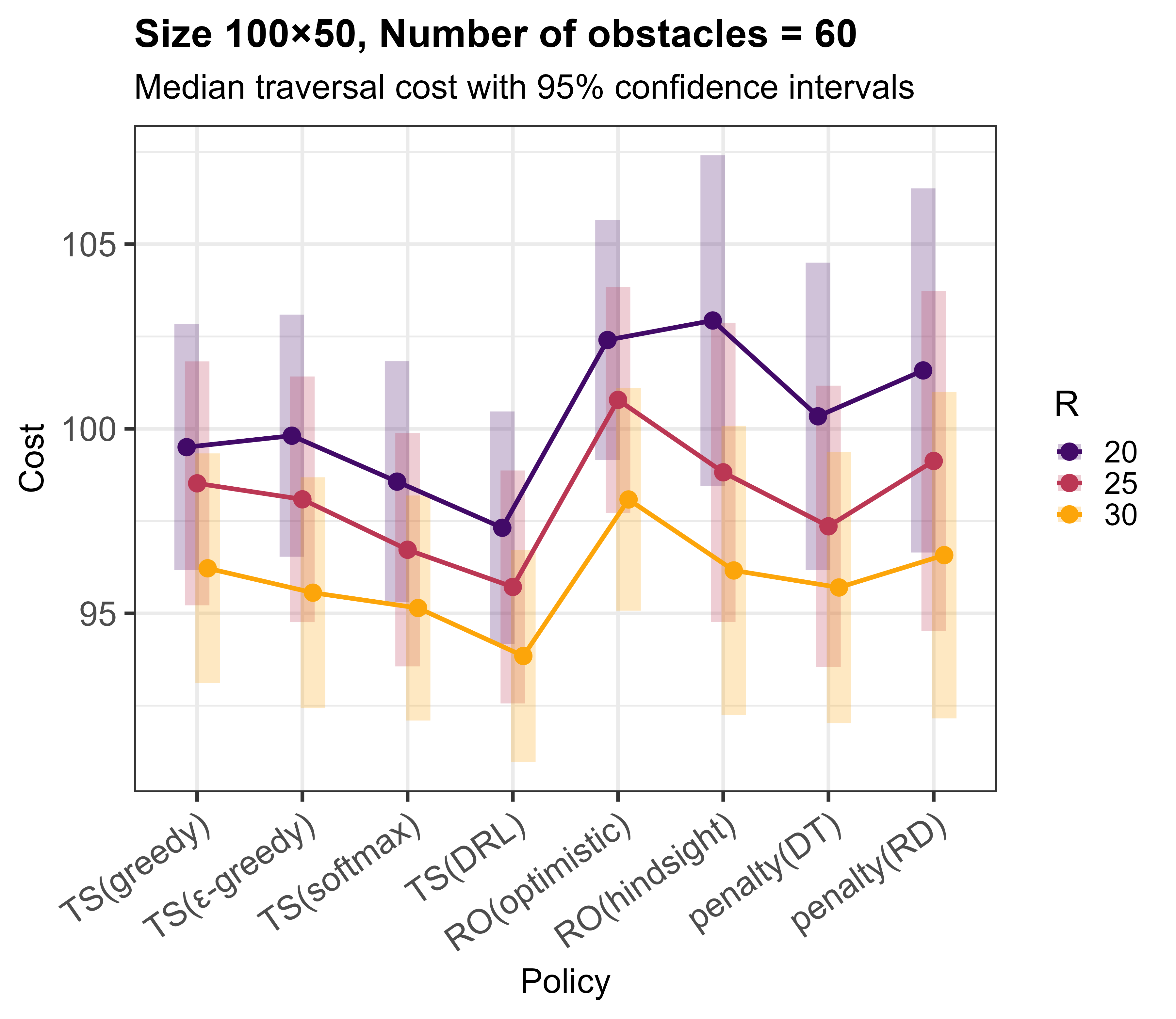}
	\end{tabular}
	\caption{The median traversal cost with 95\% confidence intervals for proposed policy variants and baselines by sensing ranges}
	\label{fig: median_cost_range}
\end{figure}

\begin{figure}[H]
	\centering
	\begin{tabular}{cc}
		\includegraphics[width=0.35\textwidth]{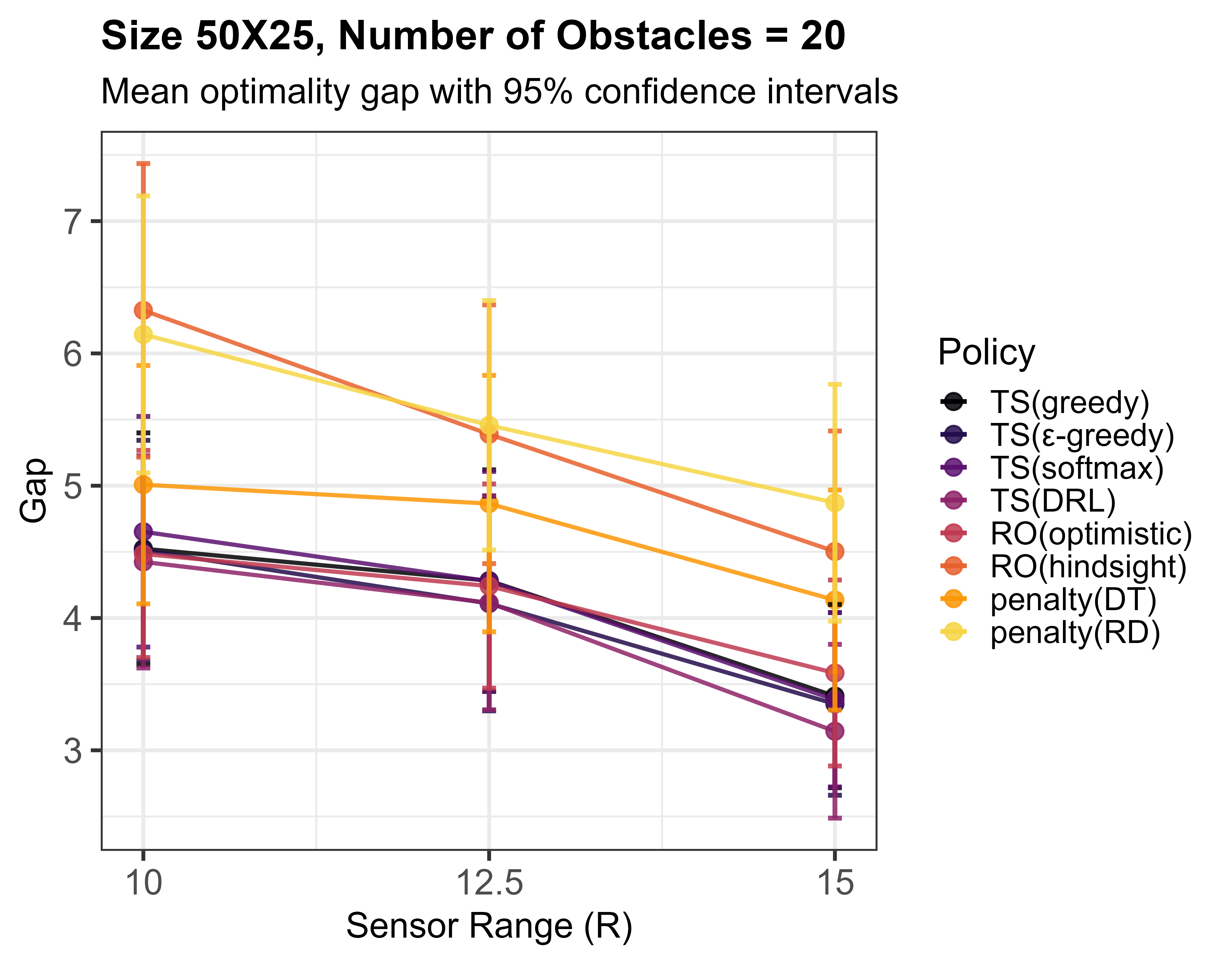} &
		\includegraphics[width=0.35\textwidth]{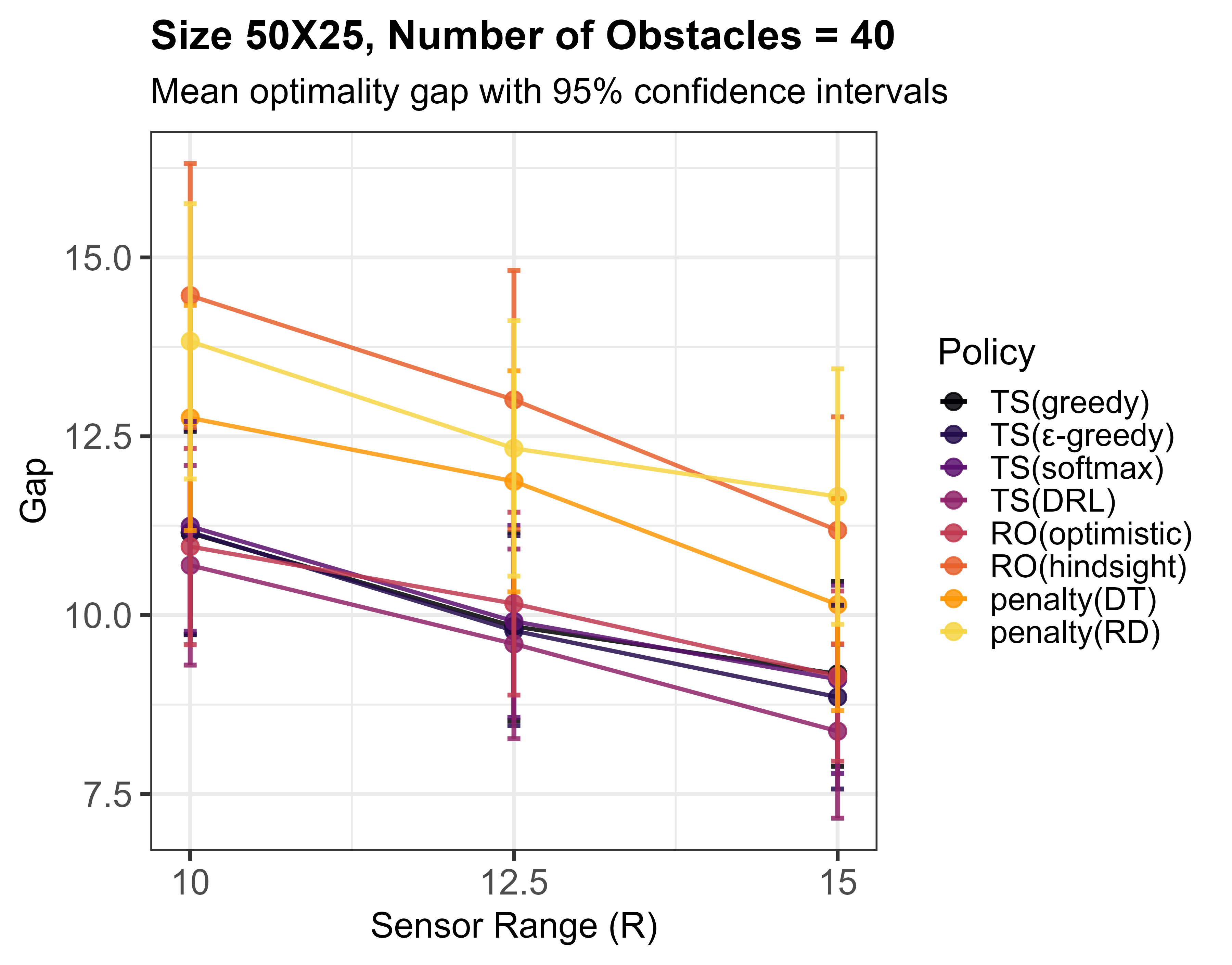}\\
		\includegraphics[width=0.35\textwidth]{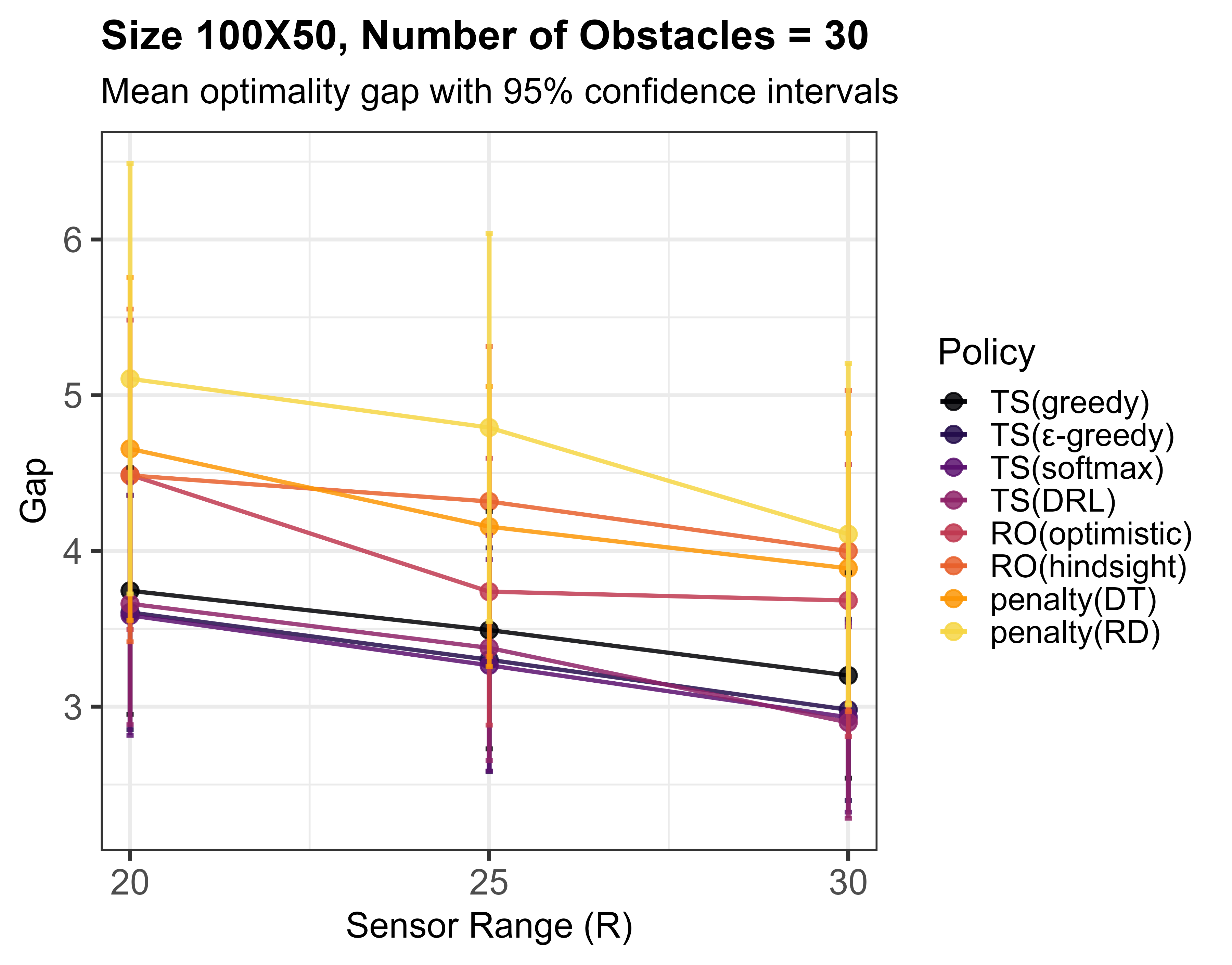} &
		\includegraphics[width=0.35\textwidth]{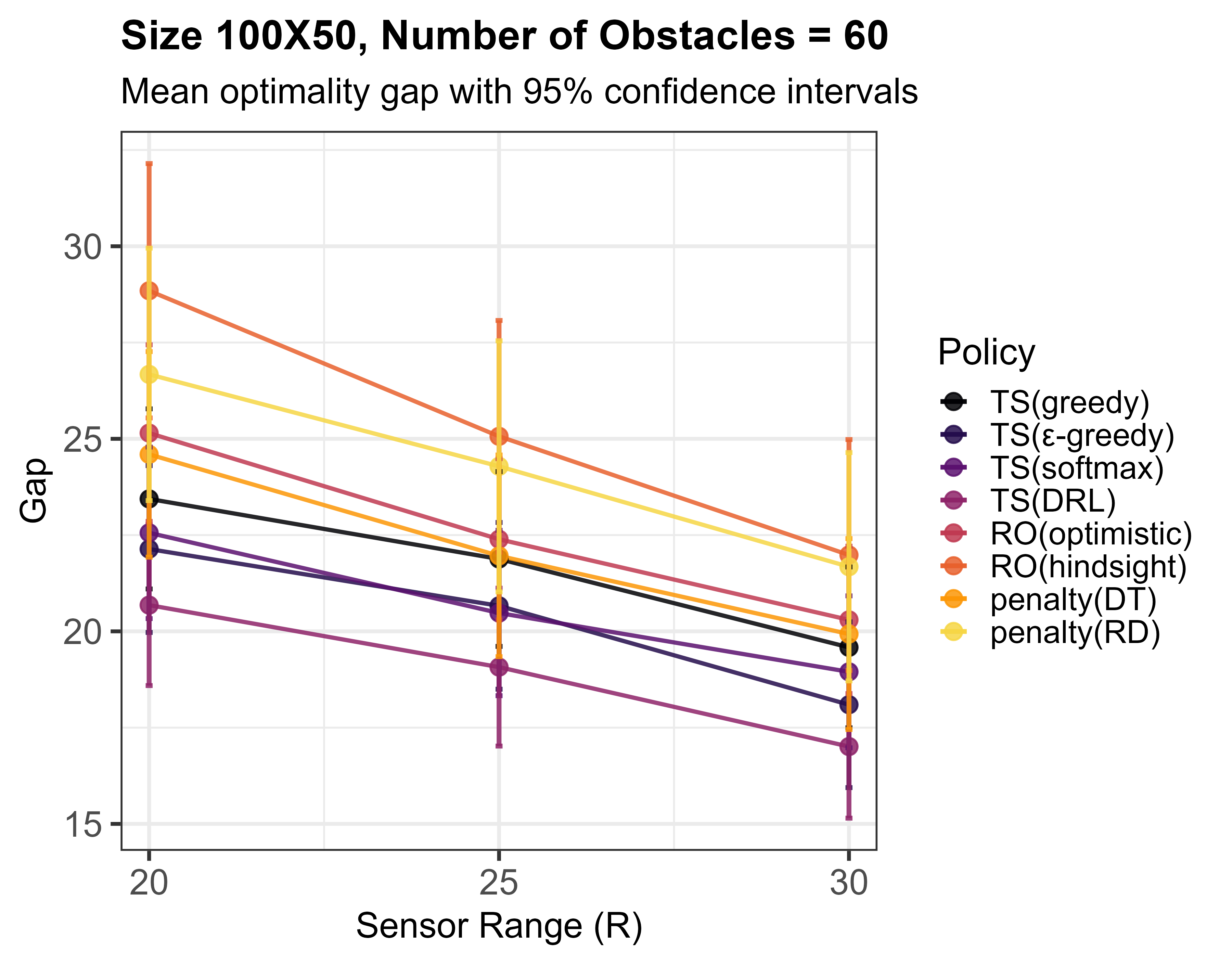}
	\end{tabular}
	\caption{The average deviation from optimal solutions with 95\% confidence intervals for proposed policy variants and baselines by sensing ranges}
	\label{fig: avg deviation_range}
\end{figure}

\clearpage

\end{document}